\pdfoutput=1 

\documentclass{article}

\usepackage[square]{natbib}
\usepackage[utf8]{inputenc} 
\usepackage[T1]{fontenc}    
\usepackage{url}            
\usepackage{booktabs}       
\usepackage{amsmath,amsfonts,amsthm} 
\usepackage{bm}
\usepackage{url}

\usepackage{nicefrac}       
\usepackage{microtype}      
\usepackage[inline]{enumitem}
\usepackage{graphicx}
\usepackage{subfig}
\usepackage{multirow}
\usepackage[title]{appendix}

\newtheorem{theorem}{Theorem}
\newtheorem{proposition}{Proposition}
\newtheorem{definition}{Definition}
\newtheorem{lemma}{Lemma}

\usepackage{xspace}
\makeatletter
\DeclareRobustCommand\onedot{\futurelet\@let@token\@onedot}
\def\@onedot{\ifx\@let@token.\else.\null\fi\xspace}

\def\eg{\emph{e.g}\onedot} 
\def\ie{\emph{i.e}\onedot}

\makeatother

\usepackage{enumitem}

\newenvironment{tight_enumerate}{
\begin{enumerate}[leftmargin=20pt]
  \setlength{\topsep}{0pt}
  \setlength{\itemsep}{0pt}
  \setlength{\parskip}{0pt}
  \setlength{\parsep}{0pt}
}{\end{enumerate}}

\newif\ifsupp
\suppfalse

\newif\ifarxiv
\arxivfalse

\newif\iffinal
\finalfalse

\newcommand{\myref}[1]{\textcolor{red}{#1}}

\usepackage[preprint]{neurips_2020}
\title{Bidirectionally Self-Normalizing Neural Networks}

\author{
Yao Lu\thanks{Corresponding author. Email: \texttt{yaolubrain@gmail.com}}
\qquad
Stephen Gould 
\qquad
Thalaiyasingam Ajanthan \\
Australian National University
}

\usepackage{color}

\begin{document}

\maketitle

\begin{abstract}
The problem of vanishing and exploding gradients has been a long-standing obstacle that hinders the effective training of neural networks.
Despite various tricks and techniques that have been employed to alleviate the problem in practice, there still lacks satisfactory theories or provable solutions. 
In this paper, we address the problem from the perspective of high-dimensional probability theory. 
We provide a rigorous result that shows, under mild conditions, how the vanishing/exploding gradients problem disappears with high probability if the neural networks have sufficient width. 
Our main idea is to constrain both forward and backward signal propagation in a nonlinear neural network through a new class of activation functions, namely Gaussian-Poincar\'e normalized functions, and orthogonal weight matrices.
Experiments on both synthetic and real-world data validate our theory and confirm its effectiveness on very deep neural networks when applied in practice.
\end{abstract}

\section{Introduction}

Neural networks have brought unprecedented performance in various artificial intelligence tasks~\citep{ciregan2012multi,krizhevsky2012imagenet,graves2013speech,silver2017mastering}. 
However, despite decades of research,
training neural networks is still mostly guided by empirical observations and successful training often requires various heuristics and extensive hyperparameter tuning. It is therefore desirable to understand the cause of the difficulty in neural network training and to propose theoretically sound solutions.

A major difficulty is the vanishing/exploding gradients problem~\citep{hochreiter1991untersuchungen,bengio1994learning,glorot2010understanding,pascanu2013difficulty,philipp2018exploding}. That is, the norm of the gradient in each layer is either growing or shrinking at an exponential rate as the gradient signal is propagated from the top layer to bottom layer. For deep neural networks, this problem might cause numerical overflow and make the optimization problem intrinsically difficult, as the gradient in each layer has vastly different magnitude and therefore the optimization landscape becomes pathological. One might attempt to solve the problem by simply normalizing the gradient in each layer. Indeed, the adaptive gradient optimization methods~\citep{duchi2011adaptive,tieleman2012lecture,kingma2014adam} implement this idea and have been widely used in practice. However, one might also wonder if there is a solution more intrinsic to deep neural networks, whose internal structure if well-exploited would lead to further advances.

To enable the trainability of deep neural networks, batch normalization~\citep{ioffe2015batch} was proposed in recent years and achieved widespread empirical success. 
Batch normalization is a differentiable operation which normalizes its inputs based on mini-batch statistics and is inserted between the linear and nonlinear layers. It is reported that batch normalization can accelerate neural network training significantly~\citep{ioffe2015batch}. However, batch normalization does not solve the vanishing/exploding gradients problem~\citep{philipp2018exploding}. Indeed it is proved that batch normalization can actually worsen the problem~\citep{yang2019mean}. Besides, batch normalization requires separate training and testing phases and can be ineffective when the mini-batch size is small~\citep{ioffe2017batch}. 

\newpage

Alternatively, self-normalizing neural networks~\citep{klambauer2017self} and dynamical isometry theory~\citep{pennington2017resurrecting} were proposed to combat the vanishing/exploding gradients problem. 
In self-normalizing neural networks, a new activation function, scaled exponential linear unit (SELU), was devised to ensure the output of each unit to have zero mean and unit variance. 
In dynamical isometry theory, all singular values of the input-output Jacobian matrix are constrained to be close to one at initialization. 
This amounts to initializing the functionality of a neural network to be close to an orthogonal matrix.
While the two theories dispense batch normalization, it is shown that neural networks with SELU still suffer from the vanishing/exploding gradients problem and dynamical isometry restricts the functionality of neural networks to be close to linear (pseudo-linearity)~\citep{philipp2018exploding}.

In this paper, we follow the above line of research to investigate neural network trainability. Our contributions are three-fold: First, we propose a new type of neural networks that consist of orthogonal weight matrices and a new class of activation functions which we call Gaussian-Poincar\'e normalized (GPN) functions. We show many common activation functions can be easily transformed into their respective GPN versions. Second, we rigorously prove that the vanishing/exploding gradients problem disappears with high probability in the neural networks if the width of each layer is sufficiently large.
Third, with experiments on synthetic and real-world data, we confirm that the vanishing/exploding gradients problem is solved to large extent in the neural networks while nonlinear functionality is maintained.

\section{Theory}\label{sec:bsnn}

In this section, we introduce bidirectionally self-normalizing neural networks (BSNNs) formally and analyze its properties. All the proofs of our results are left to Appendix.
To simplify the analysis, we define neural network in a restricted sense as the following.
\begin{definition}[\textbf{Neural Network}]
A neural network is a function from $\mathbb{R}^d$ to $\mathbb{R}^d$ that for $l=1,...,L$
\begin{align}
\mathbf{h}^{(l)} = \mathbf{W}^{(l)}\mathbf{x}^{(l)}, \quad \mathbf{x}^{(l+1)} = \phi(\mathbf{h}^{(l)}), \label{eq:nn}
\end{align}
where $\mathbf{W}^{(l)}\in \mathbb{R}^{d\times d}$, $\phi: \mathbb{R} \to \mathbb{R}$ is a differentiable function applied element-wise to a vector, $\mathbf{x}^{(1)}$ is the input and $\mathbf{x}^{(L+1)}$ is the output.
\end{definition}
Under this definition, $\phi$ is called the activation function, $\{\mathbf{W}^{(l)}\}_{l=1}^{L}$ are called the parameters, $d$ is called the width and $L$ is called the depth and superscript $(l)$ denotes the $l$-th layer of a neural network. 
The above formulation is similar to \citep{pennington2017resurrecting} but we omit the bias term in (\ref{eq:nn}) for simplicity as it plays no role in our analysis. 

Let $E$ be the objective function of $\{\mathbf{W}^{(l)}\}_{l=1}^{L}$ and $\mathbf{D}^{(l)} = \text{diag}(\phi'(h_1^{(l)}),...,\phi'(h_d^{(l)}))$, where $\phi'$ denotes the derivative of $\phi$ and $h_i^{(l)}$ denotes the $i$-th element of $\mathbf{h}^{(l)}$. 
Now, the error signal is back propagated via
\begin{align}
\mathbf{y}^{(L)} =  \mathbf{D}^{(L)}\frac{\partial E}{\partial \mathbf{x}^{(L+1)}},
\quad
\mathbf{y}^{(l)} =  \mathbf{D}^{(l)}(\mathbf{W}^{(l+1)})^T\mathbf{y}^{(l+1)},
\label{eq:grad_back}
\end{align}
and the gradient of the weight matrix for layer $l$ can be computed as 
\begin{align}
\frac{\partial E}{\partial \mathbf{W}^{(l)}} 
= \mathbf{y}^{(l)}(\mathbf{x}^{(l)})^T.
\end{align}

To solve the vanishing/exploding gradients problem, we  constrain the forward signal $\mathbf{x}^{(l)}$ and the backward signal $\mathbf{y}^{(l)}$ in order to constrain the norm of the gradient. This leads to the following.

\begin{definition}[\textbf{Bidirectional Self-Normalization}]\label{def:bsn} 
A neural network is bidirectionally self-normalizing if 
\vspace{-2ex}
\begin{align}
\|\mathbf{x}^{(1)}\|_2 &= ... = \|\mathbf{x}^{(L)}\|_2 = \sqrt{d}, \\
\|\mathbf{y}^{(1)}\|_2 &= ... = \|\mathbf{y}^{(L)}\|_2 = \Big\|\frac{\partial E}{\partial \mathbf{x}^{(L+1)}}\Big\|_2.
\end{align}
\end{definition}

\begin{proposition} \label{pro:gradnorm}
If a neural network is bidirectionally self-normalizing, then
\begin{align}
\Big\|\frac{\partial E}{\partial \mathbf{W}^{(1)}}\Big\|_F =  ... = \Big\|\frac{\partial E}{\partial \mathbf{W}^{(L)}}\Big\|_F.
\end{align}
\end{proposition}

%
%
In the rest of this section, we derive the conditions under which bidirectional self-normalization is achievable for a neural network.

\newpage

\subsection{Constraints on Weight Matrices}

We constrain the weight matrices to be orthogonal since multiplication by an orthogonal matrix preserves the norm of a vector. For linear neural networks, this guarantees bidirectional self-normalization and its further benefits are discussed in \citep{saxe2014exact}. 
Even for nonlinear neural networks, orthogonal constraints are shown to improve the trainability with proper scaling \citep{mishkin2015all,pennington2017resurrecting}.

\subsection{Constraints on Activation Functions}\label{sec:act}

To achieve bidirectional self-normalization for a nonlinear network, 
it is not enough only to constrain the weight matrices. 
We also need to constrain the activation function in such a way that both forward and backward signals are normalized. 
To this end, we propose the following constraint.

\begin{definition}[\textbf{Gaussian-Poincar\'e Normalization}]\label{def:gpn} 
Function $\phi:\mathbb{R}\to\mathbb{R}$ is Gaussian-Poincar\'e normalized if it is differentiable and
\begin{align}
\mathbb{E}_{x\sim\mathcal{N}(0,1)}[\phi(x)^2]  = \mathbb{E}_{x\sim\mathcal{N}(0,1)}[\phi'(x)^2] = 1.
\end{align}
\end{definition}
The definition is inspired by the following theorem which shows the fundamental relationship between a function and its derivative under Gaussian measure.

\begin{theorem}[\textbf{Gaussian-Poincar\'e Inequality} \citep{bogachev1998gaussian}]
If function $\phi:\mathbb{R}\to\mathbb{R}$ is differentiable with bounded $\mathbb{E}_{x\sim \mathcal{N}(0,1)}[\phi(x)^2]$ and $\mathbb{E}_{x\sim \mathcal{N}(0,1)}[\phi'(x)^2]$, then
\begin{align}
\mathrm{Var}_{x\sim \mathcal{N}(0,1)}[\phi(x)] \leq \mathbb{E}_{x\sim \mathcal{N}(0,1)}[\phi'(x)^2].
\end{align}
\end{theorem}

Note that there is an implicit assumption that the input is approximately Gaussian for a Gaussian-Poincar\'e normalized (GPN) function.
Even though this is standard in the  literature~\citep{klambauer2017self,pennington2017resurrecting,schoenholz2016deep}, we will rigorously prove that this assumption is valid when orthogonal weight matrices are used in~\eqref{eq:nn}. 
Next, we state a property of GPN functions.

\begin{proposition}
Function $\phi:\mathbb{R}\to\mathbb{R}$ is Gaussian-Poincar\'e normalized and $\mathbb{E}_{x\sim \mathcal{N}(0,1)}[\phi(x)] = 0$ if and only if $\phi(x)=x$ or $\phi(x)=-x$.
\end{proposition}

This result indicates that any nonlinear function with zero mean under Gaussian distribution (\eg, Tanh and SELU) is not GPN. 
Now we show that a large class of activation functions can be converted into their respective GPN versions using an affine transformation.

\begin{proposition}
For any differentiable function $\phi:\mathbb{R}\to\mathbb{R}$ with non-zero and bounded  $\mathbb{E}_{x\sim \mathcal{N}(0,1)}[\phi(x)^2]$ and  $\mathbb{E}_{x\sim \mathcal{N}(0,1)}[\phi'(x)^2]$, 
there exist two constants $a$ and $b$ such that $a \phi(x)+b$ is Gaussian-Poincar\'e normalized.
\end{proposition}

To obtain $a$ and $b$, one can use numerical procedure to compute the values of $\mathbb{E}_{x\sim\mathcal{N}(0,1)}[\phi'(x)^2]$, $\mathbb{E}_{x\sim\mathcal{N}(0,1)}[\phi(x)^2]$ and $\mathbb{E}_{x\sim\mathcal{N}(0,1)}[\phi(x)]$ and then solve the quadratic equations
\begin{align}
\mathbb{E}_{x\sim\mathcal{N}(0,1)}[a^2\phi'(x)^2] = 1, \\
\mathbb{E}_{x\sim\mathcal{N}(0,1)}[(a\phi(x)+b)^2] = 1.
\end{align}
We computed $a$ and $b$ (not unique) for several common activation  functions \citep{nair2010rectified,maas2013rectifier,clevert2015fast,klambauer2017self,hendrycks2016gaussian} with their default hyperparameters%
\footnote{We use $\phi(x)=\max(0,x)+0.01\min(0,x)$ for LeakyReLU, $\phi(x)=\max(0,x)+\min(0,\exp(x)-1)$ for ELU and $\phi(x)=x/(1+\exp(-1.702x))$ for GELU.}
and the results are listed in Table 1. 
Note that ReLU, LeakyReLU and SELU are not differentiable at $x=0$ but they can be regarded as approximations of their smooth counterparts.
We ignore such point and evaluate the integrals for $x\in (-\infty,0) \cup (0,\infty)$.

\begin{table}[h!]
    \centering
    \begin{small}
    \begin{tabular}{|c|c|c|c|c|c|c|}
        \hline
& Tanh &  ReLU & LeakyReLU  & ELU  & SELU  & GELU  \\
        \hline
$a$ & $1.4674$ &  $1.4142$ & $1.4141$ & $1.2234$ & $0.9660$ & $1.4915$ \\
$b$ & $0.3885$ &  $0.0000$ & $0.0000$ & $0.0742$ & $ 0.2585$ & $-0.9097$ \\
        \hline
    \end{tabular}
    \end{small}
    \vspace{0.25cm}
    \caption{Constants for Gaussian-Poincar\'e normalization of activation functions.}
    \label{tab:gpn_constants}
    \vspace{-3ex}
\end{table}


With the orthogonal constraint on the weight matrices and the Gaussian-Poincar\'e normalization on the activation function, we prove that bidirectional self-normalization is achievable with high probability under mild conditions in the next subsection.

\subsection{Norm-Preservation Theorems}\label{sec:normpres}

The bidirectional self-normalization may not be achievable precisely in general unless the neural network is a linear one. Therefore, we investigate the properties of neural networks in a probabilistic framework. The random matrix theory and the high-dimensional probability theory allow us to characterize the behaviors of a large class of neural networks by its mean behavior, which is significantly simpler to analyze. Therefore, we study neural networks of random weights whose properties may shed light on the trainability of neural networks in practice.

First, we need a probabilistic version of the vector norm constraint. 

\begin{definition}[\textbf{Thin-Shell Concentration}]
Random vector $\mathbf{x}\in\mathbb{R}^d$ is thin-shell concentrated if for any $\epsilon > 0$
\begin{align}
\mathbb{P}\Big\{ \Big| \frac{1}{d}\|\mathbf{x}\|^2_2 - 1 \Big| \geq \epsilon \Big\} \to 0
\end{align}
as $d \to \infty$.
\end{definition}

The definition is modified from the one in \citep{bobkov2003concentration}.   
Examples of thin-shell concentrated distributions include standard multivariate Gaussian and any distribution on the $d$-dimensional sphere of radius $\sqrt{d}$.

To prove the main results, \ie, the norm-preservation theorems, we require the following assumptions.

\paragraph{Assumptions.}
\begin{tight_enumerate}
    \item \emph{Random vector $\mathbf{x}\in\mathbb{R}^d$ is thin-shell concentrated.}
    \item \emph{Random orthogonal matrix $\mathbf{W}=(\mathbf{w}_1,...,\mathbf{w}_d)^T$ is uniformly distributed.}
    \item \emph{Function $\phi: \mathbb{R}\to \mathbb{R}$ is Gaussian-Poincar\'e normalized.}
    \item \emph{Function $\phi: \mathbb{R}\to \mathbb{R}$ and its derivative are Lipschitz continuous.}
\end{tight_enumerate}

The above assumptions are not restrictive.
For Assumption 1, one can always normalize the input vectors of a neural network.
For Assumption 2, orthogonal constraint or its relaxation has already been employed in neural network training~\citep{brock2016neural}.
Note, in Assumption 2, uniformly distributed means that $\mathbf{W}$ is distributed under Haar measure, which is the unique rotation invariant probability measure on orthogonal matrix group. We refer the reader to \citep{meckes2019random} for details. Furthermore, all the activation functions or their smooth counterparts listed in Table~\ref{tab:gpn_constants} satisfy Assumptions 3 and 4.

With the above assumptions, we can prove the following norm-preservation theorems. 

\begin{theorem}[\textbf{Forward Norm-Preservation}]\label{thm:forward}
Random vector 
\begin{align}
(\phi(\mathbf{w}^T_1\mathbf{x}),...,\phi(\mathbf{w}^T_d\mathbf{x}))
\end{align}
is thin-shell concentrated.
\end{theorem}

This result shows the transformation (orthogonal matrix followed by the GPN activation function) can preserve the norm of its input with high probability. Since the output is thin-shell concentrated, it serves as the input for the next layer and so on. Hence, the forward pass can preserve the norm of its input in each layer along the forward path when $d$ is sufficiently large.

\begin{theorem}[\textbf{Backward Norm-Preservation}]\label{thm:back}
Let $\mathbf{D} = \textup{diag}(\phi'(\mathbf{w}_1^T\mathbf{x}),...,\phi'(\mathbf{w}_d^T\mathbf{x}))$ and $\mathbf{y} \in \mathbb{R}^d$ be a fixed vector with bounded $\|\mathbf{y}\|_\infty$.
Then for any $\epsilon > 0$
\begin{align}
\mathbb{P}\Big\{\frac{1}{d}\Big| \|\mathbf{D}\mathbf{y}\|_2^2  - \|\mathbf{y}\|_2^2 \Big| \geq \epsilon \Big\} \to 0
\end{align}
as $d \to \infty$.
\end{theorem}

This result shows that the multiplication by the diagonal matrix $\mathbf{D}$ preserves the norm of its input with high probability. Since orthogonal matrix $\mathbf{W}$ also preserves the norm of its input, when the gradient error signal is propagated backwards as in (\ref{eq:grad_back}), the norm is preserved in each layer along the backward path when $d$ is sufficient large.

Hence, combining Theorems~\ref{thm:forward} and~\ref{thm:back}, we proved that bidirectional self-normalization is achievable with high probability if the neural network is wide enough and the conditions in the Assumptions are satisfied. Then by Proposition~\ref{pro:gradnorm}, the vanishing/exploding gradients problem disappears with high probability.

\paragraph{Sketch of the proofs.}
The proofs of Theorems~\ref{thm:forward} and~\ref{thm:back} are mainly based on a phenomenon in high-dimensional probability spaces, concentration of measure. 
We refer the reader to \citep{vershynin2018high} for an introduction to the subject.
Briefly, it can be shown that for some high-dimensional probability distributions, most mass is concentrated around certain range. For example, while most mass of a low-dimensional standard multivariate Gaussian distribution is concentrated around the center, most mass of a high-dimensional standard multivariate Gaussian distribution is concentrated around a thin-shell. 
Based on this phenomenon, it can be shown that random vector $\mathbf{Wx}$ in high dimensions is approaching a random vector uniformly distributed on a sphere. Then the random vector uniformly distributed on a high-dimensional sphere is approximately Gaussian. Then the Gaussian random variables transformed by Lipschitz and GPN functions are subgaussian with unit variance. And the random vector of the subgaussian random variables has the concentration of norm property in high dimensions. Each of these steps is rigorously proved in Appendix.

    \section{Experiments}\label{sec:exp}
    We verify our theory on both synthetic and real-world data. More experimental results can be found in Appendix.
    In short, while very deep neural networks with non-GPN activations show vanishing/exploding gradients, GPN versions show stable gradients and improved trainability in both synthetic and real data.
    Furthermore, compared to dynamical isometry theory, BSNNs do not exhibit pseudo-linearity and maintain nonlinear functionality.
    
    
    \subsection{Synthetic Data}\label{sec:exp_syn}
    
    We create synthetic data to test the norm-preservation properties of the neural networks.
    The input $\mathbf{x}^{(1)}$ is 500 data points of random standard Gaussian vectors of 500 dimensions. The gradient error ${\partial E}/{\partial\mathbf{x}^{(L+1)}}$ is also random standard Gaussian vector of 500 dimensions. All the neural networks have depth 200. All the weight matrices are random orthogonal matrices uniformly generated. No training is performed.
    
    In Figure~\ref{exp:synthetic}, we show the norm of inputs and gradients of the neural networks of width 500. From the results, we can see that with GPN, the vanishing/exploding gradients problem is eliminated to large extent. 
    The neural network with Tanh activation function does not show the vanishing/exploding gradients problem either. However, $\|\mathbf{x}^{(l)}\|$ is close to zero for large $l$ and each layer is close to a linear one since $\text{Tanh}(x)\approx x$ when $x\approx 0$ (pseudo-linearity), for which dynamical isometry is achieved.
    
    One might wonder if bidirectional self-normalization has the same effect as dynamical isometry in solving the vanishing/exploding gradients problem, that is, to make the neural network close to an orthogonal matrix. To answer this question, we show the histogram of $\phi'(h_i^{(l)})$ in Figure~\ref{exp:D_hist}. If the functionality of a neural network is close to an orthogonal matrix, since the weight matrices are orthogonal, then the values of $\phi'(h_i^{(l)})$ would concentrate around one (Figure~\ref{exp:D_hist}~(a)), which is not the case for BSNNs (Figure~\ref{exp:D_hist}~(b)). This shows that BSNNs do not suffer from the vanishing/exploding gradients problem while exhibiting nonlinear functionality.
    
    In Figure~\ref{exp:synthetic_width}, we show the gradient norm of BSNNs with varying width.
    Notice, as the width increases, the norm of gradient in each layer of the neural network becomes more equalized, as predicted by our theory.
    
    \subsection{Real-World Data}
    
    We run experiments on real-world image datasets MNIST and CIFAR-10. The neural networks have width 500 and depth 200 (plus one unconstrained layer at bottom and one at top to fit the dimensionality of the input and output). We use stochastic gradient descent of momentum 0.5 with mini-batch size 64 and learning rate 0.0001. The training is run for 50 epochs for MNIST and 100 epochs for CIFAR-10.
    We do not use data augmentation. 
    Since it is computationally expensive to enforce the orthogonality constraint, we simply constrain each row of the weight matrix to have $l_2$ norm one as a relaxation of orthogonality by the following parametrization $\mathbf{W} = (\mathbf{v}_1/\|\mathbf{v}_1\|_2,...,\mathbf{v}_d/\|\mathbf{v}_d\|_2)^T$ and optimize $\mathbf{V} = (\mathbf{v}_1,...,\mathbf{v}_d)^T$ as an unconstrained problem.

\newpage

We summarize the results in Table \ref{table:real}. We can see that, for activation functions ReLU, LeakyReLU and GELU, the neural networks are not trainable. But once these functions are GPN, the neural network can be trained. GPN activation functions consistently outperform their unnormalized counterparts in terms of the trainability, as the training accuracy is increased, but not necessarily generalization ability.

We show the test accuracy during training in Figure \ref{exp:test_acc_cifar10}, from which we can see the training is accelerated when SELU is GPN. ReLU, LeakyReLU and GELU, if not GPN, are completely untrainable due to the vanished gradients (see Appendix).

We observe that batch normalization leads to gradient explosion when combining with any of the activation functions. This confirms the claim of \citep{philipp2018exploding,yang2019mean} that batch normalization does not solve the vanishing/exploding gradients problem.
On the other hand, without batch normalization the neural network with any GPN activation function has stable gradient magnitude throughout training (see Appendix). This indicates that BSNNs can dispense batch normalization and therefore avoid its shortcomings.

\vspace{0.5cm}

\begin{figure}[h!]
\centering

\subfloat[$\|\mathbf{x}^{(l)}\|_2^2/d$, Tanh.]{\includegraphics[width=0.4\textwidth]{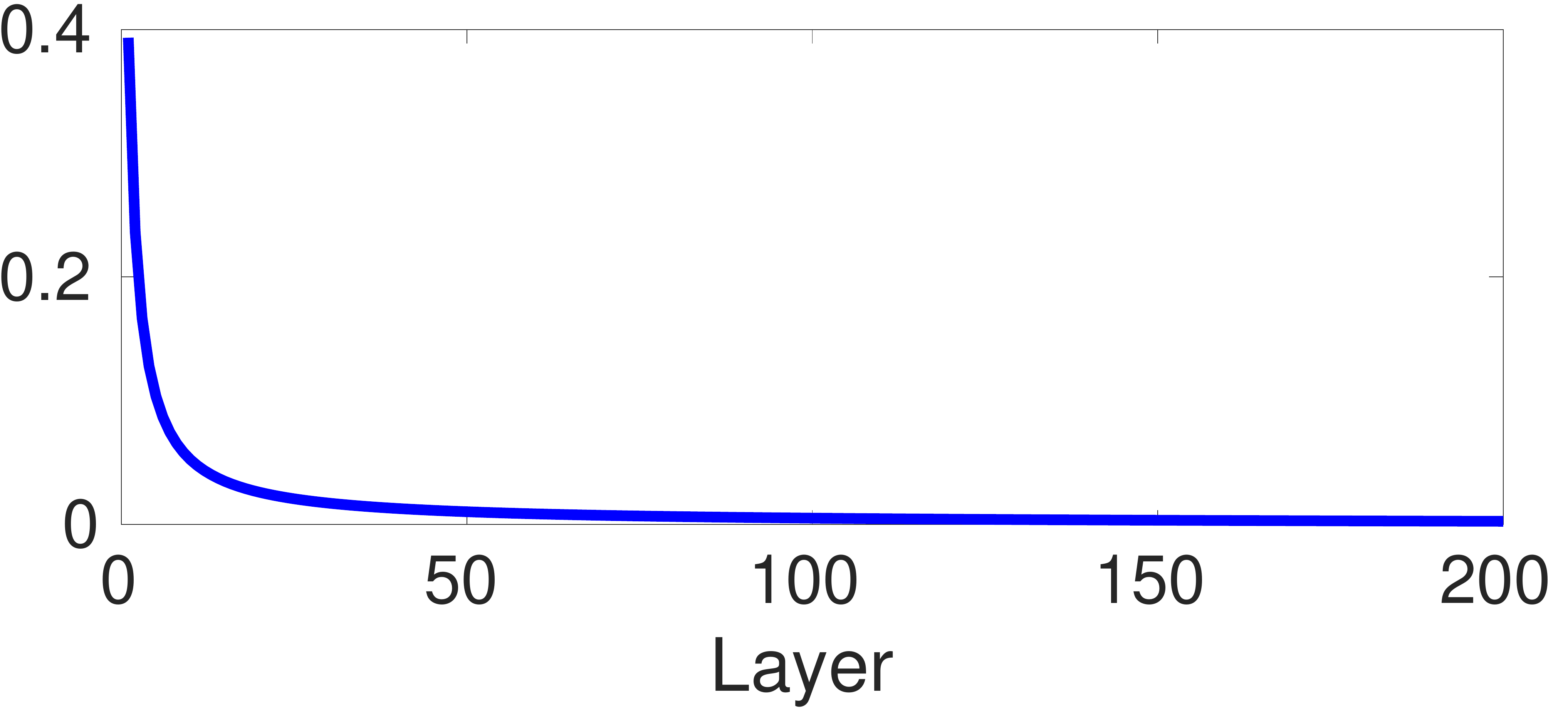}}
\hspace{0.5cm}
\subfloat[$\|\frac{\partial E}{\partial \mathbf{W}^{(l)}}\|_F$, Tanh.]{\includegraphics[width=0.4\textwidth]{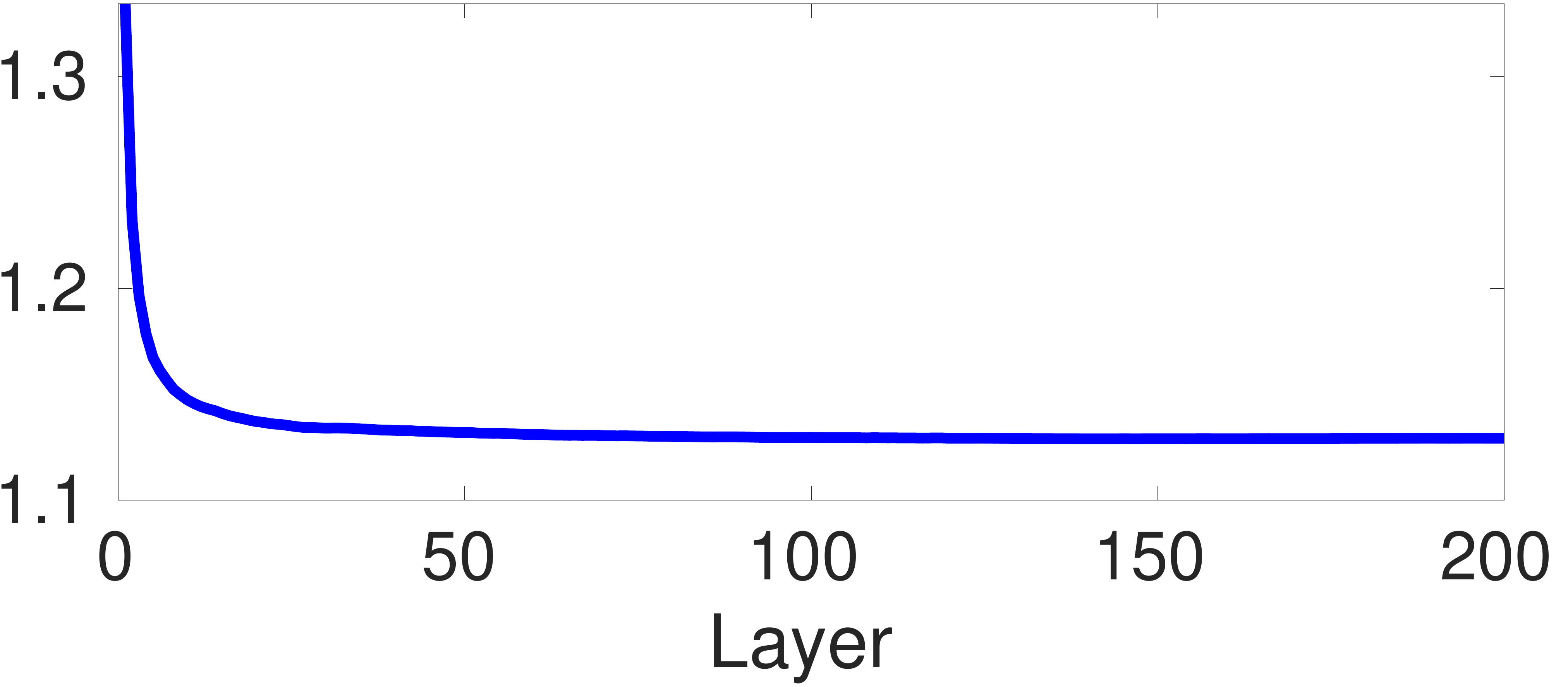}}

\subfloat[$\|\mathbf{x}^{(l)}\|_2^2/d$, Tanh-GPN.]{\includegraphics[width=0.4\textwidth]{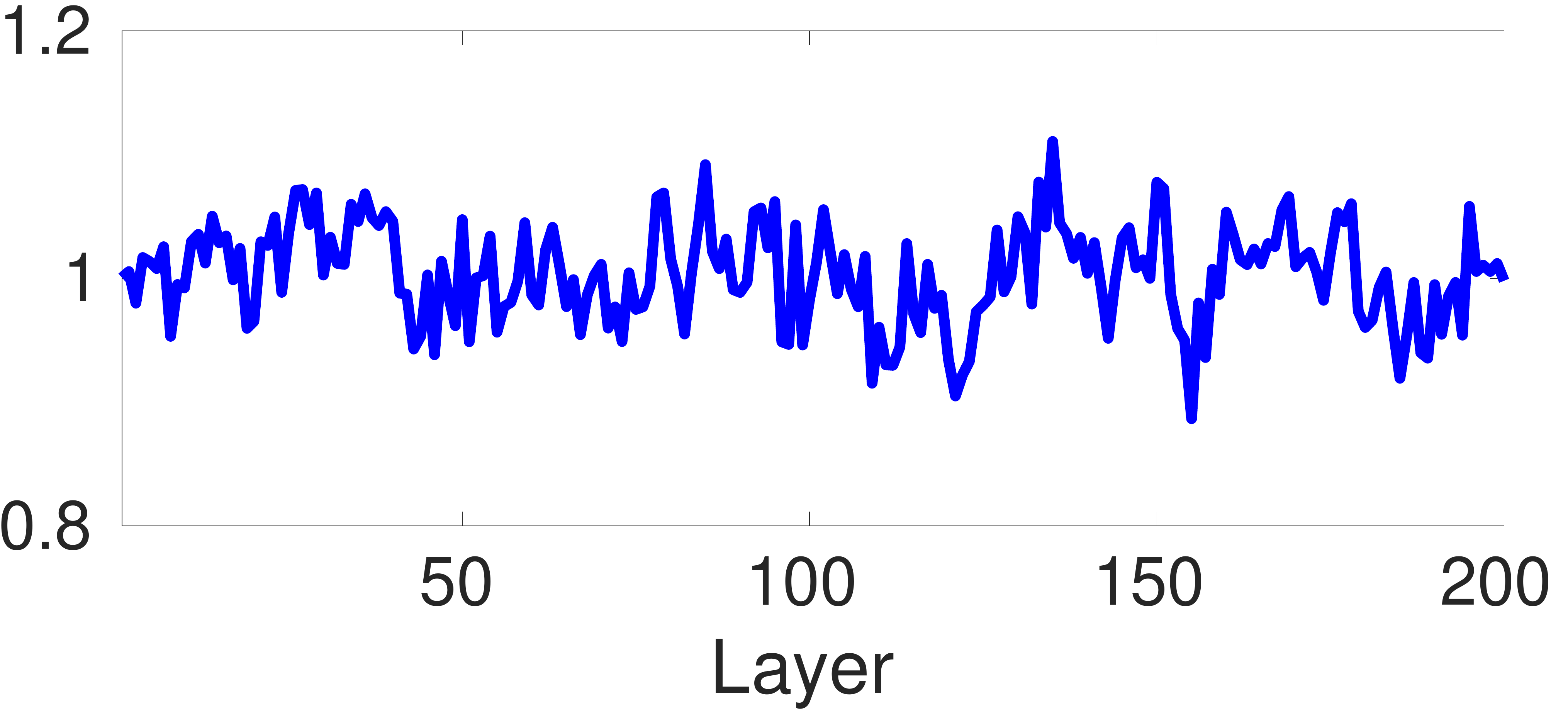}}
\hspace{0.5cm}
\subfloat[$\|\frac{\partial E}{\partial \mathbf{W}^{(l)}}\|_F$, Tanh-GPN.]{\includegraphics[width=0.4\textwidth]{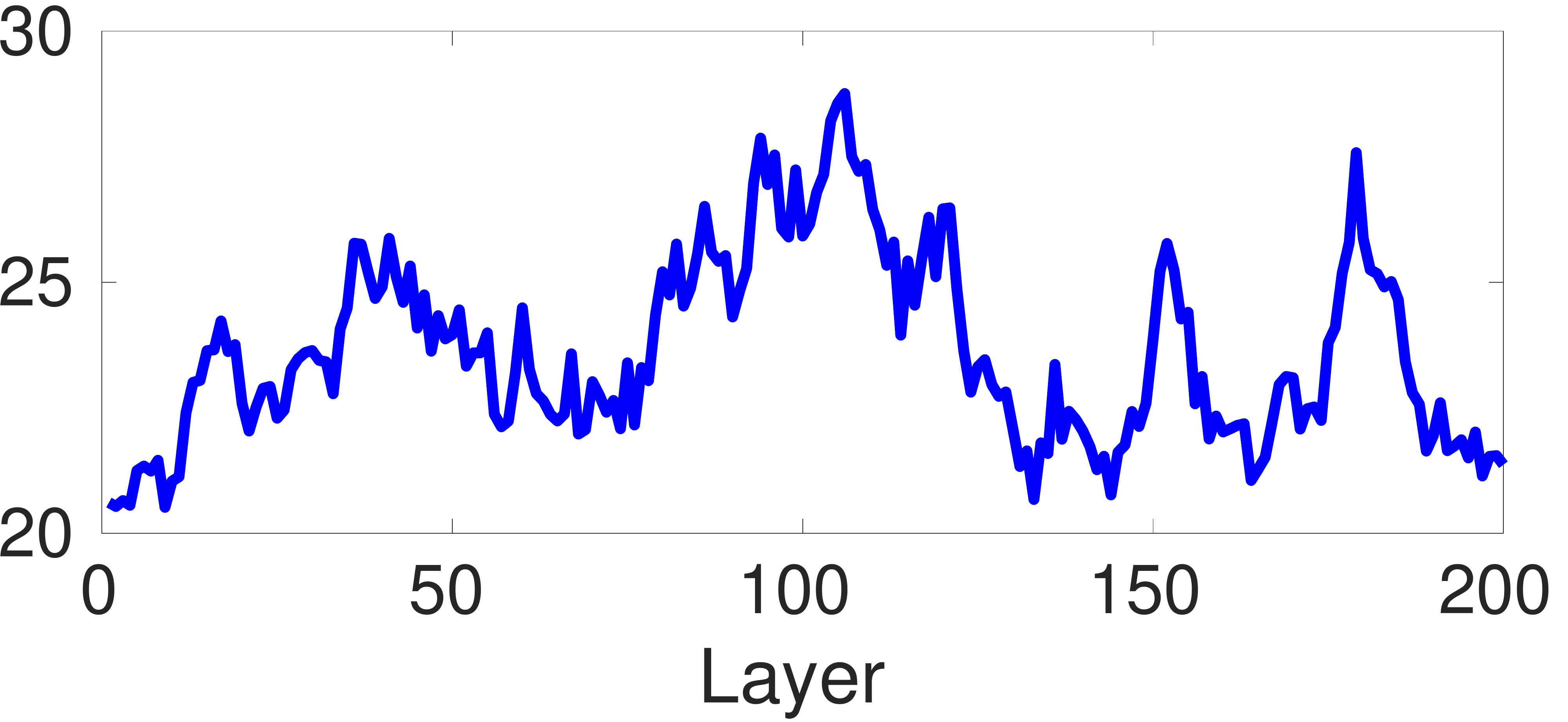}}

\subfloat[$\|\mathbf{x}^{(l)}\|_2^2/d$, SELU.]{\includegraphics[width=0.4\textwidth]{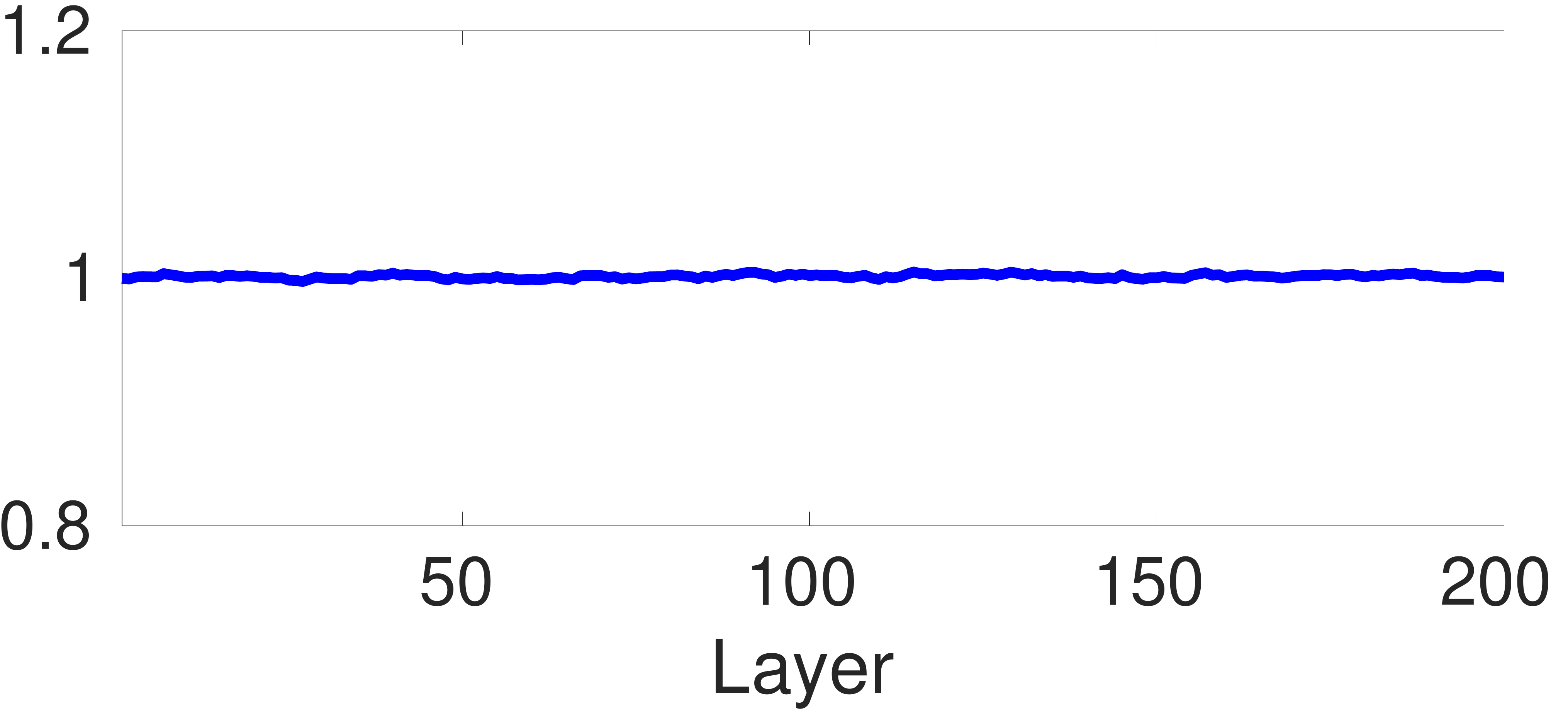}}
\hspace{0.5cm}
\subfloat[$\|\frac{\partial E}{\partial \mathbf{W}^{(l)}}\|_F$, SELU.]{\includegraphics[width=0.4\textwidth,height=0.2\textwidth,trim=0 0 0 -0.3cm]{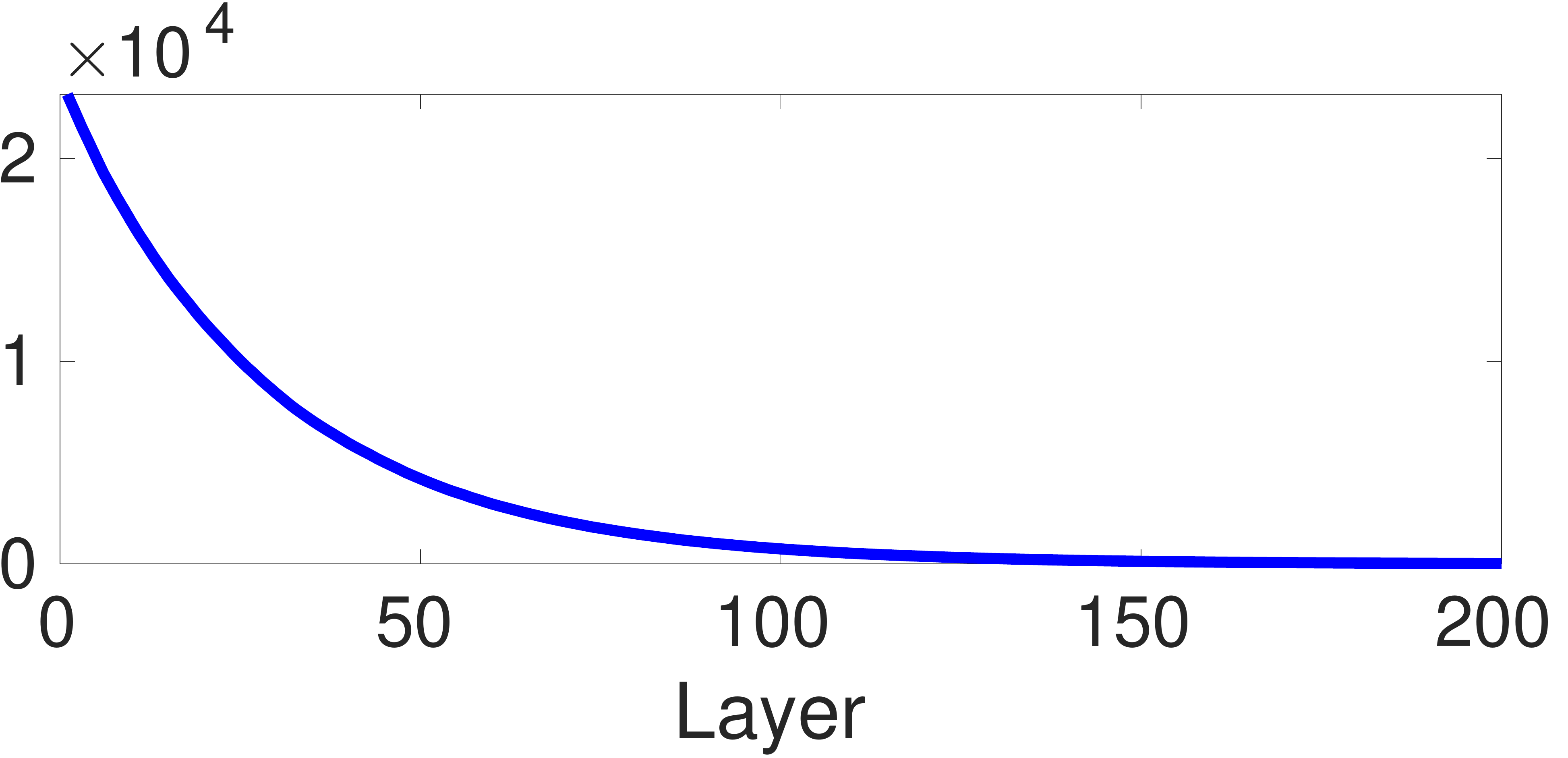}}

\subfloat[$\|\mathbf{x}^{(l)}\|_2^2/d$, SELU-GPN.]{\includegraphics[width=0.4\textwidth]{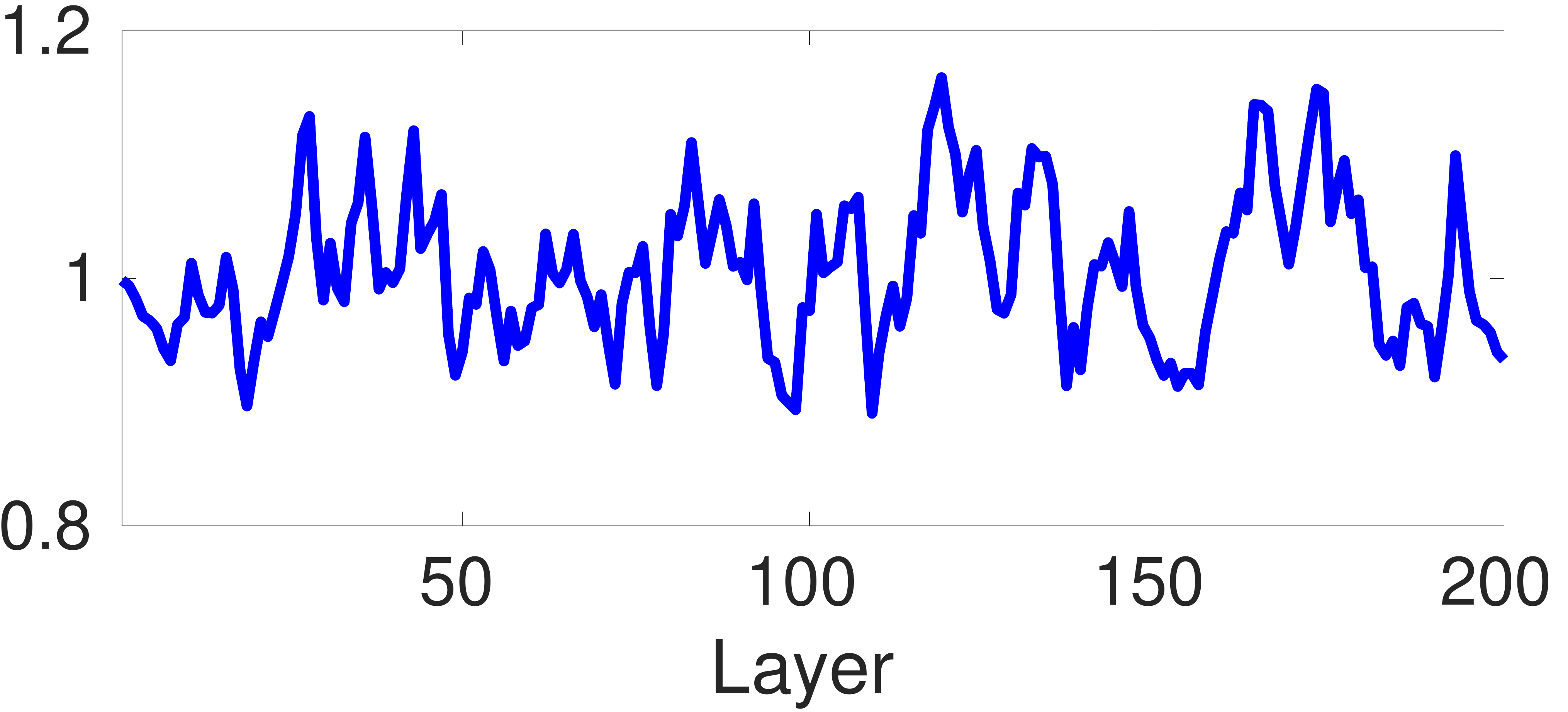}}
\hspace{0.5cm}
\subfloat[$\|\frac{\partial E}{\partial \mathbf{W}^{(l)}}\|_F$, SELU-GPN.]{\includegraphics[width=0.4\textwidth]{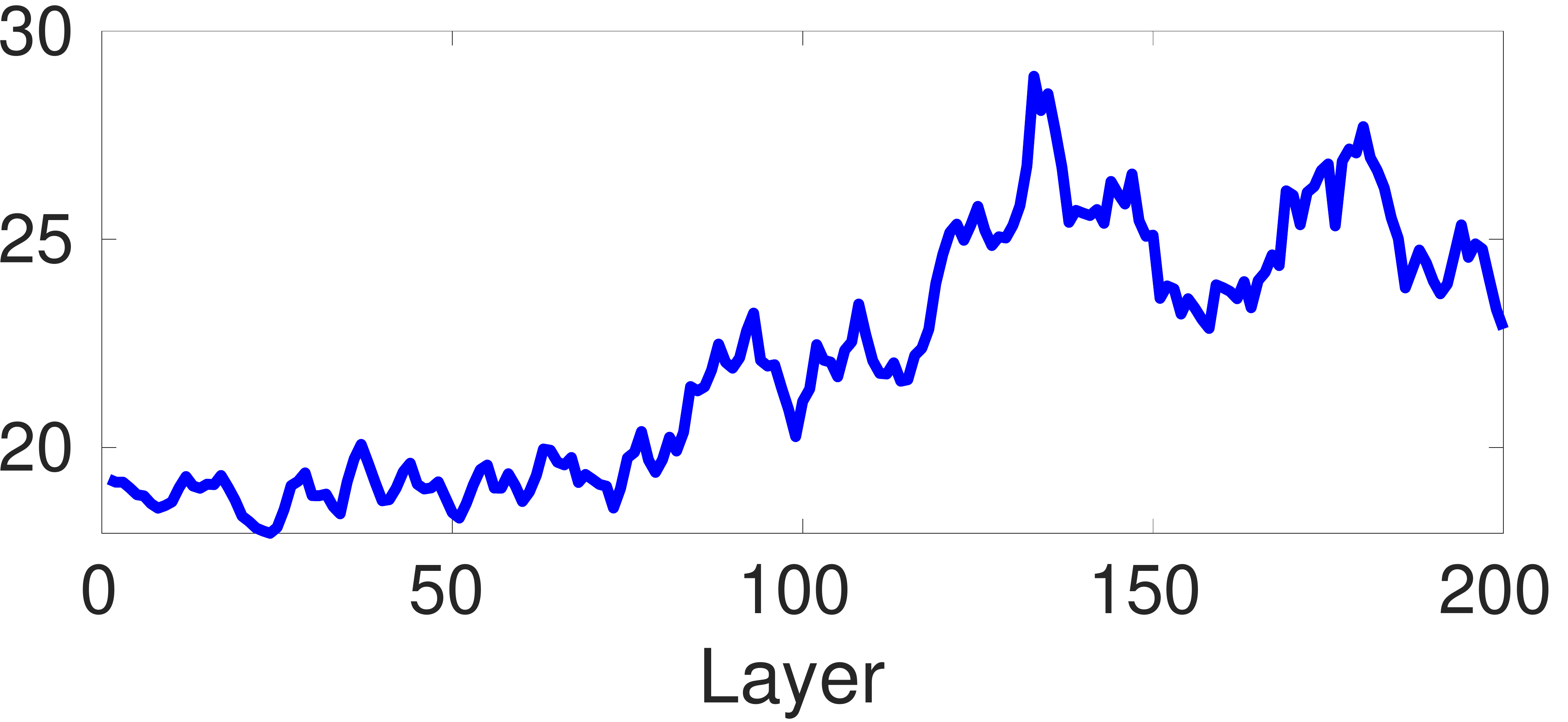}}

\vspace{0.5cm}
\caption{Results on synthetic data with different activation functions.  ``-GPN'' denotes the function is Gaussian-Poincar\'e normalized. $\|\mathbf{x}^{(l)}\|_2$ denotes the $l_2$ norm of the outputs of the $l$-th layer. $d$ denotes the width. $\|\frac{\partial E}{\partial \mathbf{W}^{(l)}}\|_F$ is the Frobenius norm of the gradient of the weight matrix in the $l$-th layer.
}
\label{exp:synthetic}
\vspace{-2ex}
\end{figure}

\begin{figure}[t]
\centering
\subfloat[Tanh.]{\includegraphics[width=0.45\textwidth]{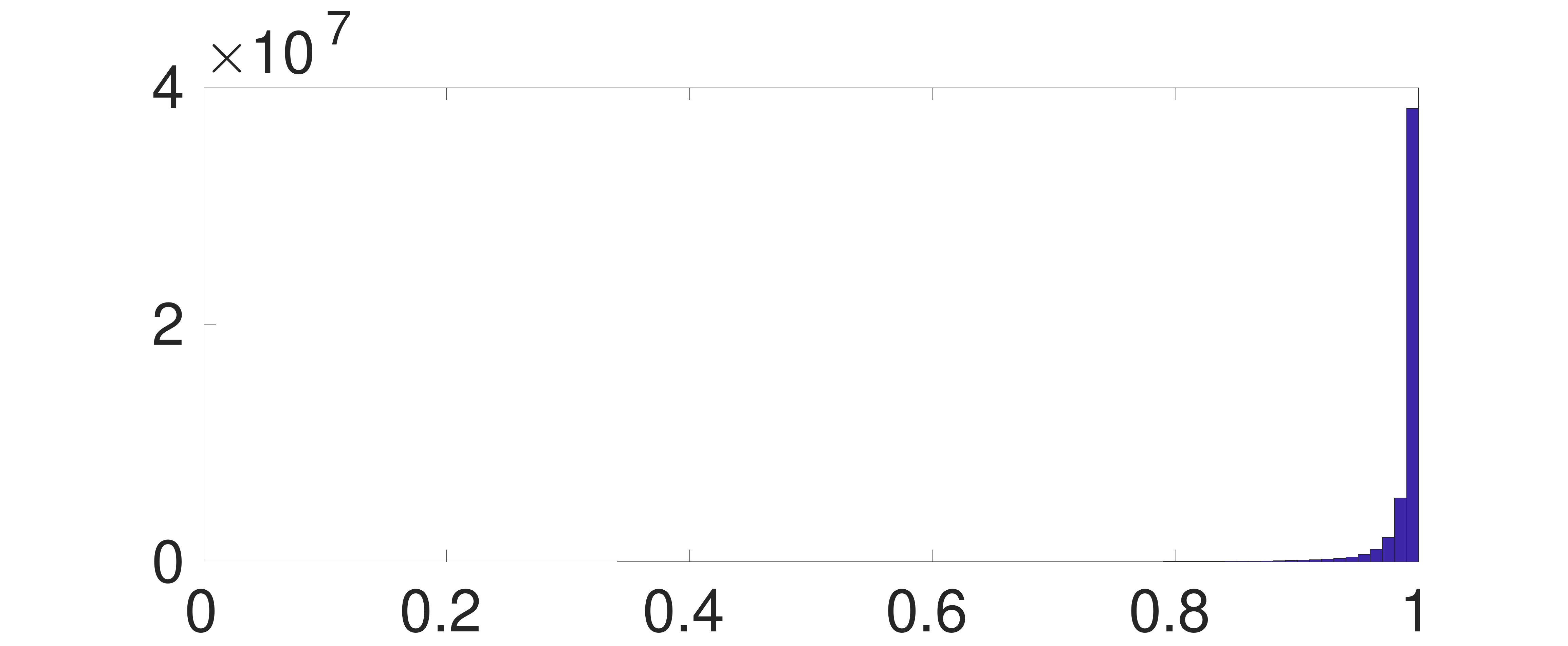}}
\subfloat[Tanh-GPN.]{\includegraphics[width=0.45\textwidth]{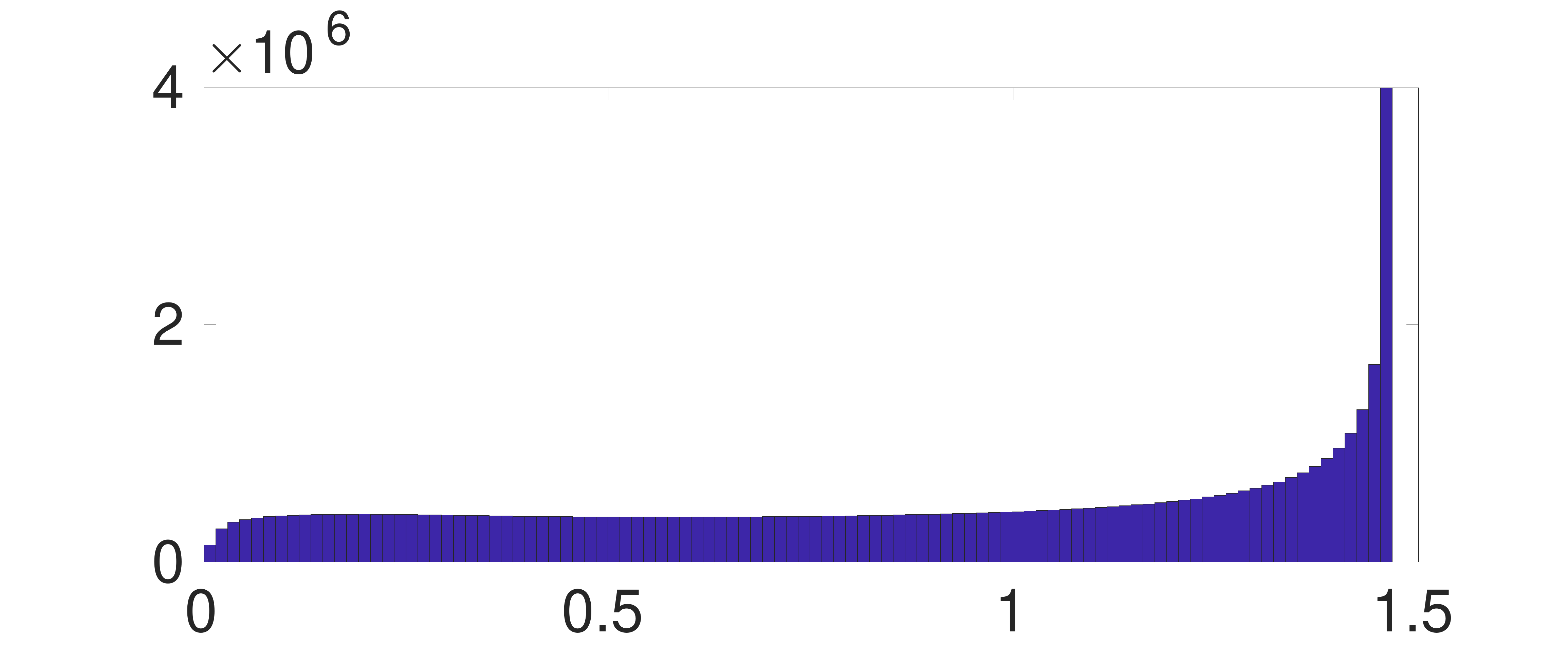}}

\caption{Histogram of $\phi'(h_i^{(l)})$. The values of $\phi'(h_i^{(l)})$ are accumulated for all units, all layers and all samples in the histogram.}
\label{exp:D_hist}
\vspace{-4ex}
\end{figure}
\begin{figure}[t]
\centering
\subfloat[Tanh-GPN.]{\includegraphics[width=0.4\textwidth]{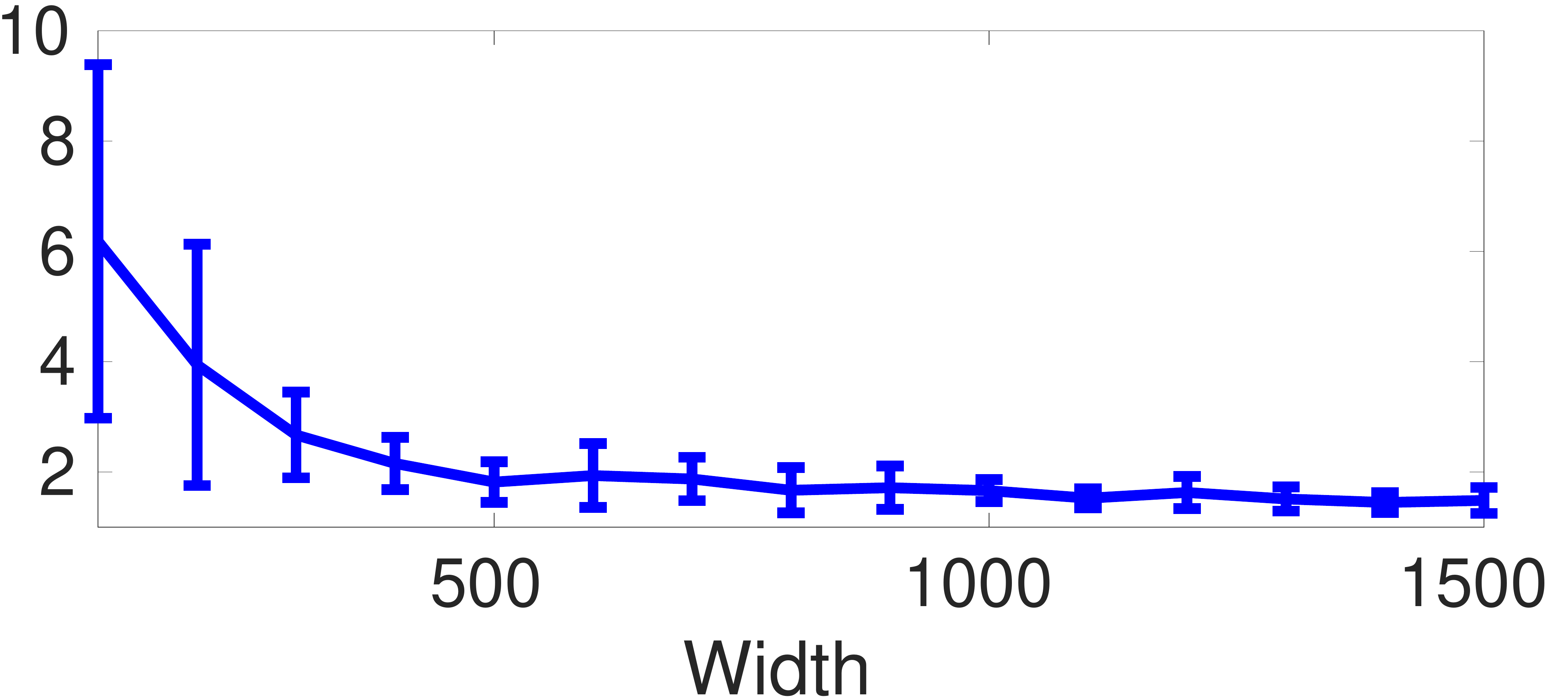}}
\hspace{0.5cm}
\subfloat[SELU-GPN.]{\includegraphics[width=0.4\textwidth]{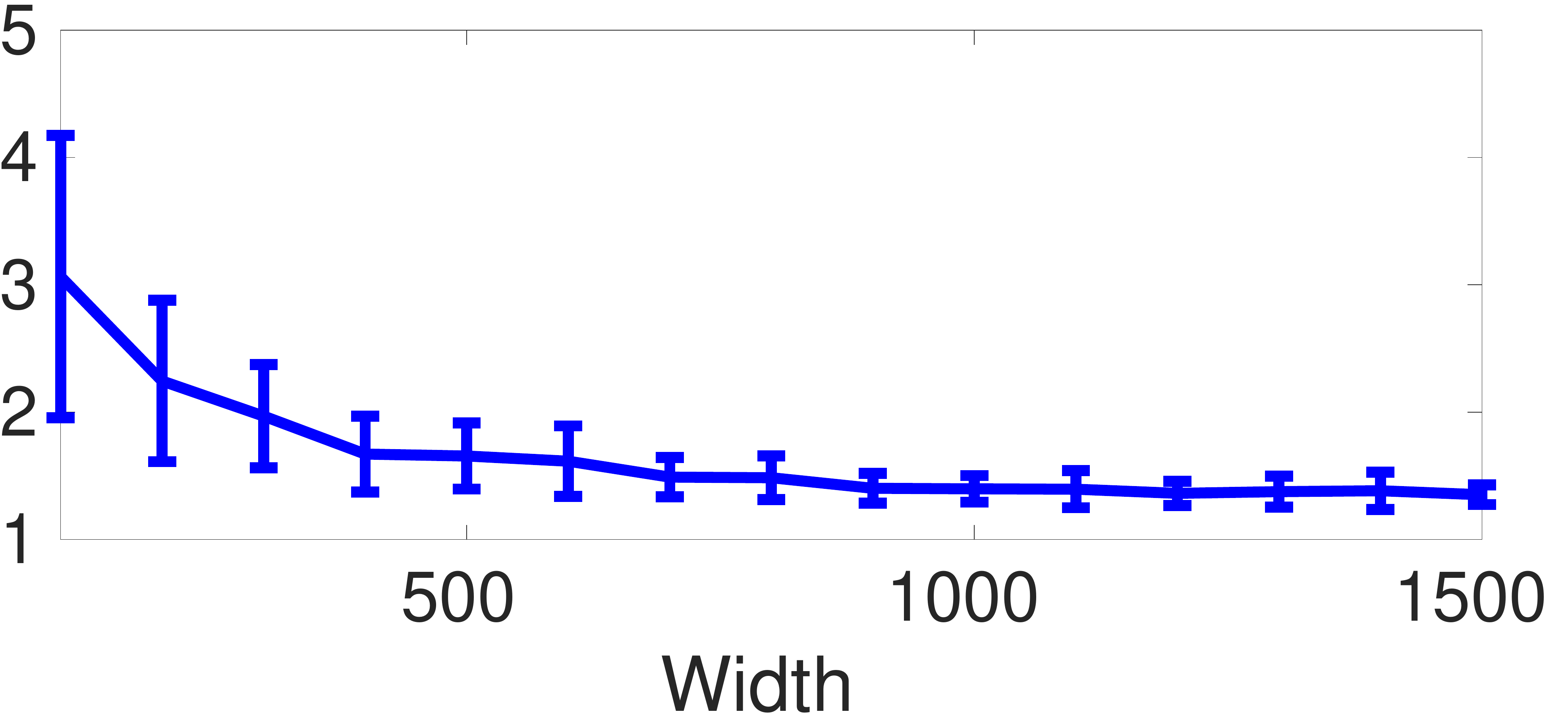}}

\caption{Gradient norm ratio for different layer width on synthetic data. The ratio is  $\max_l\|\frac{\partial E}{\partial \mathbf{W}^{(l)}}\|_F / \min_l\|\frac{\partial E}{\partial \mathbf{W}^{(l)}}\|_F$. The width ranges from 100 to 1500. The error bars show standard deviation.}
\label{exp:synthetic_width}
\vspace{-2ex}
\end{figure}

\begin{table}[t]
\begin{small}
    \centering
    \begin{tabular}{|l|c|c|c|c|}
        \hline 
        & \multicolumn{2}{c|}{MNIST} & \multicolumn{2}{c|}{CIFAR-10} \\
        \cline{2-5} 
         & Train & Test & Train & Test \\
        \hline
        Tanh & 99.05 (87.39) & \textbf{96.57} (89.32) & 80.84 (27.90) & \textbf{42.71} (29.32) \\
        Tanh-GPN & \textbf{99.81} (84.93) & 95.54 (87.11) & \textbf{96.39} (25.13) & 40.95 (26.58)\\
        \hline
        ReLU & 11.24 (11.24) & 11.35 (11.42) & 10.00 (10.00) & 10.00 (10.00) \\
        ReLU-GPN & \textbf{33.28} (11.42) & \textbf{28.13} (11.34) & \textbf{46.60} (10.09) & \textbf{34.96} (9.96) \\        
        \hline
        LeakyReLU & 11.24 (11.24) & 11.35 (11.63) & 10.00 (10.21) & 10.00 (10.06)\\
        LeakyReLU-GPN & \textbf{43.17} (11.19) & \textbf{49.28} (11.66) & \textbf{51.85} (9.89) &  \textbf{39.38} (10.00) \\ 
        \hline
        ELU & 99.06 (98.24) & 95.41 (\textbf{97.48}) & 80.73 (42.39) & \textbf{45.76} (44.16) \\
        ELU-GPN & \textbf{100.00} (97.86) & 96.56 (96.69) & \textbf{99.37} (43.35) & 43.12 (44.36) \\        
        \hline
        SELU & 99.86 (97.82) & 97.33 (97.38) & 29.23 (46.47) & 29.55 (45.88) \\
        SELU-GPN & \textbf{99.92} (97.91) & 96.97 (\textbf{97.39}) & \textbf{98.24} (47.74) & \textbf{45.90} (45.52) \\        
        \hline
        GELU & 11.24 (12.70) & 11.35 (10.28) & 10.00 (10.43) & 10.00 (10.00) \\
        GELU-GPN & \textbf{97.67} (11.22) & \textbf{95.82} (9.74) & \textbf{90.51} (10.00) & \textbf{36.94} (10.00) \\
        \hline
    \end{tabular}
    \vspace{0.25cm}
    \caption{Accuracy (percentage) of neural networks of depth 200 with different activation functions on real-world data. The numbers in parenthesis denote the results when batch normalization is applied before the activation function.}
    \label{table:real}
    \vspace{-4ex}
\end{small}    
\end{table}
\begin{figure}[t]
\centering
\subfloat[Tanh.]{\includegraphics[width=0.4\textwidth]{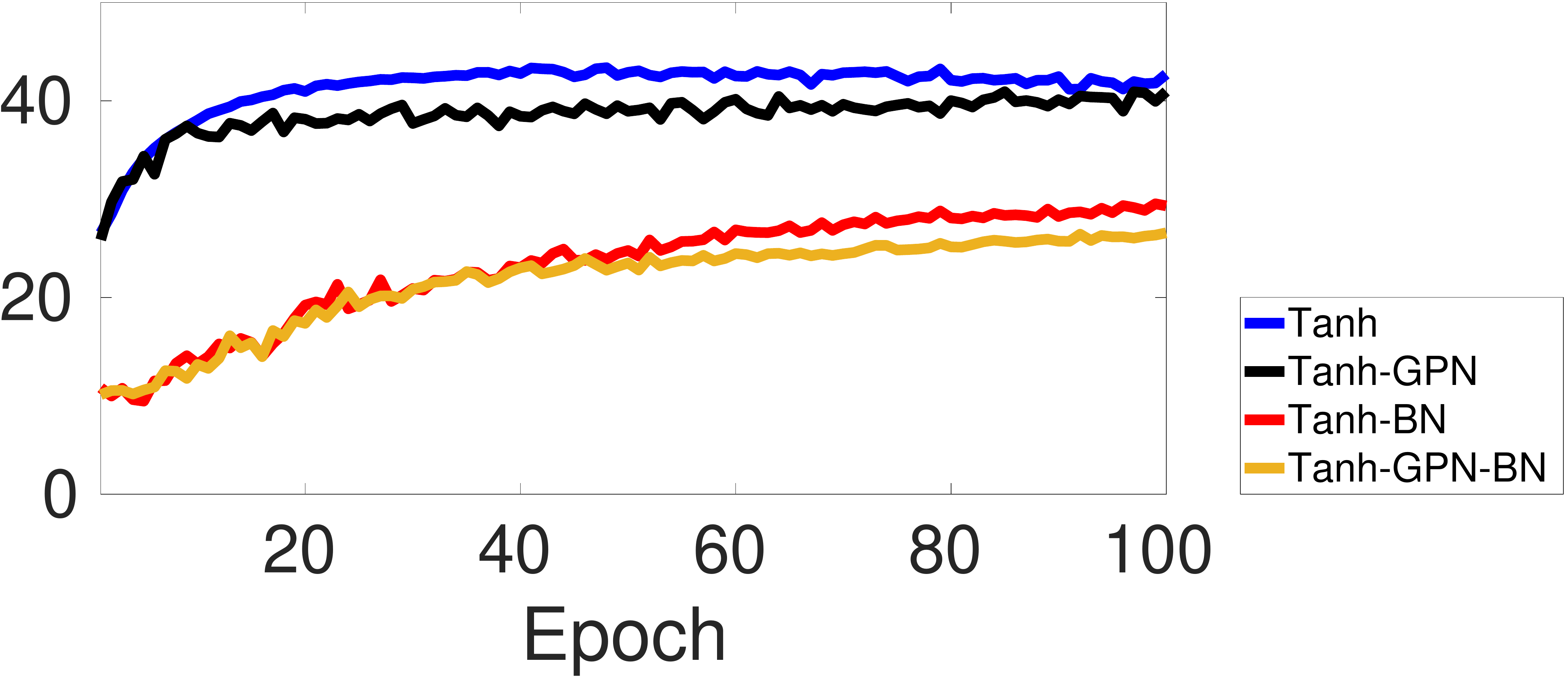}}
\hspace{0.5cm}
\subfloat[SELU.]{\includegraphics[width=0.4\textwidth]{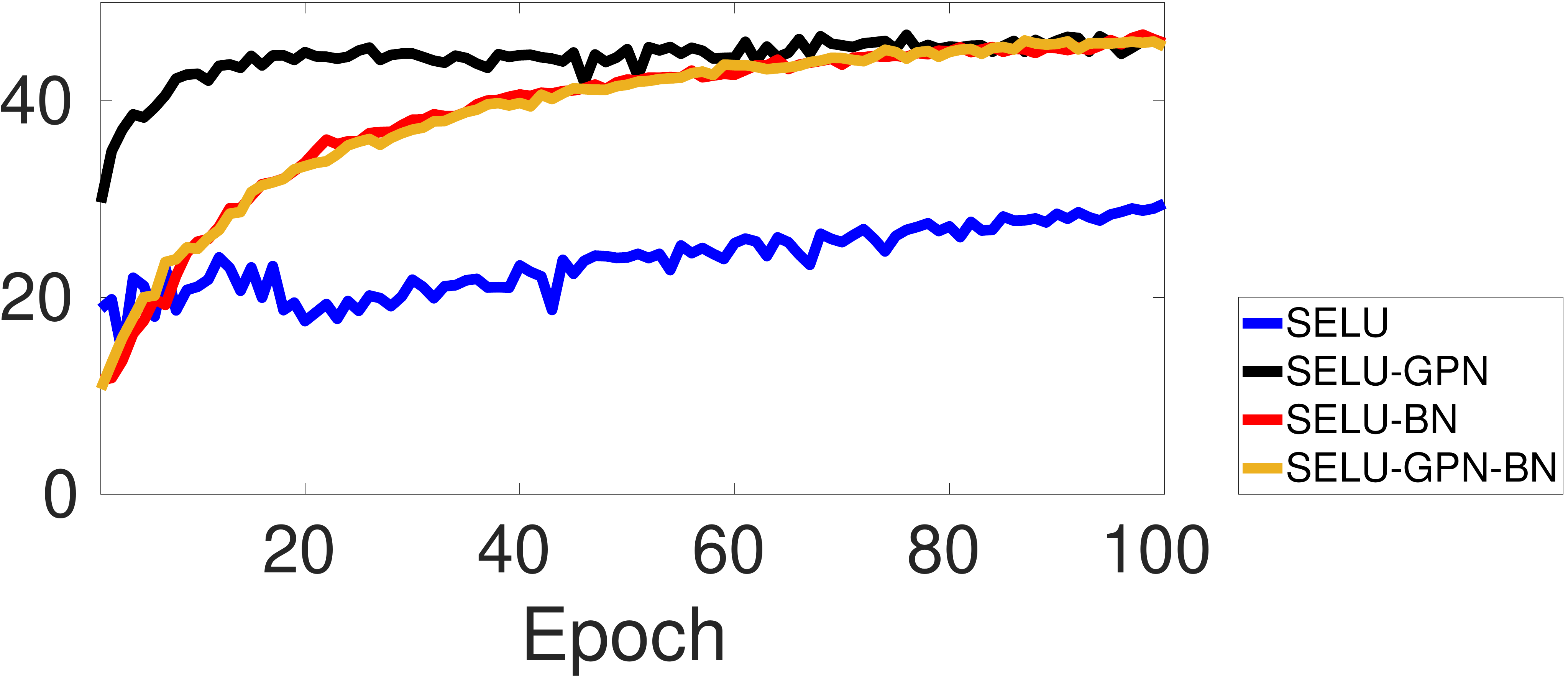}}
\caption{Test accuracy (percentage) during training on CIFAR-10. ``-BN'' denotes that batch normalization is applied before the activation function.}
\label{exp:test_acc_cifar10}

\end{figure}

\clearpage

\section{Discussion}\label{sec:relwork}

We compare our theory to several most relevant theories in literature. 
A key distinguishing feature of our theory is that we provide rigorous proofs of the conditions under which the vanishing/exploding gradients problem disappears. To the best of our knowledge, this is the first time that the problem is provably solved for nonlinear neural networks.

Self-normalizing neural networks enforce zero mean and unit variance for the output of each unit with the SELU activation function~\citep{klambauer2017self}. However, as pointed out in \citep{philipp2018exploding}, only constraining forward signal propagation does not solve the vanishing/exploding gradients problem since the norm of the backward signal can grow or shrink. In \citep{philipp2018exploding} and our experiments, SELU is indeed shown to cause gradient exploding.
To solve the problem, the signal propagation in both directions need to be constrained, as in our theory.

Our theory is largely developed from the deep signal propagation theory~\citep{poole2016exponential,schoenholz2016deep}. 
Both theories require $\mathbb{E}_{x\sim \mathcal{N}(0,1)}[\phi'(x)^2] = 1$. However, ours also requires the quantity $\mathbb{E}_{x\sim \mathcal{N}(0,1)}[\phi(x)^2]$ to be one while in \citep{poole2016exponential,schoenholz2016deep} it can be an arbitrary positive number. We emphasize that it is desirable to enforce $\mathbb{E}_{x\sim \mathcal{N}(0,1)}[\phi(x)^2]=1$ to avoid trivial solutions. For example, if $\phi(x)=\text{Tanh}(\epsilon x)$ with $\epsilon \approx 0$, then $\phi(\epsilon x)\approx \epsilon x$ and the neural network becomes essentially a linear one for which depth is unnecessary (pseudo-linearity \citep{philipp2018exploding}).
This is observed in Figure~\ref{exp:synthetic}~(a).
Moreover, in \citep{poole2016exponential,schoenholz2016deep} the signal propagation analysis is done based on random weights under i.i.d. Gaussian distribution whereas 
we proved how one can solve the vanishing/exploding gradients problem assuming the weight matrices are orthogonal and uniformly distributed under Haar measure.

Dynamical isometry theory~\citep{pennington2017resurrecting} enforces the Jacobian matrix of the input-output function of a neural network to have all singular values close to one. Since the weight matrices are constrained to be orthogonal, it is equivalent to enforce each $\mathbf{D}^{(l)}$ in (\ref{eq:grad_back}) to be close to the identity matrix, which implies the functionality of neural network at initialization is close to an orthogonal matrix (pseudo-linearity). This indeed enables trainability since linear neural networks with orthogonal weight matrices do not suffer from the vanishing/exploding gradients problem. As neural networks need to learn a nonlinear input-output functionality to solve certain tasks, during training the weights of a neural network are unconstrained so that the neural network would move to a nonlinear region where the vanishing/exploding gradients problem might return. 
In our theory, although the orthogonality of weight matrices is also required, we approach the problem from a different perspective. We do not encourage the linearity at initialization.
The neural network can be initialized to be nonlinear and stay nonlinear during the training even when the weights are constrained. This is shown in \textsection\ref{sec:exp_syn}.

\section{Conclusion}

In this paper, we have introduced bidirectionally self-normalizing neural networks (BSNNs) which constrain both forward and backward signal propagation using Gaussian-Poincar\'e normalized activation functions and orthogonal weight matrices.
BSNNs are not restrictive since many commonly used activation functions can be Gaussian-Poincar\'e normalized. We have rigorously proved that the vanishing/exploding gradients problem disappears in BSNNs with high probability under mild conditions.
Experiments on synthetic and real-world data confirm the validity of our theory and demonstrate that BSNNs have excellent trainability without batch normalization.
Currently, the theoretical analysis is limited to same width, fully-connected neural networks. Future work includes extending our theory to more sophisticated networks such as convolutional architectures as well as investigating the generalization capabilities of BSNNs.

\section*{Acknowledgements}
This work was supported by the Australian Research Council Centre of Excellence for Robotic Vision (project number CE140100016). Yao Lu is supported by a Data61/CSIRO scholarship. We thank Elizabeth Meckes for a helpful discussion on random matrix theory.

\clearpage

\bibliographystyle{plainnat}
\bibliography{bsnn}

\clearpage

\setcounter{definition}{0}
\setcounter{proposition}{0}
\setcounter{theorem}{1}
\setcounter{equation}{13}
\setcounter{figure}{4}
\setcounter{table}{2}

\begin{appendices}

\setcounter{definition}{0}
\setcounter{proposition}{0}
\setcounter{theorem}{1}

\section{Proofs}

\begin{proposition} If a neural network is bidirectionally self-normalizing, then
\begin{align}
\Big\|\frac{\partial E}{\partial \mathbf{W}^{(1)}}\Big\|_F = ... = \Big\|\frac{\partial E}{\partial \mathbf{W}^{(L)}}\Big\|_F.
\end{align}
\end{proposition}
\begin{proof}
For each $l$, we have
\begin{align}
\Big\|\frac{\partial E}{\partial \mathbf{W}^{(l)}}\Big\|_F 
&=\sqrt{\text{trace}\Big(\frac{\partial E}{\partial \mathbf{W}^{(l)}}\Big(\frac{\partial E}{\partial \mathbf{W}^{(l)}}\Big)^T\Big)} \\
&=\sqrt{\text{trace}(\mathbf{y}^{(l)}(\mathbf{x}^{(l)})^T\mathbf{x}^{(l)}
(\mathbf{y}^{(l)})^T)} \\
&=\sqrt{\text{trace}((\mathbf{x}^{(l)})^T\mathbf{x}^{(l)}
(\mathbf{y}^{(l)})^T\mathbf{y}^{(l)})} \\
&=\sqrt{(\mathbf{x}^{(l)})^T\mathbf{x}^{(l)}}\sqrt{(\mathbf{y}^{(l)})^T\mathbf{y}^{(l)}} \\
&=\|\mathbf{x}^{(l)}\|_2\|\mathbf{y}^{(l)}\|_2.
\end{align}
By the definition of bidirectional self-normalization, we have $\|\frac{\partial E}{\partial \mathbf{W}^{(1)}}\|_F=...=\|\frac{\partial E}{\partial \mathbf{W}^{(L)}}\|_F$.
\end{proof}

\begin{proposition}
Function $\phi:\mathbb{R}\to\mathbb{R}$ is Gaussian-Poincar\'e normalized and $\mathbb{E}_{x\sim \mathcal{N}(0,1)}[\phi(x)] = 0$ if and only if $\phi(x)=x$ or $\phi(x)=-x$.
\end{proposition}

\begin{proof}
Since $\mathbb{E}_{x\sim\mathcal{N}(0,1)}[\phi(x)^2] < \infty$  and $\mathbb{E}_{x\sim\mathcal{N}(0,1)}[\phi'(x)^2] < \infty$, $\phi(x)$ and $\phi'(x)$ can be expanded in terms of Hermite polynomials. Let the  Hermite polynomial of degree $k$ be
\begin{align}
H_k(x) = \frac{(-1)^k}{\sqrt{k!}}\exp(\frac{x^2}{2})\frac{d^k}{dx^k}\exp(-\frac{x^2}{2})
\end{align}
and due to $H_k'(x) = \sqrt{k} H_{k-1}(x)$, we have
\begin{align}
\phi(x) &= \sum_{k=0}^\infty a_k H_k(x), \\
\phi'(x) &= \sum_{k=1}^\infty \sqrt{k}a_k H_{k-1}(x).
\end{align}

Since $\mathbb{E}_{x\sim \mathcal{N}(0,1)}[\phi(x)] = 0$, we have 
\begin{align}
a_0 &= \mathbb{E}_{x\sim \mathcal{N}(0,1)}[H_0(x)\phi(x)] \\
&= \mathbb{E}_{x\sim \mathcal{N}(0,1)}[\phi(x)] \\
&= 0.
\end{align}

Since
\begin{align}
\mathbb{E}_{x\sim\mathcal{N}(0,1)}[\phi(x)^2]  = 
\mathbb{E}_{x\sim\mathcal{N}(0,1)}[\phi'(x)^2]  = 1
\end{align}
and Hermite polynomials are orthonormal, we have
\begin{align}
\mathbb{E}_{x\sim\mathcal{N}(0,1)}[\phi(x)^2] = \sum_{k=1}^\infty a_k^2 = \mathbb{E}_{x\sim\mathcal{N}(0,1)}[\phi'(x)^2] =  \sum_{k=1}^\infty k a_k^2 = 1.
\end{align}
Therefore, we have 
\begin{align}
\sum_{k=1}^\infty k a_k^2 - \sum_{k=1}^\infty a_k^2  = 0
\end{align}
that is
\begin{align}
\sum_{k=2}^\infty (k - 1)a_k^2 = 0.
\end{align}
Since each term in $\sum_{k=2}^\infty (k - 1)a_k^2$ is nonnegative, the only solution is $a_k = 0$ for $k \geq 2$. And since $\mathbb{E}_{x\sim\mathcal{N}(0,1)}[\phi(x)^2]  = a_1^2 = 1$, we have $a_1 = \pm 1$. Hence, $\phi(x) = \pm H_1(x) = \pm x$.
\end{proof}

\begin{proposition}
For any differentiable function $\phi:\mathbb{R}\to\mathbb{R}$ with non-zero and bounded $\mathbb{E}_{x\sim \mathcal{N}(0,1)}[\phi(x)^2]$ and  $\mathbb{E}_{x\sim \mathcal{N}(0,1)}[\phi'(x)^2]$, 
there exist two constants $a$ and $b$ such that $a\phi(x)+b$ is Gaussian-Poincar\'e normalized.
\end{proposition}

\begin{proof}
Let $\varphi(x) = \phi(x) + c$. Then let
\begin{align}
\psi(c) &= \mathbb{E}_{x\sim \mathcal{N}(0,1)}[\varphi(x)^2]
- \mathbb{E}_{x\sim \mathcal{N}(0,1)}[(\phi'(x))^2] \\
&= \mathrm{Var}_{x\sim \mathcal{N}(0,1)}[\varphi(x)] + (\mathbb{E}_{x\sim \mathcal{N}(0,1)}[\varphi(x)])^2 - \mathbb{E}_{x\sim \mathcal{N}(0,1)}[(\phi'(x))^2] \\
&= \mathrm{Var}_{x\sim \mathcal{N}(0,1)}[\phi(x)] + (\mathbb{E}_{x\sim \mathcal{N}(0,1)}[\phi(x)]+c)^2 - \mathbb{E}_{x\sim \mathcal{N}(0,1)}[(\phi'(x))^2].
\end{align}
Therefore, $\psi(c)$ is a quadratic function of $c$. We also have $\psi(c) > 0$ as $c\to\infty$ and $\psi(-\mathbb{E}_{x\sim \mathcal{N}(0,1)}[\phi(x)]) \leq 0$ due to Gaussian-Poincar\'e inequality. Hence, there exists $c$ for which $\psi(c) = 0$ such that $\mathbb{E}_{x\sim \mathcal{N}(0,1)}[(\phi(x)+c)^2] = \mathbb{E}_{x\sim \mathcal{N}(0,1)}[\phi'(x)^2]$. Let $a = (\mathbb{E}_{x\sim \mathcal{N}(0,1)}[\phi'(x)^2])^{-1/2}$ and $b=ac$, we have  $\mathbb{E}_{x\sim \mathcal{N}(0,1)}[(a\phi(x)+b)^2] = \mathbb{E}_{x\sim \mathcal{N}(0,1)}[(a\phi'(x))^2]$ = 1.
\end{proof}

The proof is largely due to \citep{eldredge2020} with minor modification in here.

\paragraph{Assumptions.}
\begin{tight_enumerate}
    \item \emph{Random vector $\mathbf{x}\in\mathbb{R}^d$ is thin-shell concentrated.}
    \item \emph{Random orthogonal matrix $\mathbf{W}=(\mathbf{w}_1,...,\mathbf{w}_d)^T$ is uniformly distributed.}
    \item \emph{Function $\phi: \mathbb{R}\to \mathbb{R}$ is Gaussian-Poincar\'e normalized.}
    \item \emph{Function $\phi$ and its derivative are Lipschitz continuous.}
\end{tight_enumerate}

\begin{theorem}[\textbf{Forward Norm-Preservation}]
Random vector 
\begin{align}
(\phi(\mathbf{w}^T_1\mathbf{x}),...,\phi(\mathbf{w}^T_d\mathbf{x}))
\end{align}
is thin-shell concentrated.
\end{theorem}

\begin{theorem}[\textbf{Backward Norm-Preservation}]\label{thm:back}
Let $\mathbf{D} = \textup{diag}(\phi'(\mathbf{w}_1^T\mathbf{x}),...,\phi'(\mathbf{w}_d^T\mathbf{x}))$ and $\mathbf{y} \in \mathbb{R}^d$ be a fixed vector with bounded $\|\mathbf{y}\|_\infty$.
Then for any $\epsilon > 0$
\begin{align}
\mathbb{P}\Big\{\frac{1}{d}\Big| \|\mathbf{D}\mathbf{y}\|_2^2  - \|\mathbf{y}\|_2^2 \Big| \geq \epsilon \Big\} \to 0
\end{align}
as $d \to \infty$.
\end{theorem}

\paragraph{Notations.}
$\mathbb{S}^{d-1} = \{\mathbf{x}\in\mathbb{R}^d: \|\mathbf{x}\|_2=1\}$. $\mathbb{O}(d)$ is the orthogonal matrix group of size $d$. $\mathbf{1}_{\{\cdot\}}$ denotes the indicator function. $\mathbf{0}_d$ denotes the vector of dimension $d$ and all elements equal to zero. $\mathbf{I}_d$ denotes the identity matrix of size $d\times d$.

\begin{lemma}
If random variable $x\sim \mathcal{N}(0,1)$ and function $f:\mathbb{R}\to\mathbb{R}$ is Lipschitz continuous, then random variable $f(x)$ is sub-gaussian.
\end{lemma}

\begin{proof}
Due to the Gaussian concentration theorem (Theorem 5.2.2 in \citep{vershynin2018high}), we have 
\begin{align}
\|f(x) - \mathbb{E}[f(x)]\|_{\psi_2} \leq CK 
\end{align}
where $\|\cdot\|_{\psi_2}$ denotes sub-gaussian norm, $C$ is a constant and $K$ is the Lipschitz constant of $f$. This implies $f(x) - \mathbb{E}[f(x)]$ is sub-gaussian (Proposition 2.5.2 in \citep{vershynin2018high}). Therefore $f(x)$ is sub-gaussian (Lemma 2.6.8 in \citep{vershynin2018high}).
\end{proof}

\begin{lemma}
Let $\mathbf{x}=(x_1,...,x_d)\in\mathbb{R}^{d}$ be a random vector that each coordinate $x_i$ is independent and sub-gaussian and $\mathbb{E}[x_i^2] = 1$. Let $\mathbf{y}=(y_1,...,y_d)\in\mathbb{R}^d$ be a fixed vector with bounded $\|\mathbf{y}\|_{\infty}$. Then
\begin{align}
\mathbb{P}\Big\{\frac{1}{d}\Big| \sum_ix_i^2y_i^2 - \sum_i y_i^2 \Big| \geq \epsilon \Big\} \to 0
\end{align}
as $d\to\infty$.
\end{lemma}

\begin{proof}
Since $y_ix_i$ is sub-gaussian, then $y_i^2x_i^2$ is sub-exponential (Lemma 2.7.6 in \citep{vershynin2018high}). Since $\mathbb{E}[y_i^2x_i^2] = y_i^2\mathbb{E}[x_i^2] = y_i^2$, $y_i^2x_i^2 - y_i^2$ is sub-exponential with zero mean (Exercise 2.7.10 in \citep{vershynin2018high}). 
Applying Bernstein’s inequality (Corollary 2.8.3 in \citep{vershynin2018high}), we proved the lemma.
\end{proof}

\begin{lemma}
Let $\mathbf{z}\sim\mathcal{N}(\mathbf{0}_d,\mathbf{I}_d)$. Then for  any $0 < \delta < 1$
\begin{align}
\mathbb{P}\{\mathbf{z}\in \mathbb{R}^d : (1-\delta)\sqrt{d}\leq\|\mathbf{z}\|_2\leq (1+\delta)\sqrt{d}\} \geq 1 - 2\exp(-d\delta^2).
\end{align}
\end{lemma}

\begin{proof}
See \citep{alberts2018} (Theorem 1.2).
\end{proof}

\begin{lemma}
Let $\mathbf{z}\sim\mathcal{N}(\mathbf{0}_d,\mathbf{I}_d)$. Then $\mathbf{z}/\|\mathbf{z}\|_2$ is uniformly distributed on $\mathbb{S}^{d-1}$.

\end{lemma}
\begin{proof}
See \citep{dawkins}.
\end{proof}

\begin{lemma}
Let $\mathbf{z}=(z_1,...,z_d)\sim\mathcal{N}(\mathbf{0}_d,\mathbf{I}_d)$, $\mathbf{a}=(a_1,...,a_d)$ be a fixed vector with bounded $\|\mathbf{a}\|_\infty$ and $f:\mathbb{R}\to\mathbb{R}$ be a continuous function. Then for any $\epsilon>0$
\begin{align}
\mathbb{P}\Big\{ \frac{1}{d}  \Big| \sum_i a_if(\sqrt{d}/\|\mathbf{z}\|_2z_i) - \sum_i a_i f(z_i) \Big| \geq \epsilon \Big\} \to 0
\end{align}
as $d\to\infty$.
\end{lemma}

\begin{proof}

Since
\begin{align}
\frac{1}{d}  \Big| \sum_i a_i f(\sqrt{d}/\|\mathbf{z}\|_2z_i) - \sum_i a_i f(z_i) \Big| \leq \frac{1}{d} \sum_i |a_i| \cdot |  f(\sqrt{d}/\|\mathbf{z}\|_2z_i) - f(z_i) |,
\end{align}
if, as $d\to\infty$,
\begin{align}
\mathbb{P}\Big\{ \frac{1}{d} \sum_i |a_i| \cdot |  f(\sqrt{d}/\|\mathbf{z}\|_2z_i) - f(z_i) | \geq \epsilon \Big\} \to 0,
\end{align}
then
\begin{align}
\mathbb{P}\Big\{ \frac{1}{d}  \Big| \sum_i a_if(\sqrt{d}/\|\mathbf{z}\|_2z_i) - \sum_i a_i f(z_i) \Big| \geq \epsilon \Big\} \to 0.
\end{align}

For $0 < \delta < 1$, let
\begin{align}
A &= \Big\{\mathbf{z}\in\mathbb{R}^d: \frac{1}{d} \sum_i |a_i| \cdot |  f(\sqrt{d}/\|\mathbf{z}\|_2z_i) - f(z_i) | \geq \epsilon \Big\}, \\
\mathcal{U}_\delta &= \Big\{\mathbf{z}\in \mathbb{R}^d : (1-\delta)\sqrt{d}\leq\|\mathbf{z}\|_2\leq (1+\delta)\sqrt{d}\Big\}.
\end{align}
Then
\begin{align}
\mathbb{P}\Big\{ \frac{1}{d} \sum_i |a_i| \cdot|  f(\sqrt{d}/\|\mathbf{z}\|_2z_i) - f(z_i) | \geq \epsilon \Big\}  &= \int_{\mathbb{R}^d}\mathbf{1}_{\{\mathbf{z}\in A\}} \gamma(\mathbf{z}) d\mathbf{z} \\
&= \int_{\mathbb{R}^d\setminus \mathcal{U}_\delta}\mathbf{1}_{\{\mathbf{z}\in A\}} \gamma(\mathbf{z})d\mathbf{z}  + \int_{\mathcal{U}_\delta}\mathbf{1}_{\{\mathbf{z}\in A\}} \gamma(\mathbf{z})d\mathbf{z}
\end{align}
where $\gamma(\mathbf{z})$ denotes the density function of $\mathbf{z}$.

Let $\delta=d^{-1/4}$. From Lemma 3, we have, as $d\to\infty$,
\begin{align}
\int_{\mathbb{R}^d\setminus \mathcal{U}_\delta}\mathbf{1}_{\{\mathbf{z}\in A\}} \gamma(\mathbf{z})d\mathbf{z}  \leq \int_{\mathbb{R}^d\setminus \mathcal{U}_\delta} \gamma(\mathbf{z})d\mathbf{z} = 1 -\mathbb{P}\{ \mathbf{z}\in \mathcal{U}_\delta  \} \leq 2\exp(-d\delta^2) \to 0.
\end{align}

For $\mathbf{z}\in \mathcal{U}_\delta$ and $\delta=d^{-1/4}$, we have $\|\mathbf{z}\|_2\to\sqrt{d}$, $\sqrt{d}/\|\mathbf{z}\|_2z_i\to z_i$ and therefore $f(\sqrt{d}/\|\mathbf{z}\|_2z_i)\to f(z_i)$ as $d\to\infty$. Since $|a_i|$ is bounded, we have $\frac{1}{d} \sum_i |a_i| \cdot|  f(\sqrt{d}/\|\mathbf{z}\|_2z_i) - f(z_i) |\to 0$ and therefore
$\int_{\mathcal{U}_\delta}\mathbf{1}_{\{\mathbf{z}\in A\}} d\mathbf{z} \to 0$, as $d\to\infty$.

\end{proof}

\begin{lemma}
Let random matrix $\mathbf{W}$ be uniformly distributed on $\mathbb{O}(d)$ random vector $\bm{\theta}$ be uniformly distributed on $\mathbb{S}^{d-1}$ and random vector $\mathbf{x}\in\mathbb{R}^d$ be thin-shell concentrated. Then $\mathbf{Wx}\to\sqrt{d}\bm{\theta}$ as $d\to\infty$.
\end{lemma}

\begin{proof}
Let $\mathbf{y}\in\mathbb{R}^d$ be any vector with $\|\mathbf{y}\|_2=\sqrt{d}$ and $\mathbf{a}=(\sqrt{d},0,...,0)\in\mathbb{R}^d$. 
Since $\mathbf{W}$ is uniformly distributed, $\mathbf{Wy}$ has the same distribution as $\mathbf{Wa}$. $\mathbf{Wa}$ is the first column of $\sqrt{d}\mathbf{W}$, which is equivalent to random vector $\sqrt{d}\bm{\theta}$ \citep{meckes2019random}. Since $\mathbf{x}$ is thin-shell concentrated, $\mathbf{x} \to \sqrt{d}/\|\mathbf{x}\|_2\mathbf{x} = \mathbf{y}$ and therefore $\mathbf{Wx}\to\sqrt{d}\bm{\theta}$ as $d\to\infty$.

\end{proof}

\begin{proof}[Proof of Theorem 2]
Let $\mathbf{z}=(z_1,...,z_d)\sim \mathcal{N}(\mathbf{0}_d,\mathbf{I}_d)$. Due to Lemma 1, random variable $\phi(z_i)$ is sub-gaussian. 
Since $\phi$ is Gaussian-Poincar\'e normalized, $\mathbb{E}_{z_i\sim\mathcal{N}(0,1)}[\phi(z_i)^2]=1$. 
Applying Lemma 2 with each $y_i=1$, we have for $\epsilon > 0$
\begin{align}
\mathbb{P}\Big\{\Big| \frac{1}{d}\sum_i\phi(z_i)^2 - 1 \Big| \geq \epsilon \Big\} \to 0
\end{align}
as $d\to\infty$.

Due to Lemma 4 and 5 (with each $a_i = 1$), for random vector $\bm{\theta}=(\theta_1,...,\theta_d)$ uniformly distributed on $\mathbb{S}^{d-1}$,  we have
\begin{align}
\mathbb{P}\Big\{\Big|\frac{1}{d}\sum_i\phi(\sqrt{d}\theta_i)^2 - \frac{1}{d}\sum_i\phi(z_i)^2\Big| \geq \epsilon\Big\} \to 0
\end{align}
and therefore
\begin{align}
\mathbb{P}\Big\{\Big|\frac{1}{d}\sum_i\phi(\sqrt{d}\theta_i)^2 - 1\Big| \geq \epsilon\Big\} \to 0
\end{align}
as $d\to\infty$.

Then from Lemma 6, we have $\mathbf{Wx}\to\sqrt{d}\bm{\theta}$ and therefore
\begin{align}
\mathbb{P}\Big\{\Big|\frac{1}{d}\sum_i\phi(\mathbf{w}_i^T\mathbf{x})^2 - 1\Big| \geq \epsilon\Big\} \to 0
\end{align}
as $d\to\infty$.
\end{proof}

\begin{proof}[Proof of Theorem 3]
Let $\mathbf{z}=(z_1,...,z_d)\sim \mathcal{N}(\mathbf{0}_d,\mathbf{I}_d)$. Due to Lemma 1, random variable $\phi'(z_i)$ is sub-gaussian. 
Since $\phi$ is Gaussian-Poincar\'e normalized, $\mathbb{E}_{z_i\sim\mathcal{N}(0,1)}[\phi'(z_i)^2]=1$. Applying Lemma 2, we have for $\epsilon > 0$
\begin{align}
\mathbb{P}\Big\{\frac{1}{d}\Big| \sum_i y_i^2\phi'(z_i)^2 - y_i^2 \Big| \geq \epsilon \Big\} \to 0
\end{align}
as $d\to\infty$.

Due to Lemma 4 and 5 (with each $a_i = y_i^2$), for random vector $\bm{\theta}=(\theta_1,...,\theta_d)$ uniformly distributed on $\mathbb{S}^{d-1}$, we have
\begin{align}
\mathbb{P}\Big\{\Big|\frac{1}{d}\sum_iy_i^2\phi'(\sqrt{d}\theta_i)^2 - \frac{1}{d}\sum_iy_i^2\phi'(z_i)^2\Big| \geq \epsilon\Big\} \to 0
\end{align}
and therefore
\begin{align}
\mathbb{P}\Big\{\Big|\frac{1}{d}\sum_iy_i^2\phi'(\sqrt{d}\theta_i)^2 - y_i^2\Big| \geq \epsilon\Big\} \to 0
\end{align}
as $n\to\infty$.

Then from Lemma 6, we have $\mathbf{Wx}\to\sqrt{d}\bm{\theta}$ and therefore
\begin{align}
\mathbb{P}\Big\{\Big|\frac{1}{d}\sum_iy_i^2\phi'(\mathbf{w}_i^T\mathbf{x})^2 - y_i^2\Big| \geq \epsilon\Big\} \to 0
\end{align}
as $d\to\infty$.

\end{proof}

\newpage

\section{Additional Experiments}

Due to the space limitation, we only showed the experiments with Tanh and SELU activation functions in the main text. In this section, we show the experiments with ReLU, LeakyReLU, ELU and SELU. Additionally, we also measure the magnitude of the vanishing/exploding gradients during training on the real-world data.

\subsection{Synthetic Data}

In Figure \ref{exp:synthetic_2} and \ref{exp:synthetic_3}, we show the experiments in addition to Figure \myref{1}. In Figure \ref{exp:D_hist_2}, we show the experiments in addition to Figure \myref{2}. In Figure \ref{exp:synthetic_width_2}, we show the experiments in addition to Figure \myref{3}. ELU shows similar behaviors as Tanh since $\text{ELU}(x)\approx x$ for $x\approx 0$.

\begin{figure}[h!]
\centering

\subfloat[$\|\mathbf{x}^{(l)}\|_2^2/d$, ReLU.]{\includegraphics[width=0.4\textwidth]{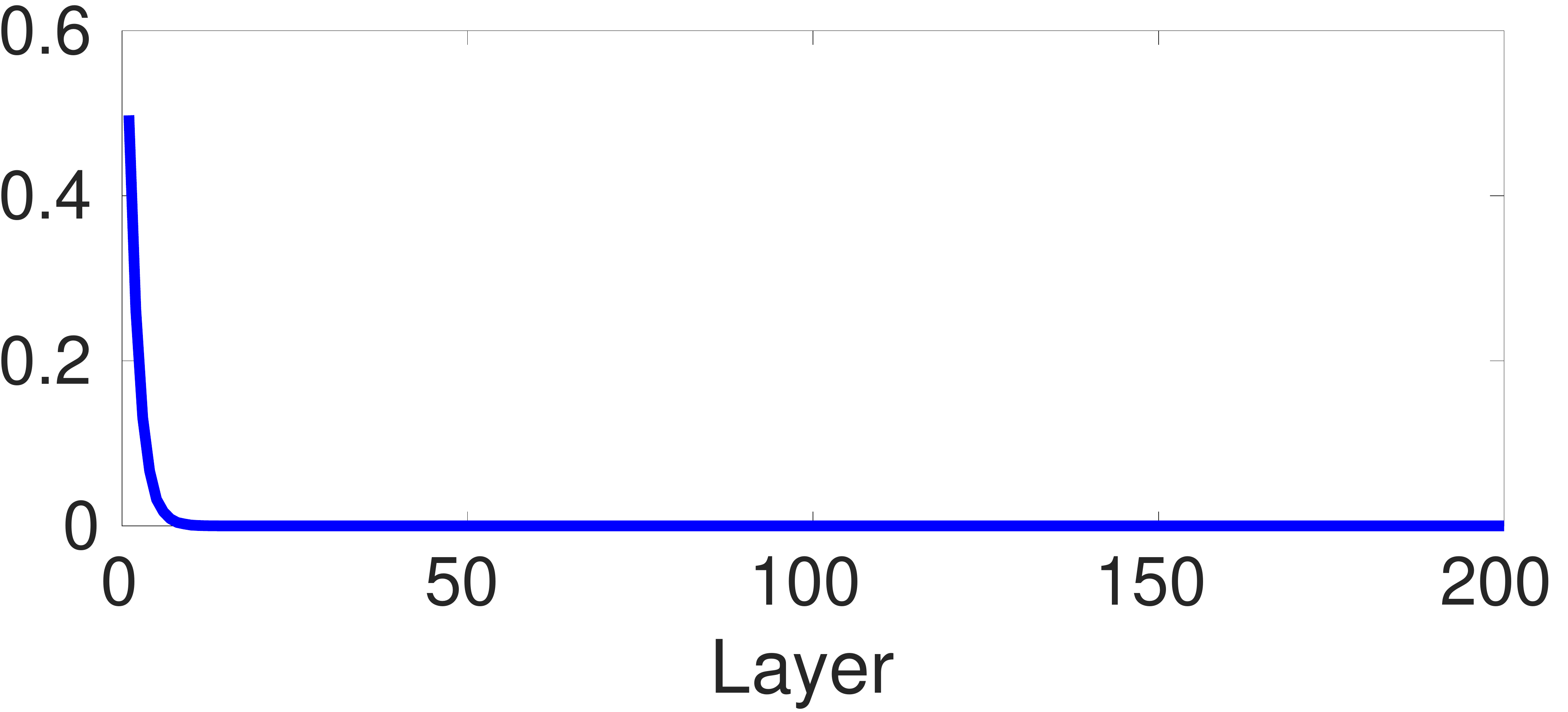}}
\hspace{0.5cm}
\subfloat[$\|\frac{\partial E}{\partial \mathbf{W}^{(l)}}\|_F$, ReLU.]{\includegraphics[width=0.4\textwidth,height=0.205\textwidth,trim=0 0 0 -0.3cm]{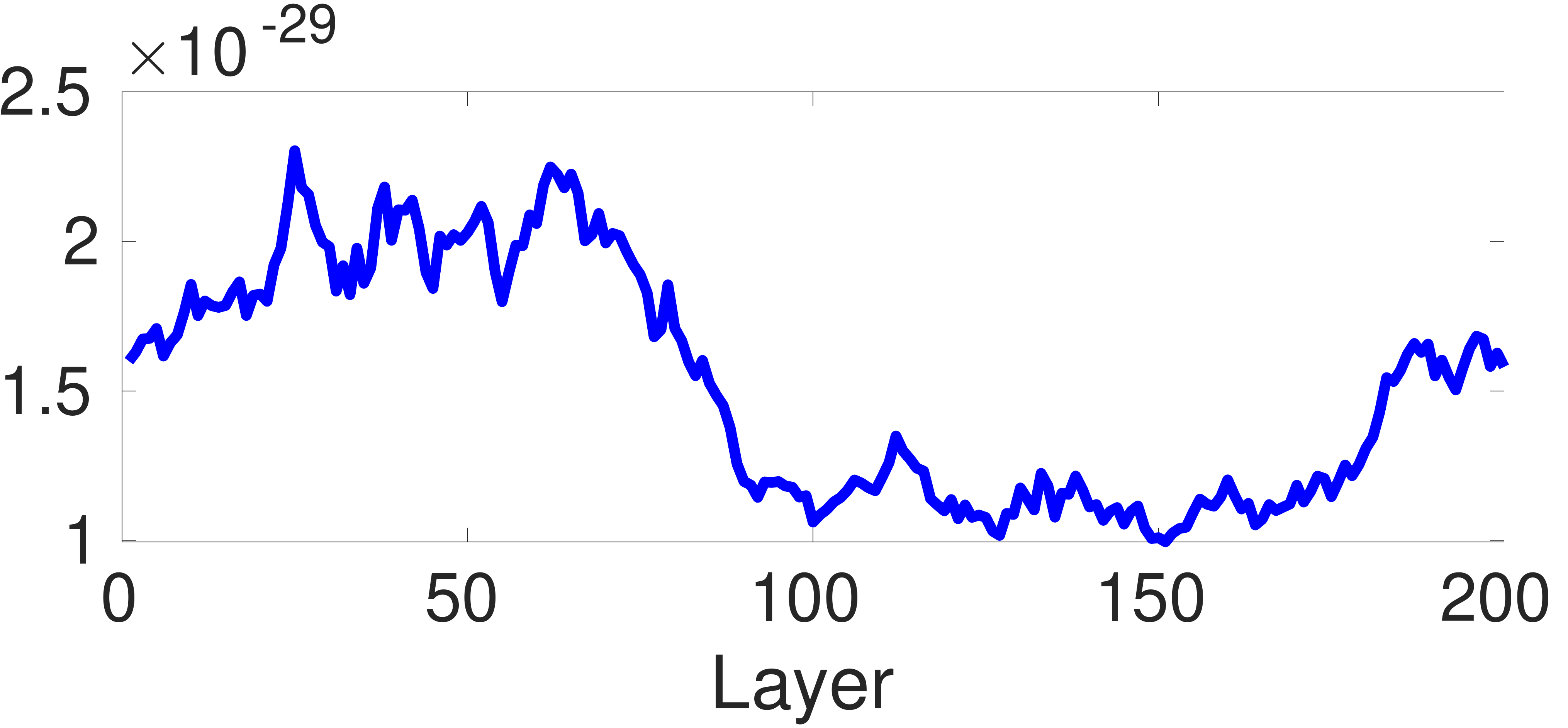}}

\subfloat[$\|\mathbf{x}^{(l)}\|_2^2/d$, ReLU-GPN.]{\includegraphics[width=0.4\textwidth]{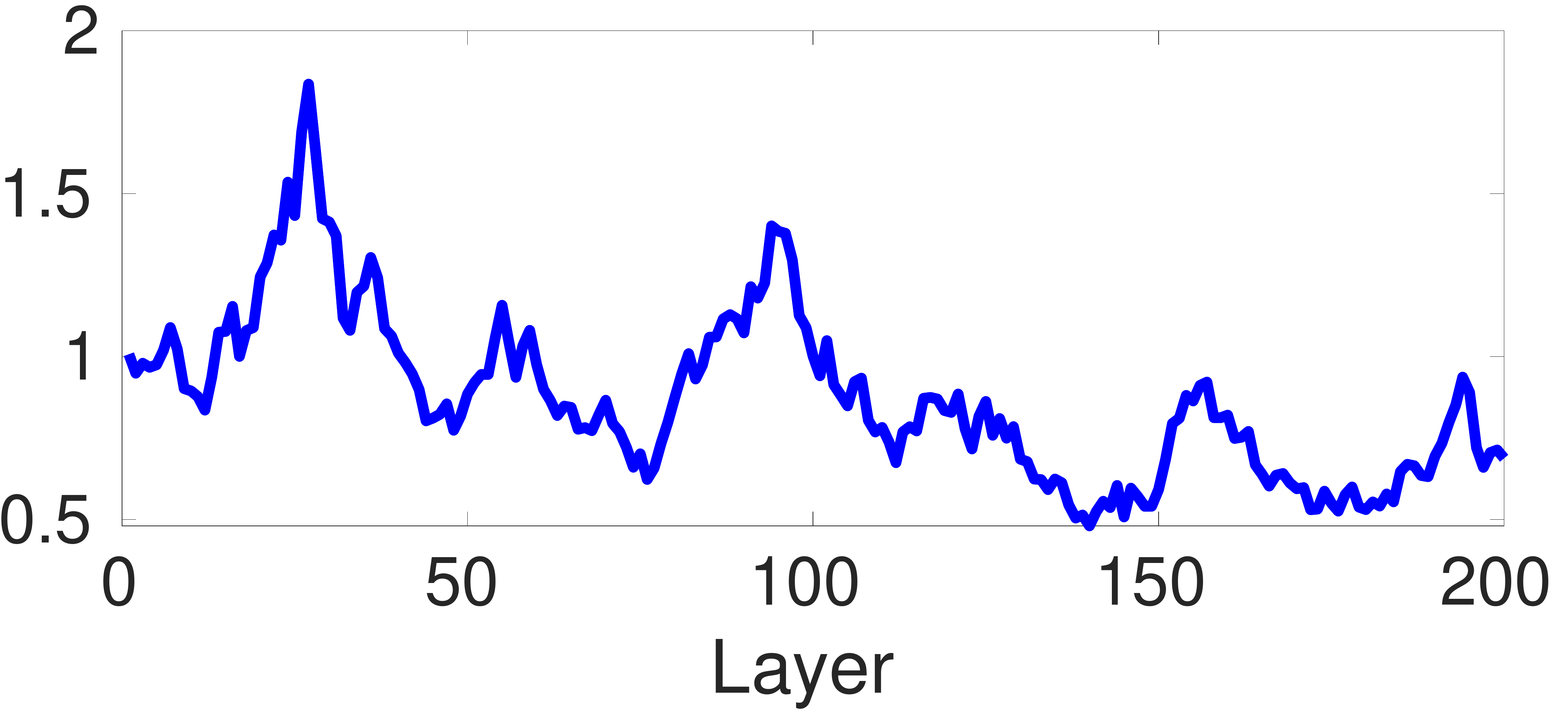}}
\hspace{0.5cm}
\subfloat[$\|\frac{\partial E}{\partial \mathbf{W}^{(l)}}\|_F$, ReLU-GPN.]{\includegraphics[width=0.4\textwidth]{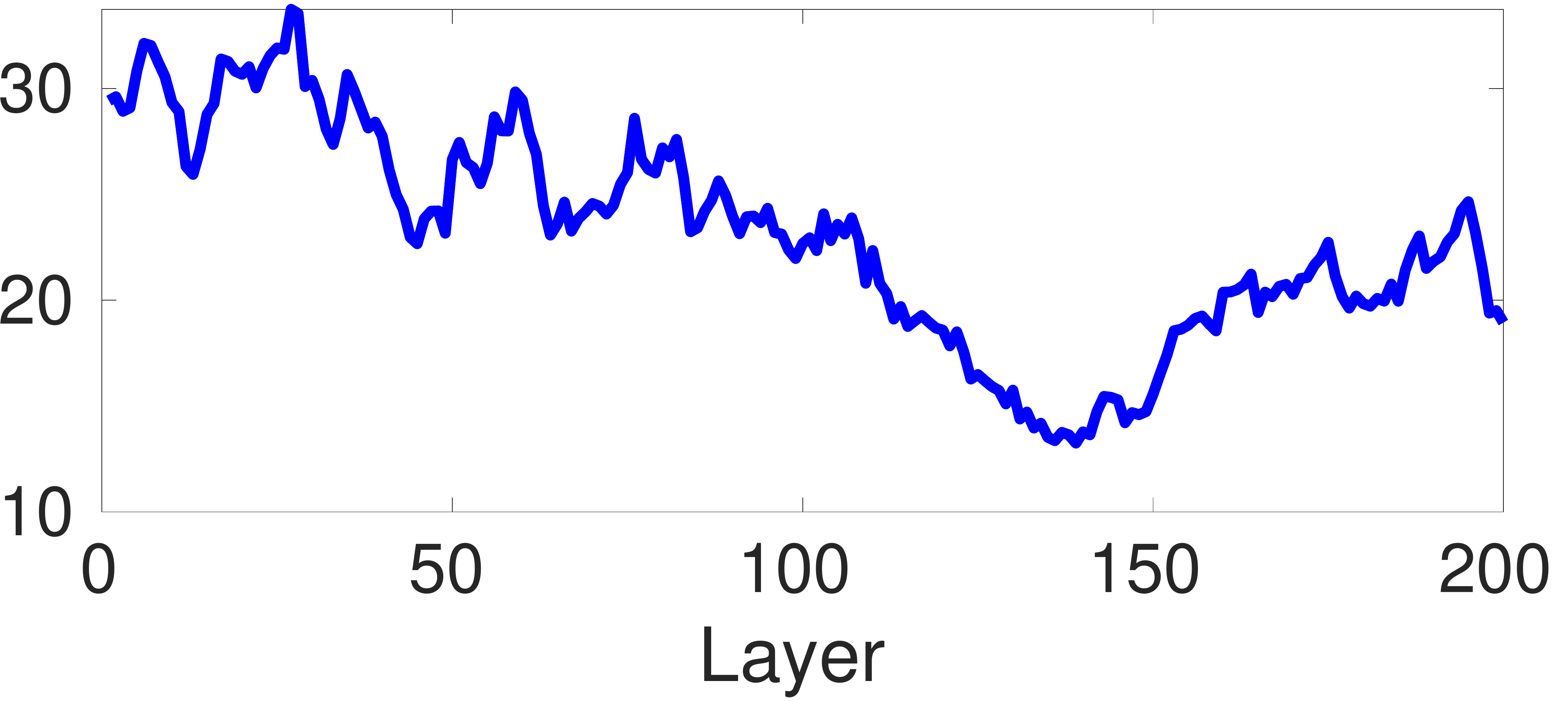}}

\subfloat[$\|\mathbf{x}^{(l)}\|_2^2/d$,
LeakyReLU.]{\includegraphics[width=0.4\textwidth]{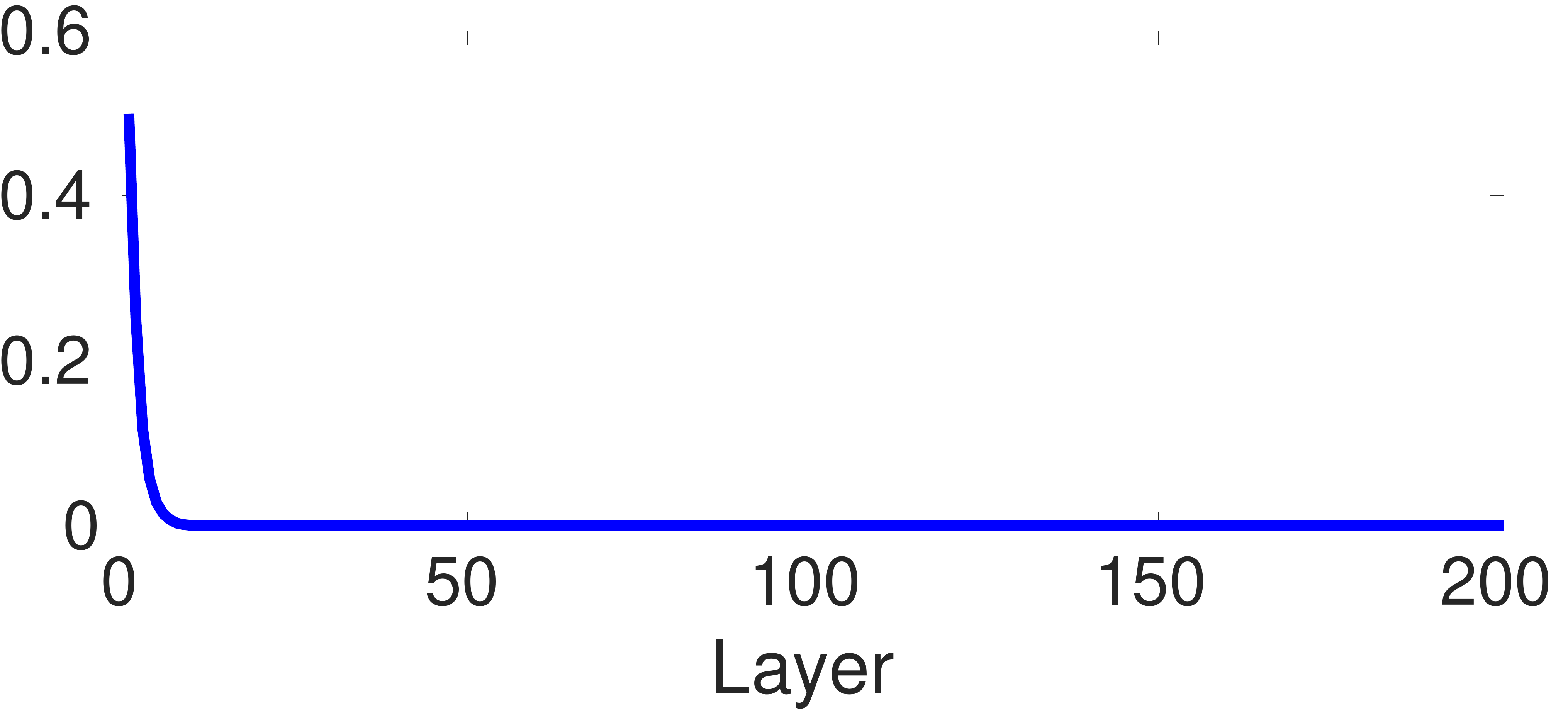}}
\hspace{0.5cm}
\subfloat[$\|\frac{\partial E}{\partial \mathbf{W}^{(l)}}\|_F$, LeakyReLU.]{\includegraphics[width=0.4\textwidth,height=0.205\textwidth,trim=0 0 0 -0.3cm]{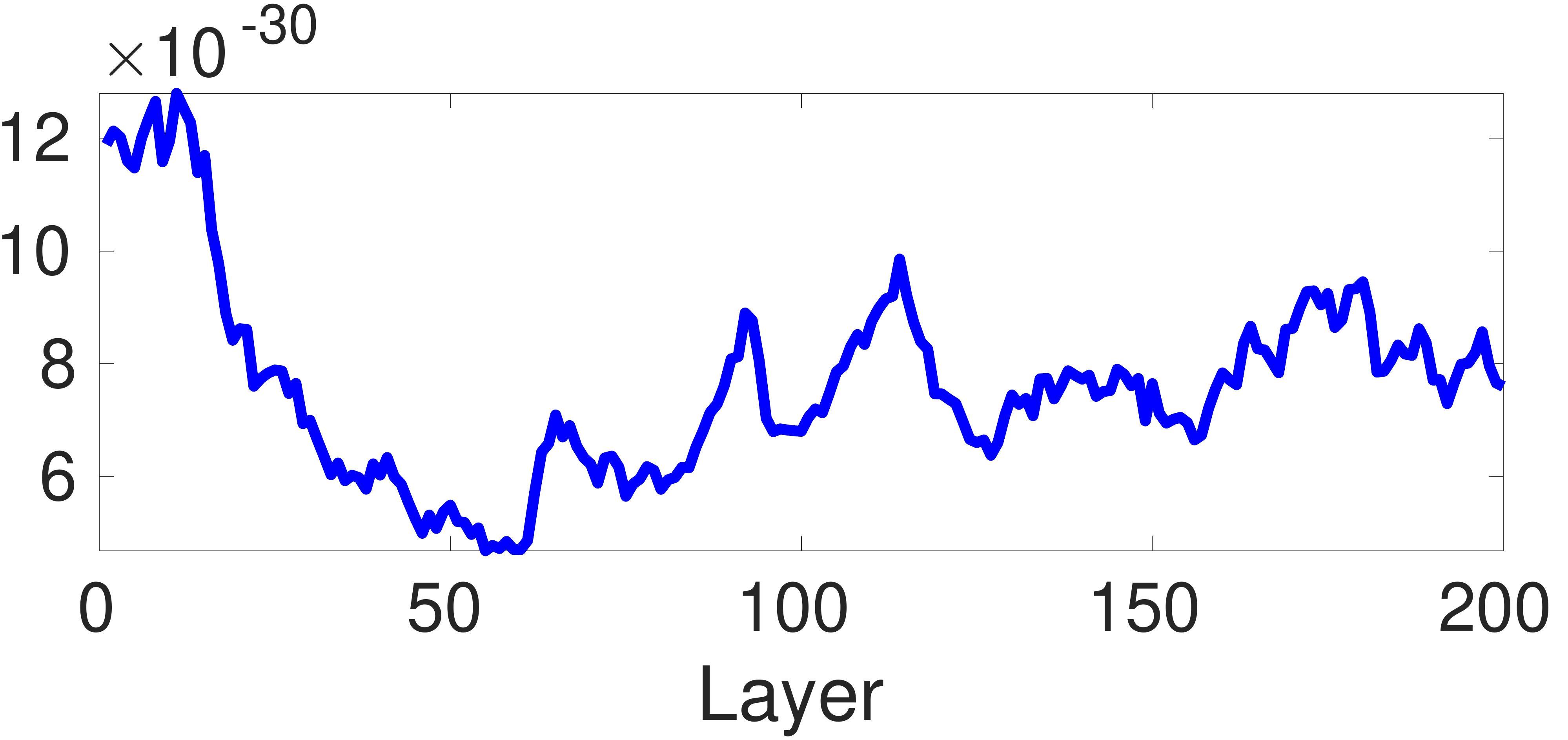}}

\subfloat[$\|\mathbf{x}^{(l)}\|_2^2/d$, LeakyReLU-GPN.]{\includegraphics[width=0.4\textwidth]{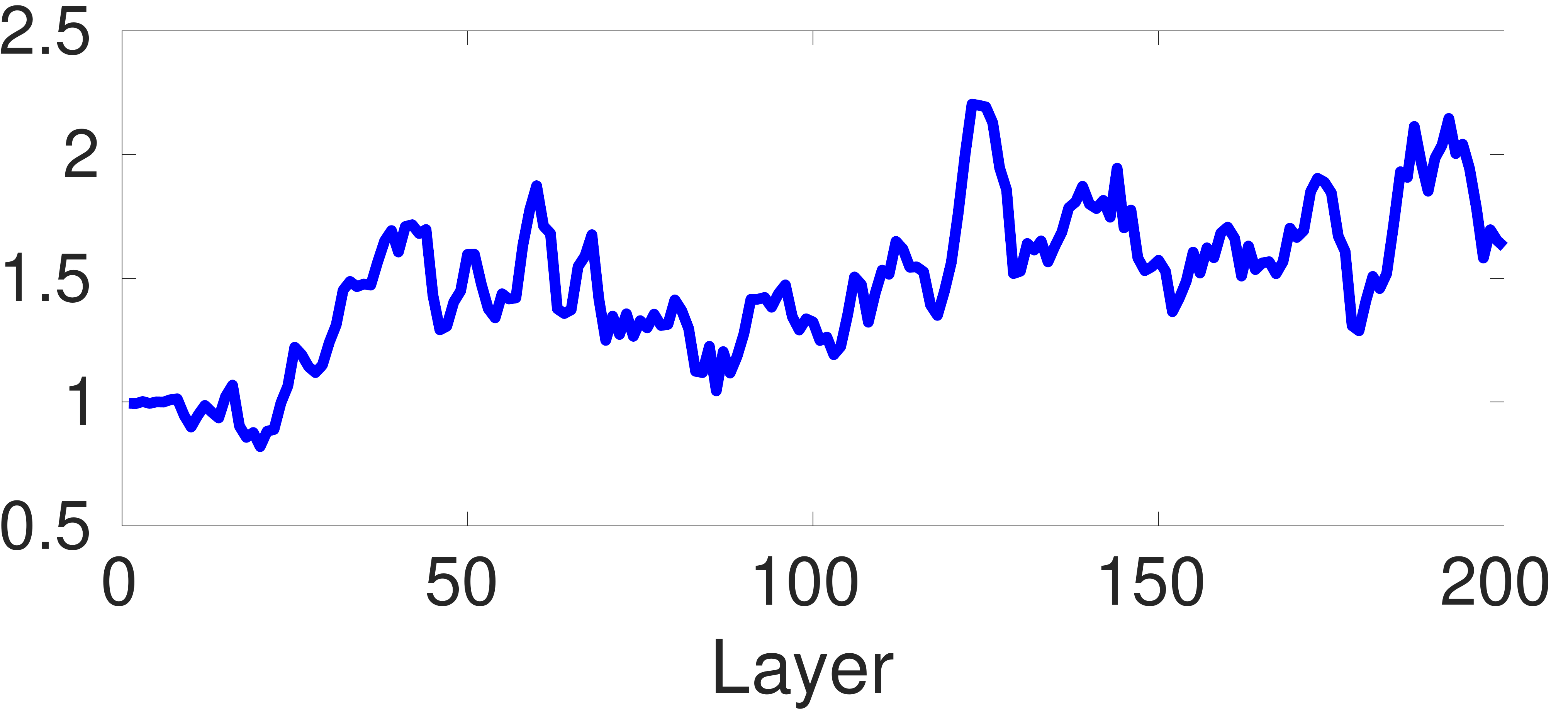}}
\hspace{0.5cm}
\subfloat[$\|\frac{\partial E}{\partial \mathbf{W}^{(l)}}\|_F$, LeakyReLU-GPN.]{\includegraphics[width=0.4\textwidth]{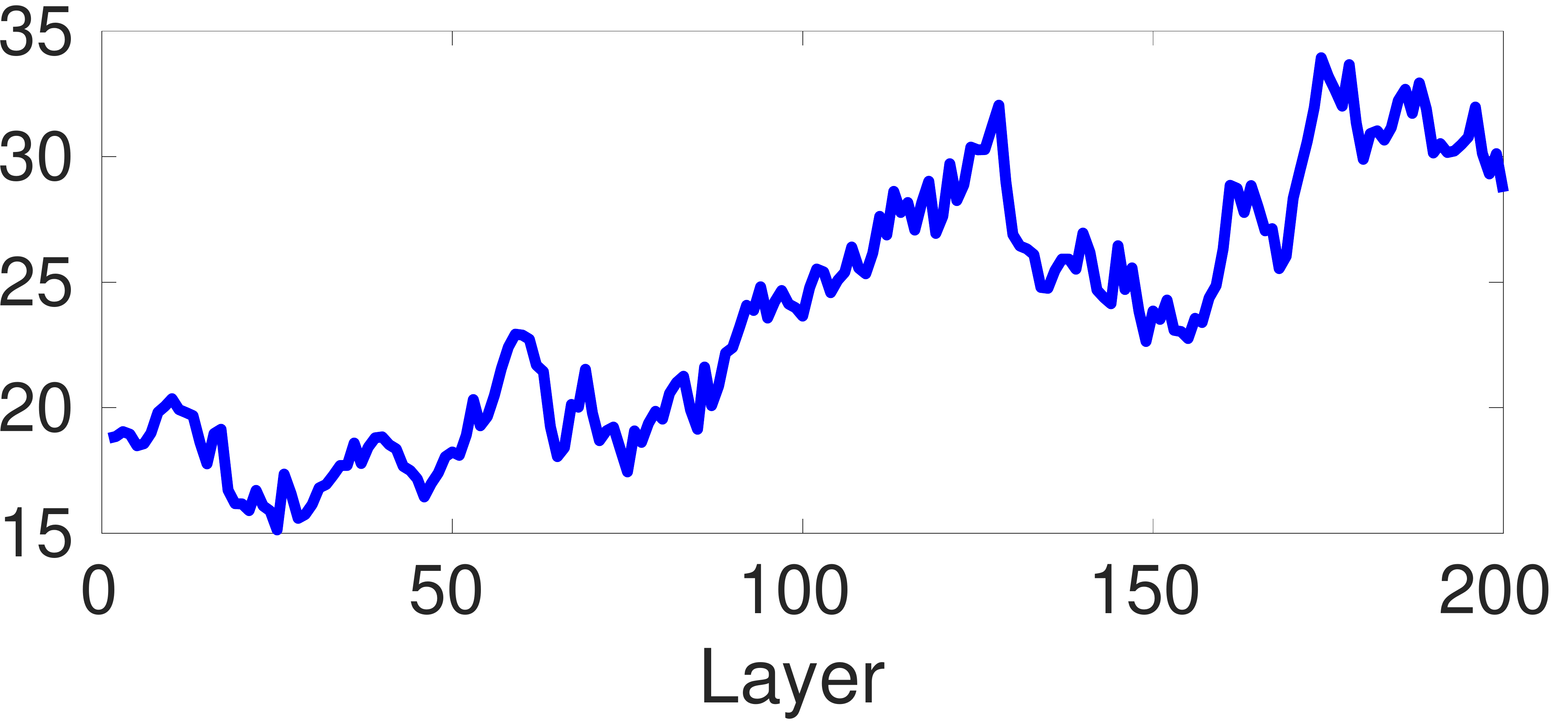}}

\vspace{0.5cm}
\caption{Results on synthetic data with different activation functions.  ``-GPN'' denotes the function is Gaussian-Poincar\'e normalized. $\|\mathbf{x}^{(l)}\|_2$ denotes the $l_2$ norm of the outputs of the $l$-th layer. $d$ denotes the width. $\|\frac{\partial E}{\partial \mathbf{W}^{(l)}}\|_F$ is the Frobenius norm of the gradient of the weight matrix in the $l$-th layer.}
\label{exp:synthetic_2}
\end{figure}

\begin{figure}[h!]
\centering

\subfloat[$\|\mathbf{x}^{(l)}\|_2^2/d$, ELU.]{\includegraphics[width=0.4\textwidth]{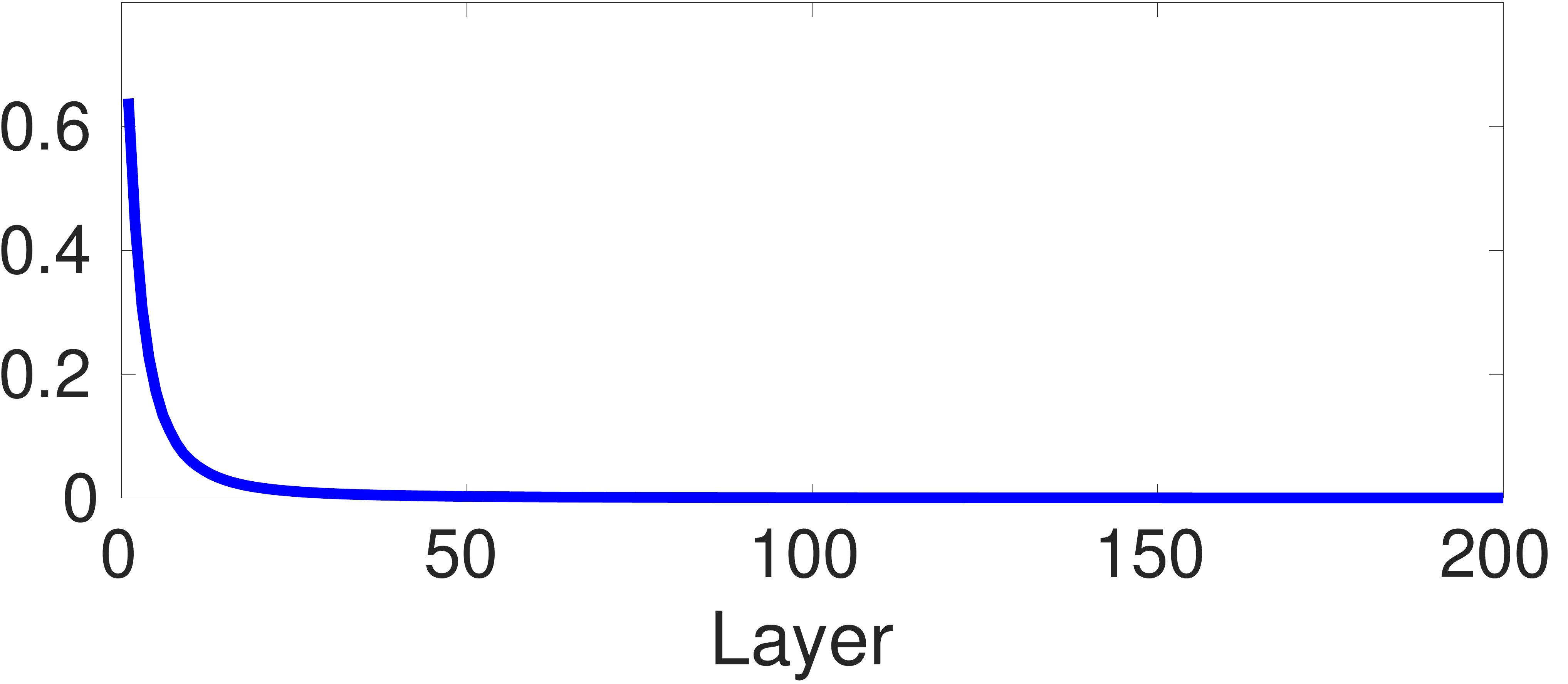}}
\hspace{0.5cm}
\subfloat[$\|\frac{\partial E}{\partial \mathbf{W}^{(l)}}\|_F$, ELU.]{\includegraphics[width=0.4\textwidth]{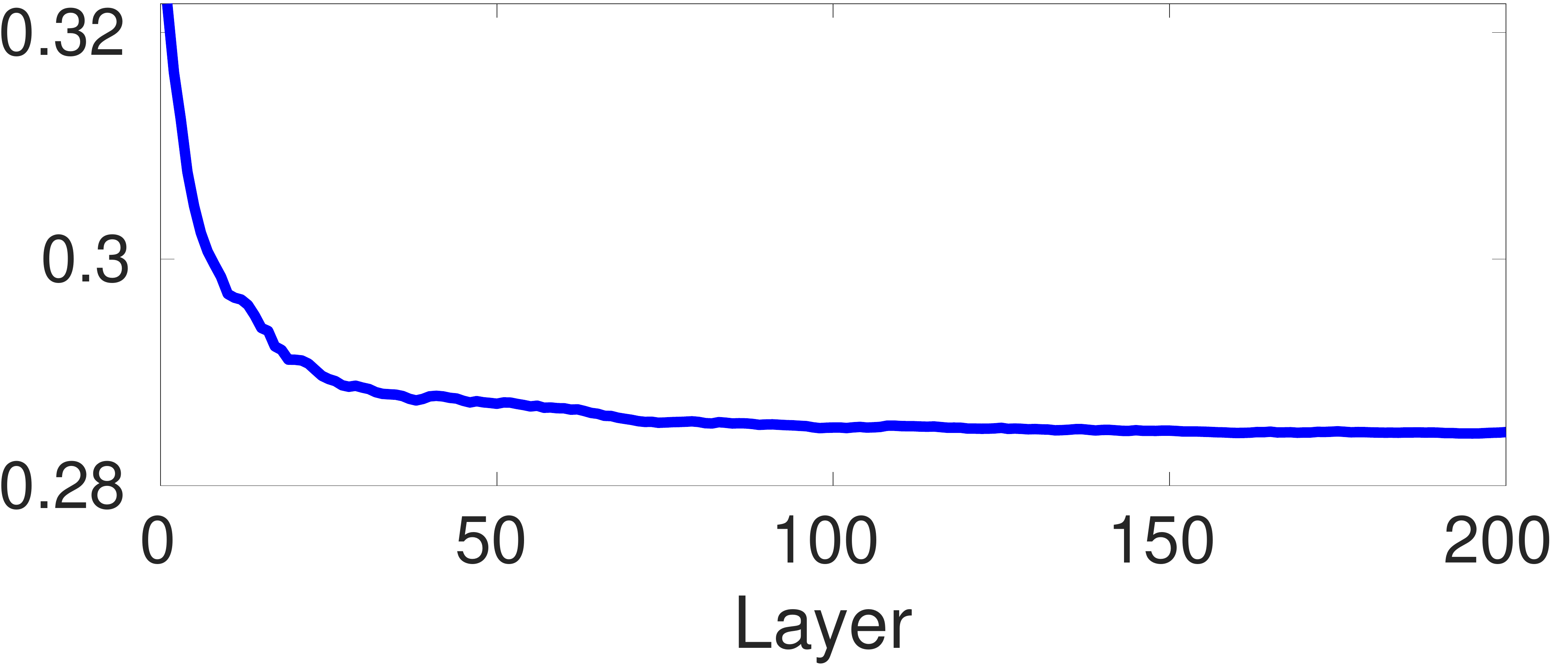}}

\subfloat[$\|\mathbf{x}^{(l)}\|_2^2/d$, ELU-GPN.]{\includegraphics[width=0.4\textwidth]{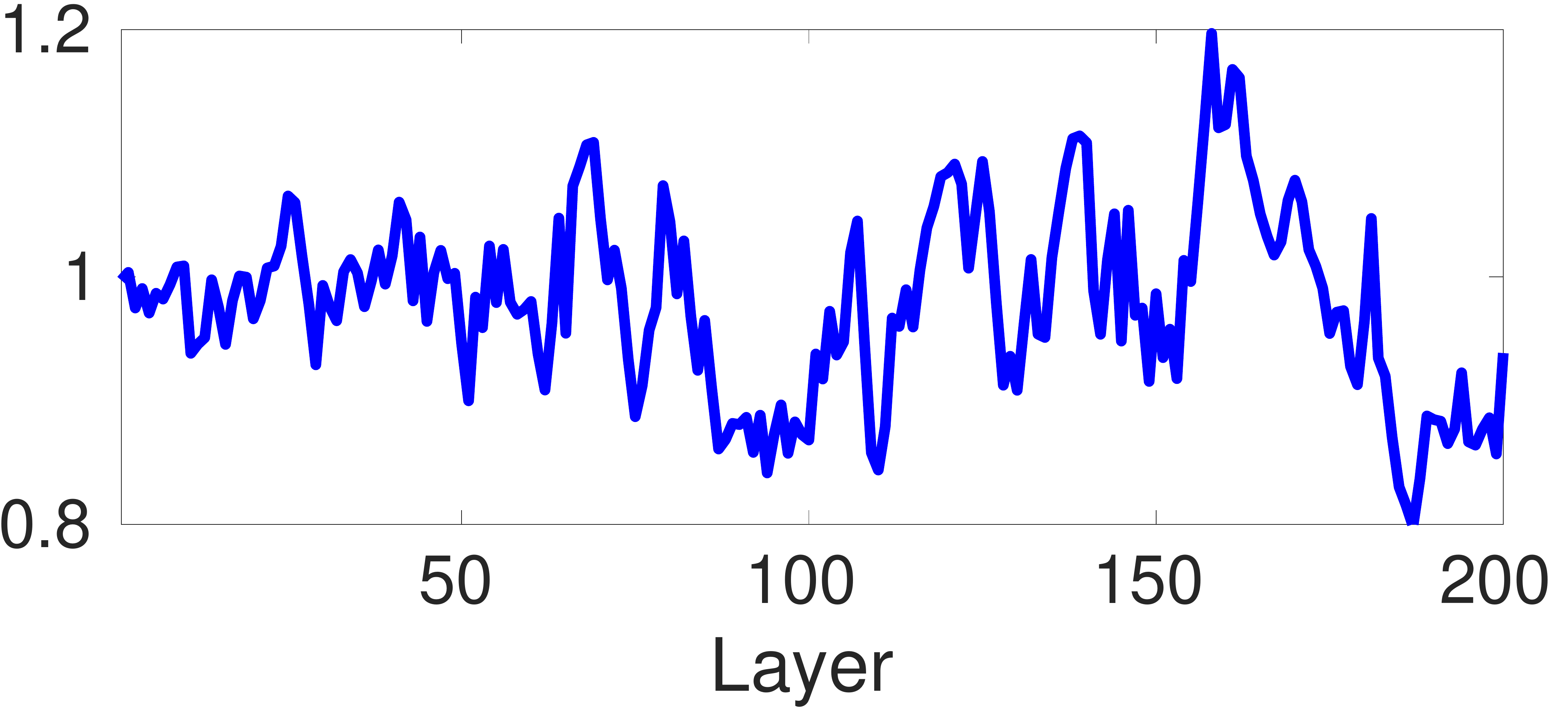}}
\hspace{0.5cm}
\subfloat[$\|\frac{\partial E}{\partial \mathbf{W}^{(l)}}\|_F$, ELU-GPN.]{\includegraphics[width=0.4\textwidth]{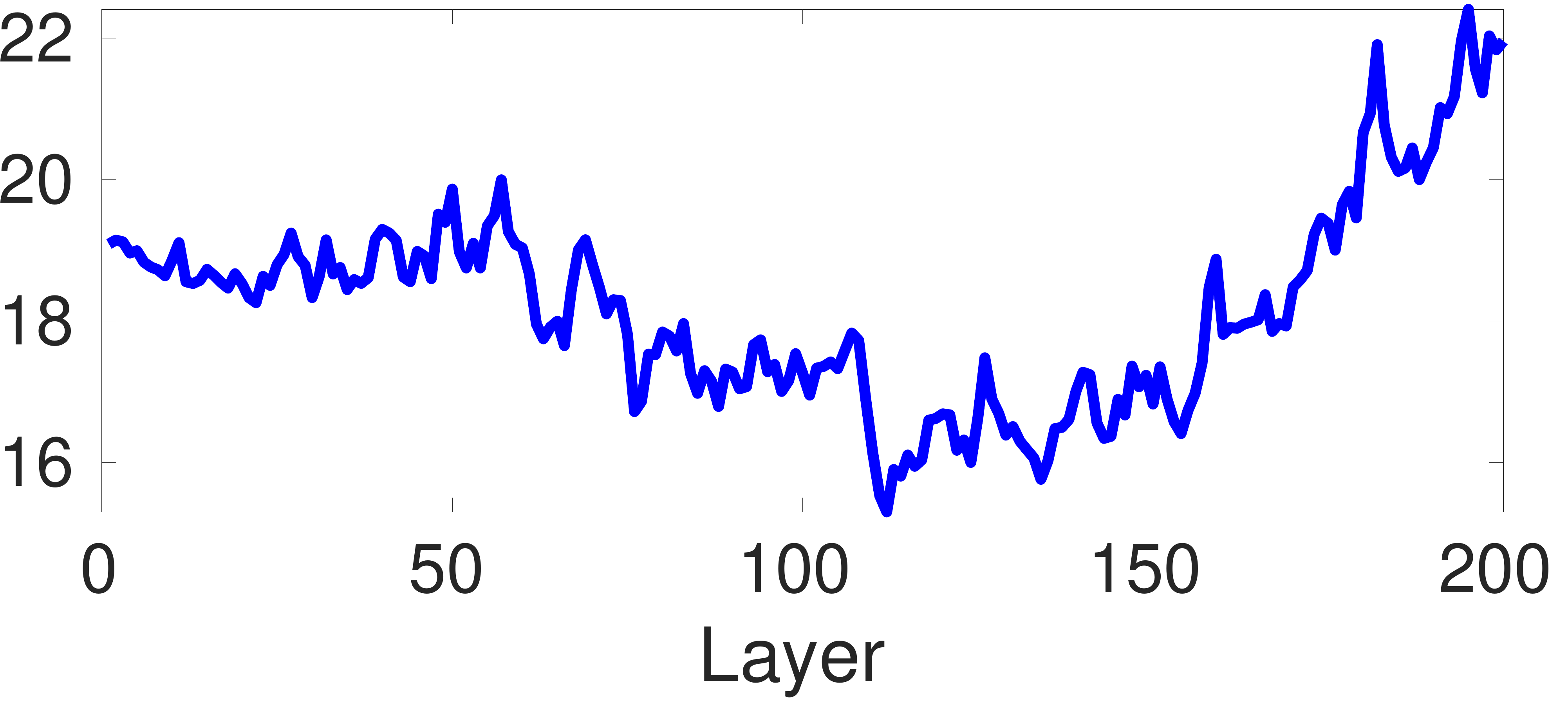}}

\subfloat[$\|\mathbf{x}^{(l)}\|_2^2/d$, GELU.]{\includegraphics[width=0.4\textwidth]{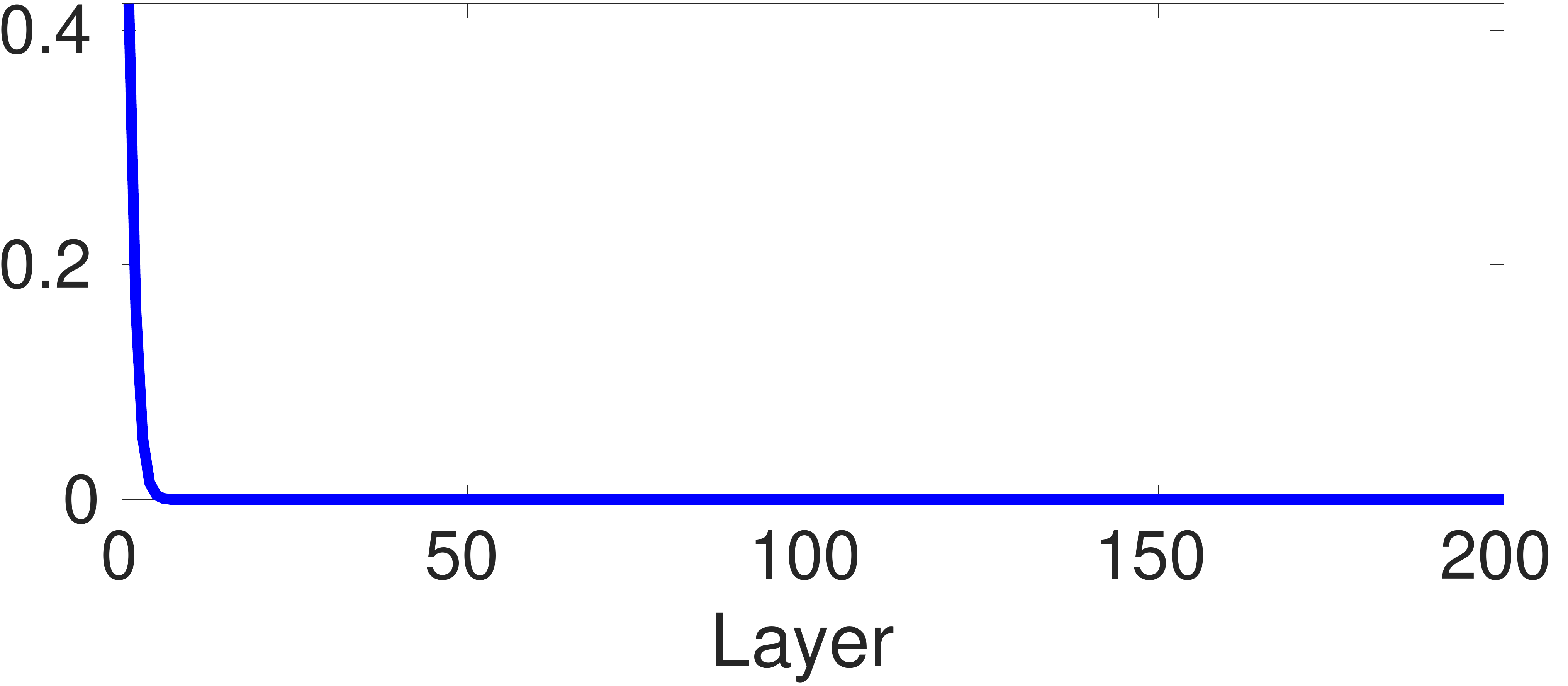}}
\hspace{0.5cm}
\subfloat[$\|\frac{\partial E}{\partial \mathbf{W}^{(l)}}\|_F$, GELU.]{\includegraphics[width=0.4\textwidth,height=0.205\textwidth,trim=0 0 0 -0.3cm]{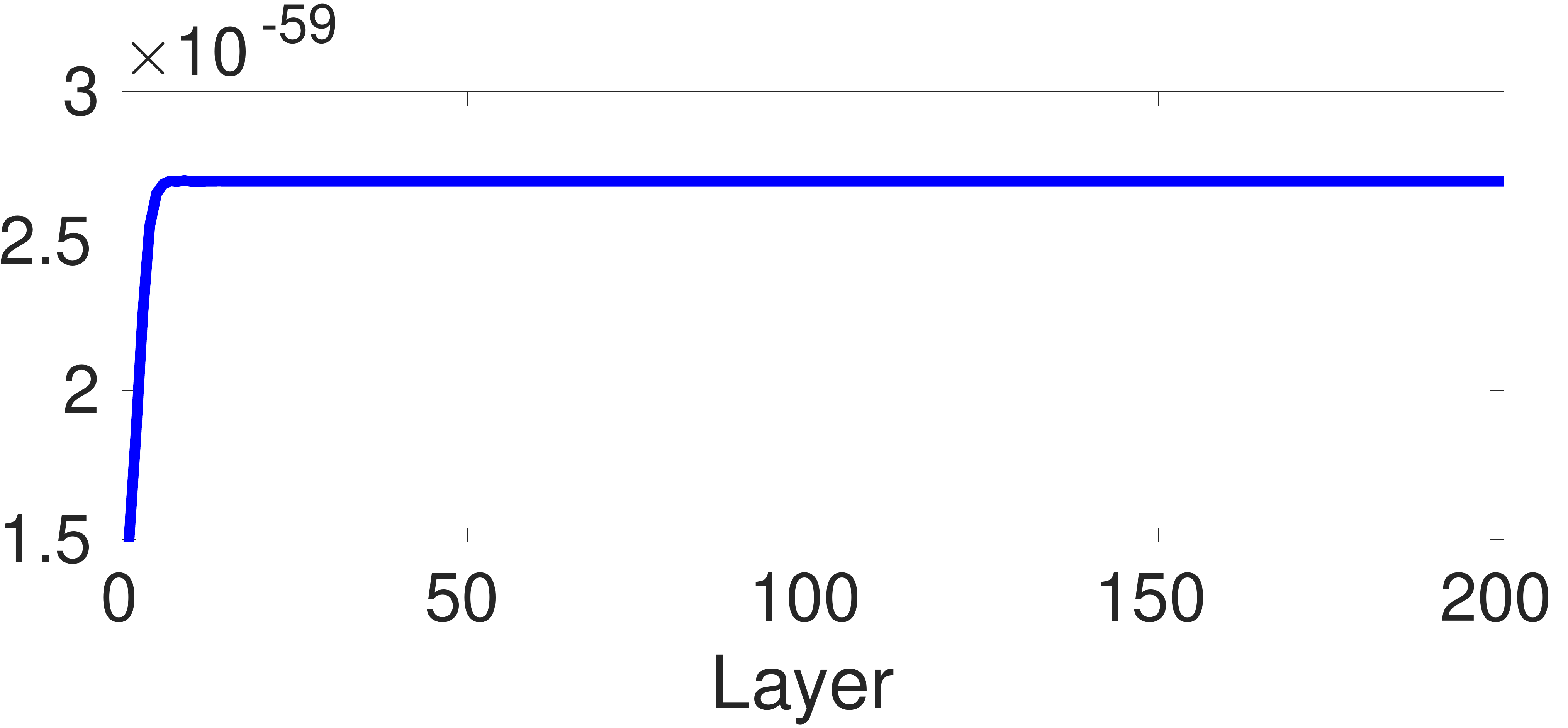}}

\subfloat[$\|\mathbf{x}^{(l)}\|_2^2/d$, GELU-GPN.]{\includegraphics[width=0.4\textwidth]{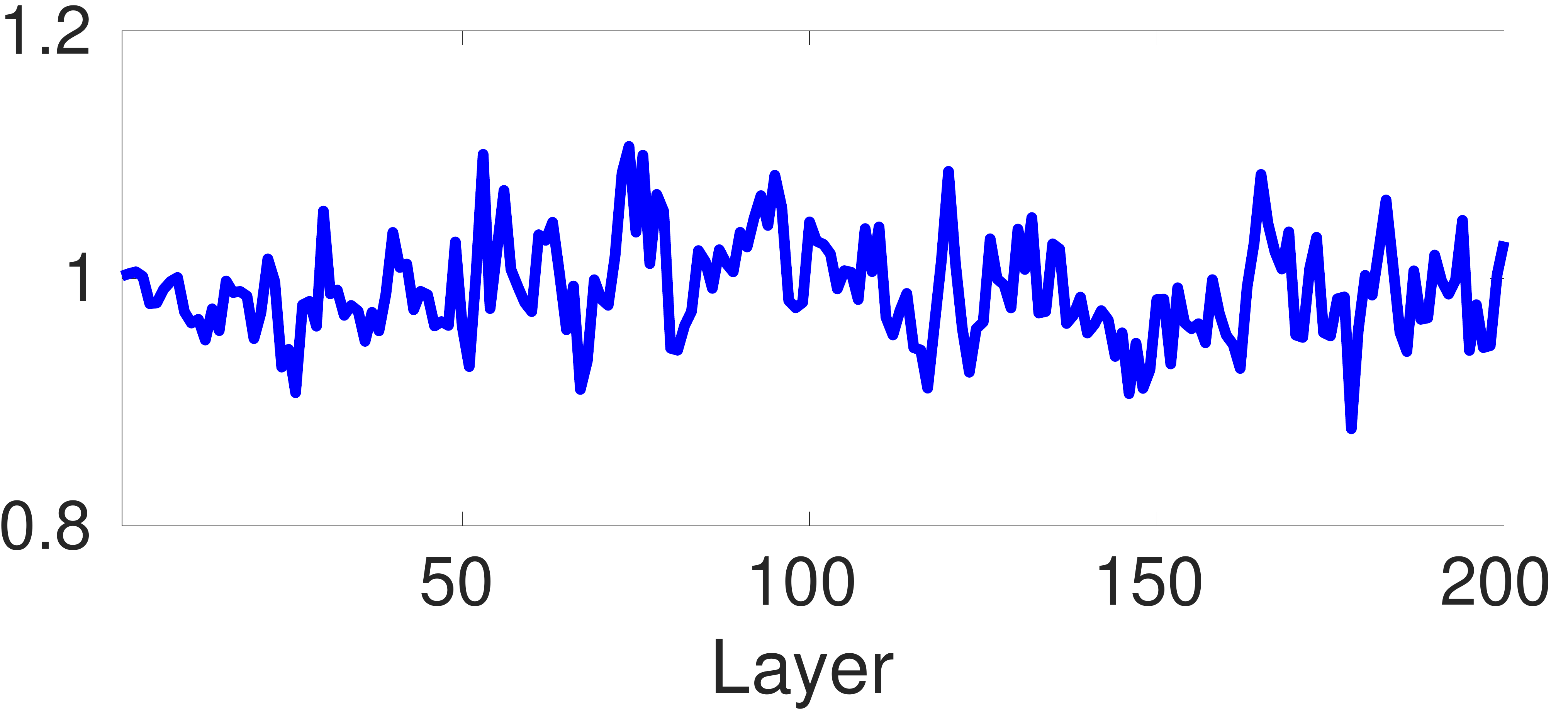}}
\hspace{0.5cm}
\subfloat[$\|\frac{\partial E}{\partial \mathbf{W}^{(l)}}\|_F$, GELU-GPN.]{\includegraphics[width=0.4\textwidth,,height=0.205\textwidth,trim=0 0 0 -0.3cm]{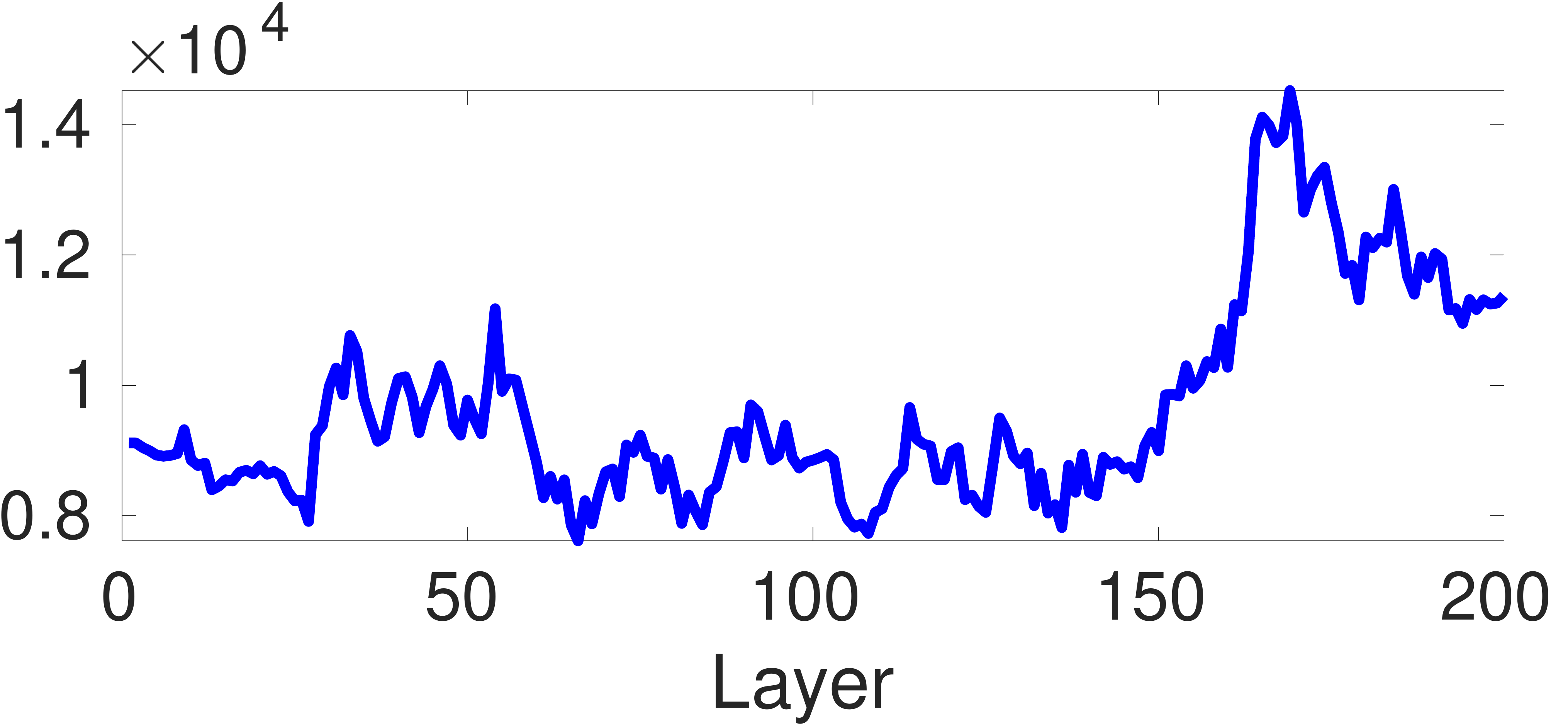}}

\vspace{0.5cm}
\caption{Results on synthetic data with different activation functions.  ``-GPN'' denotes the function is Gaussian-Poincar\'e normalized. $\|\mathbf{x}^{(l)}\|_2$ denotes the $l_2$ norm of the outputs of the $l$-th layer. $d$ denotes the width. $\|\frac{\partial E}{\partial \mathbf{W}^{(l)}}\|_F$ is the Frobenius norm of the gradient of the weight matrix in the $l$-th layer.}
\label{exp:synthetic_3}
\end{figure}

\clearpage

\begin{figure}[h!]
\centering
\subfloat[ReLU.]{\includegraphics[width=0.5\textwidth]{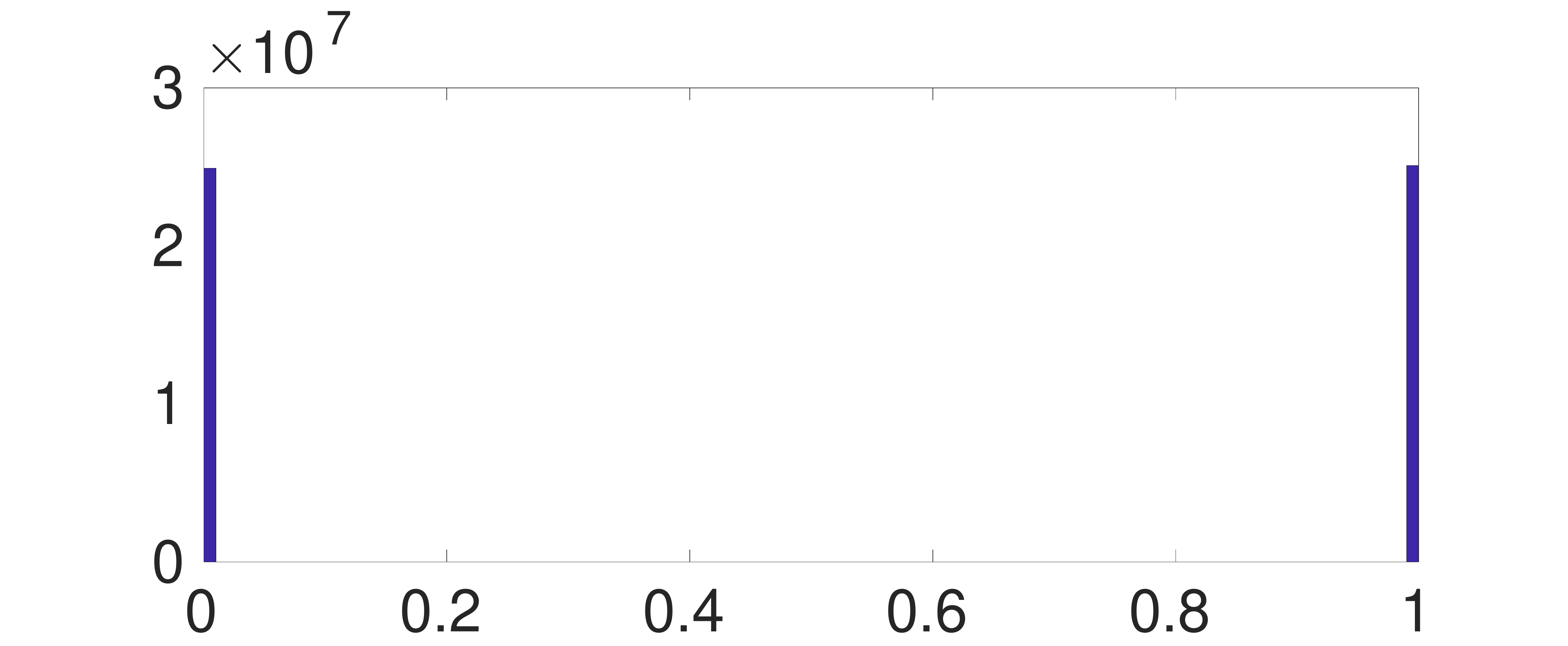}}
\subfloat[ReLU-GPN.]{\includegraphics[width=0.5\textwidth]{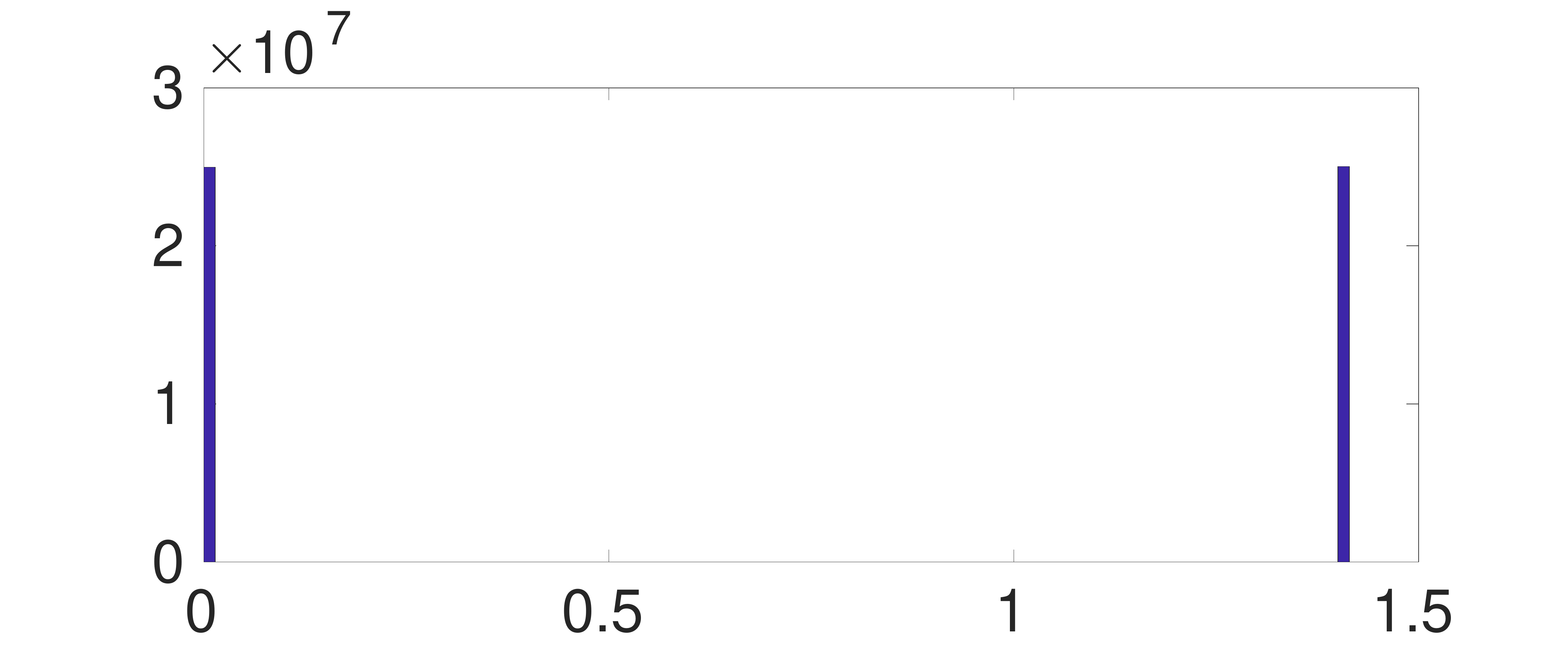}}

\subfloat[LeakyReLU.]{\includegraphics[width=0.5\textwidth]{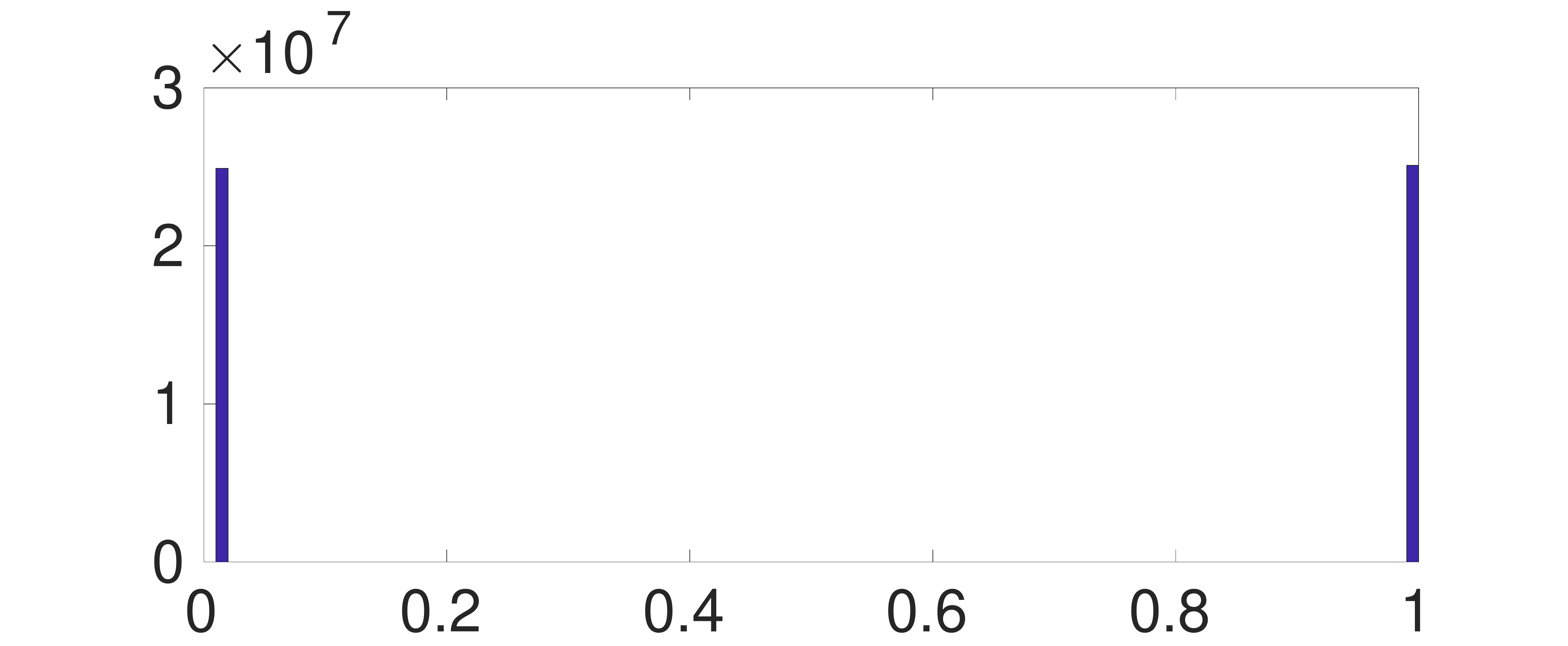}}
\subfloat[LeakyReLU-GPN.]{\includegraphics[width=0.5\textwidth]{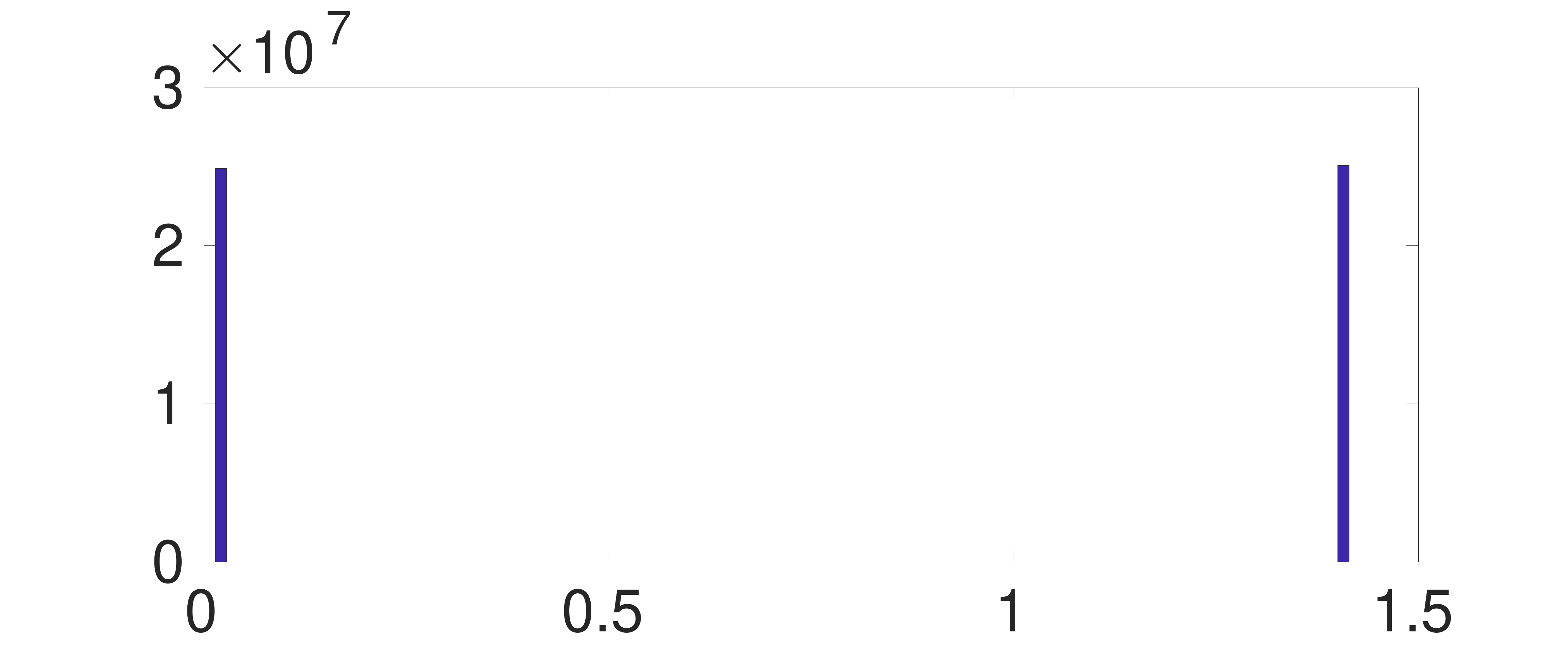}}

\subfloat[ELU.]{\includegraphics[width=0.5\textwidth]{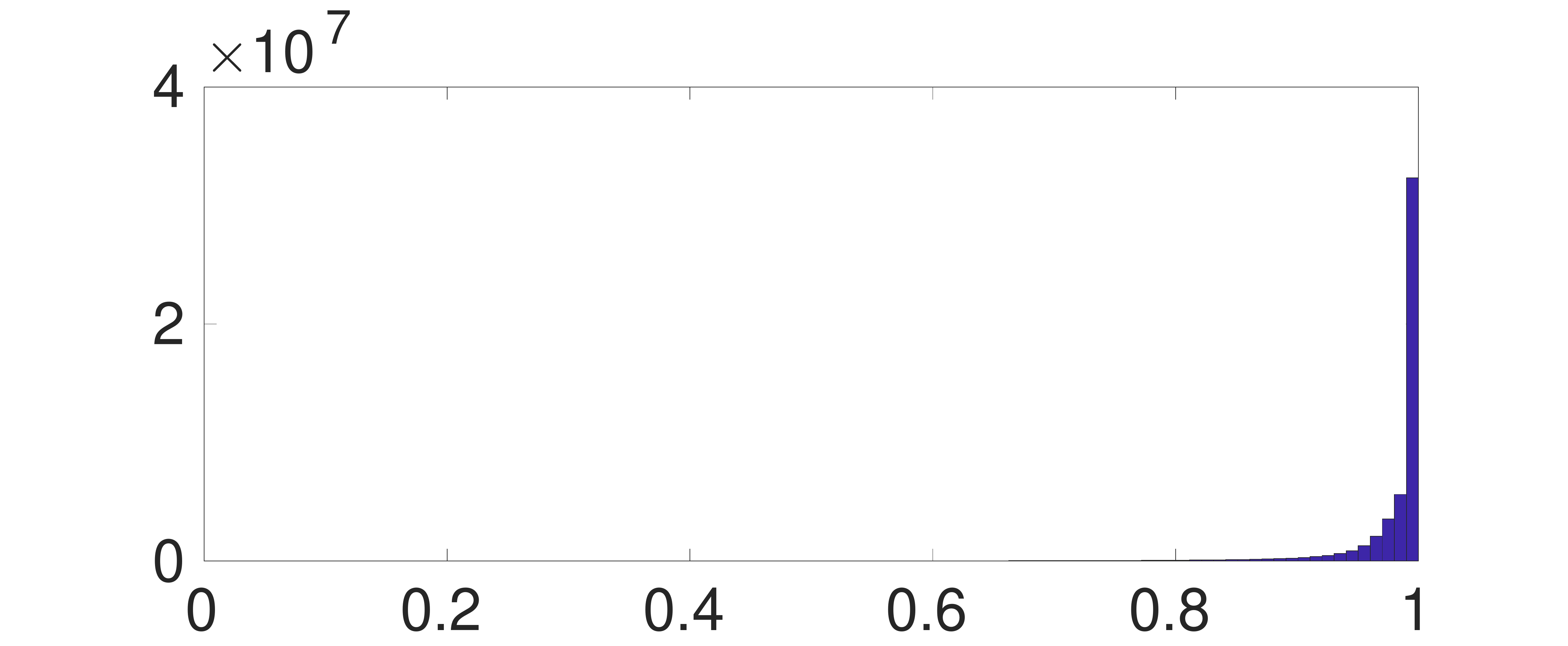}}
\subfloat[ELU-GPN.]{\includegraphics[width=0.5\textwidth]{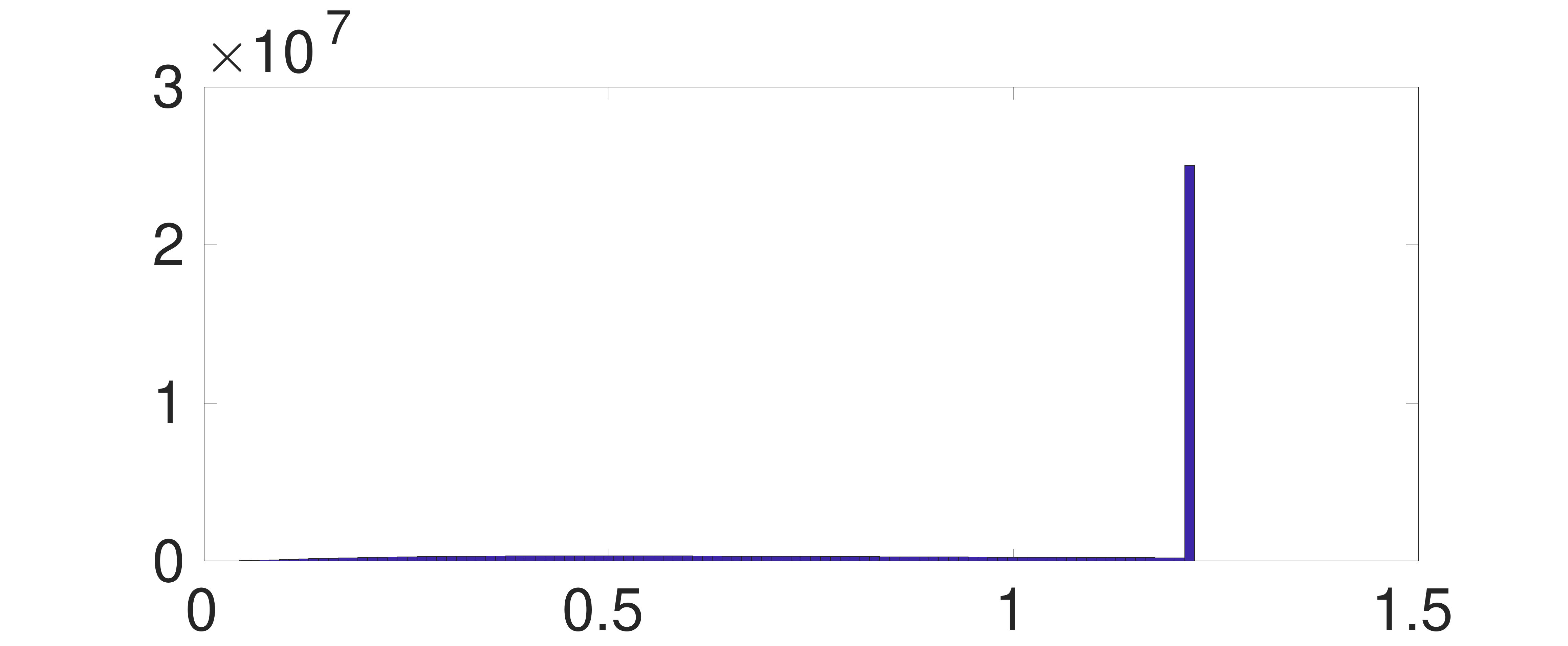}}

\subfloat[SELU.]{\includegraphics[width=0.5\textwidth]{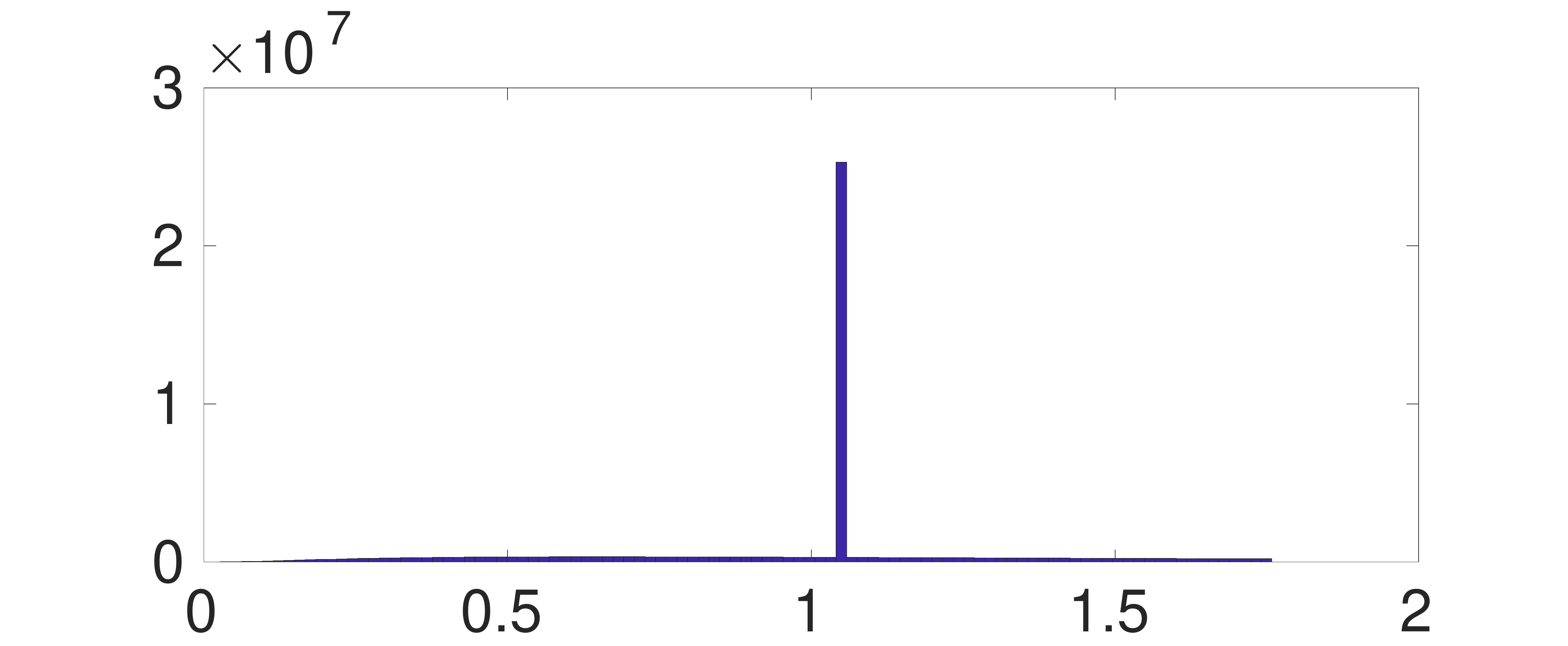}}
\subfloat[SELU-GPN.]{\includegraphics[width=0.5\textwidth]{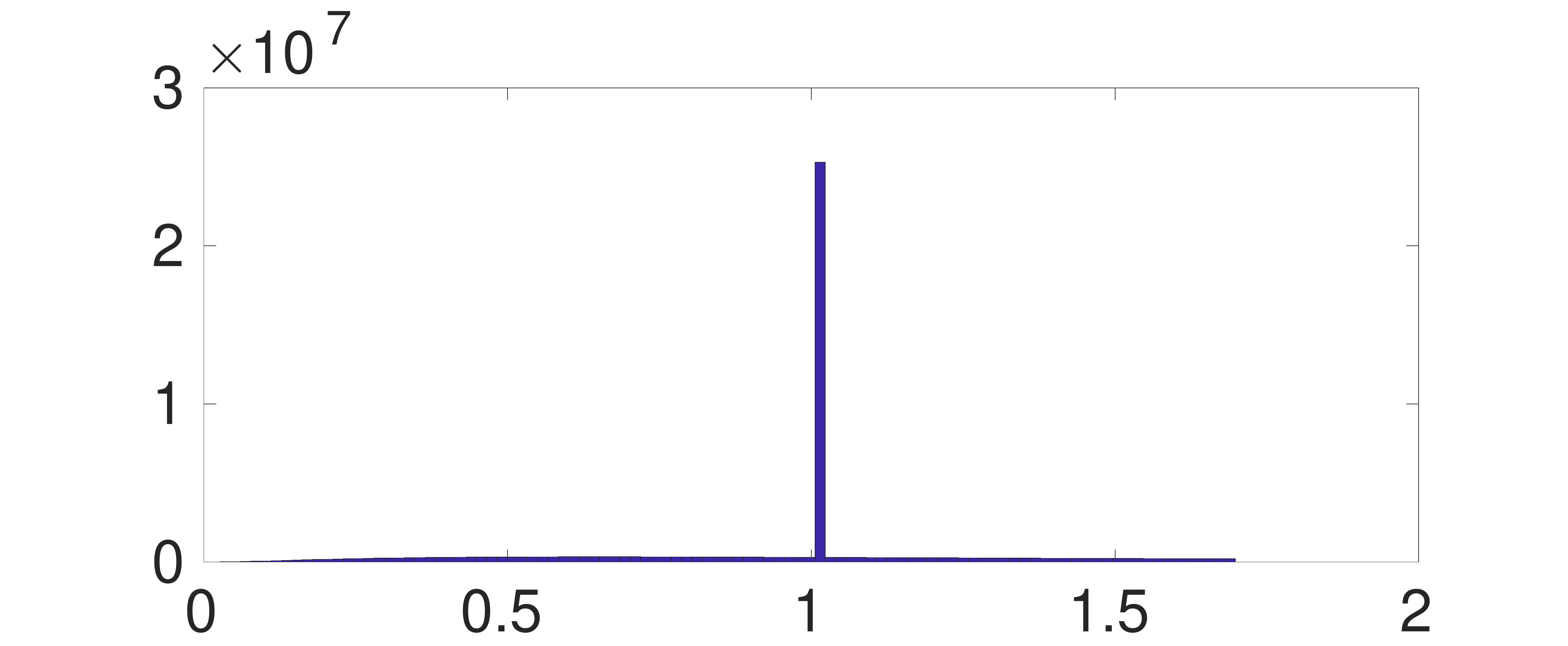}}

\subfloat[GELU.]{\includegraphics[width=0.5\textwidth]{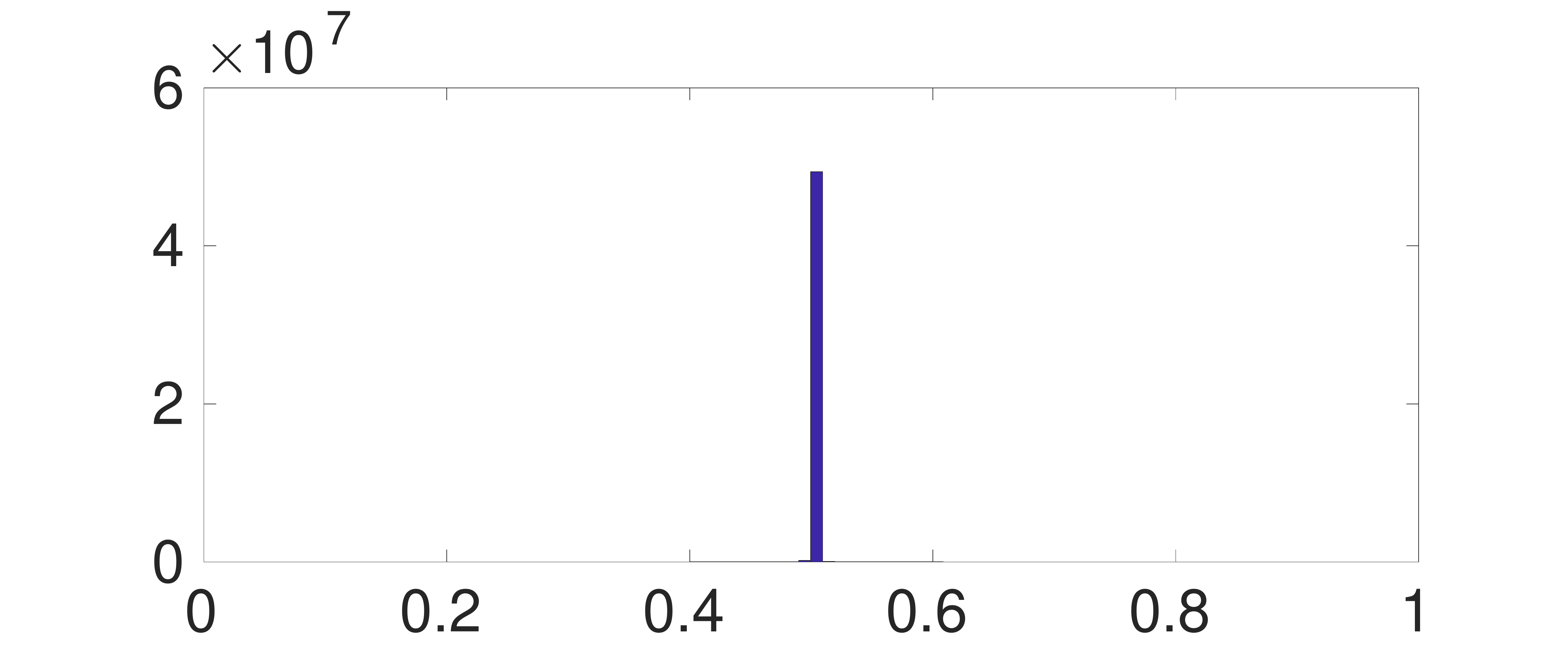}}
\subfloat[GELU-GPN.]{\includegraphics[width=0.5\textwidth]{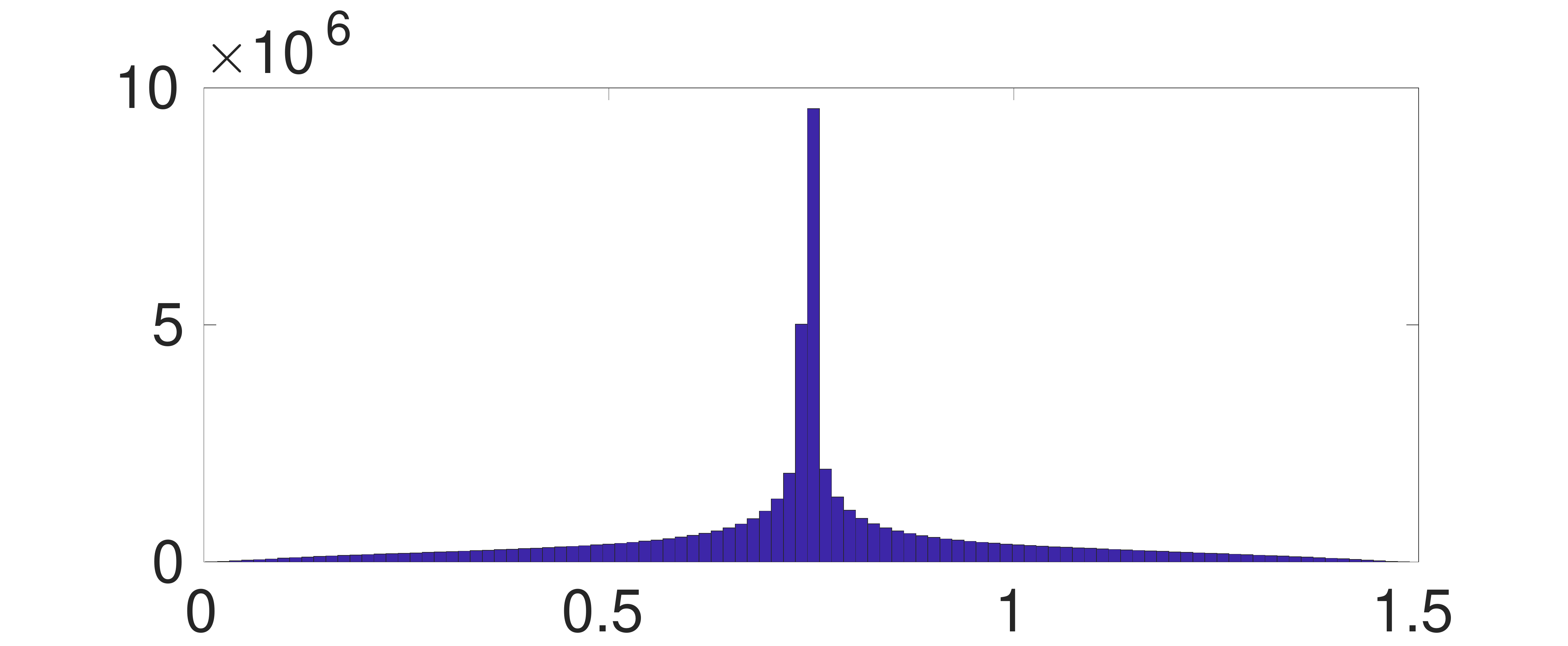}}

\vspace{0.5cm}
\caption{Histogram of $\phi'(h_i^{(l)})$. The values of $\phi'(h_i^{(l)})$ are accumulated for all units, all layers and all samples in the histogram. Except for ELU, none of them has values concentrated around one.}
\label{exp:D_hist_2}
\end{figure}

\clearpage

\begin{figure}[h!]
\centering
\subfloat[ReLU-GPN.]{\includegraphics[width=0.4\textwidth]{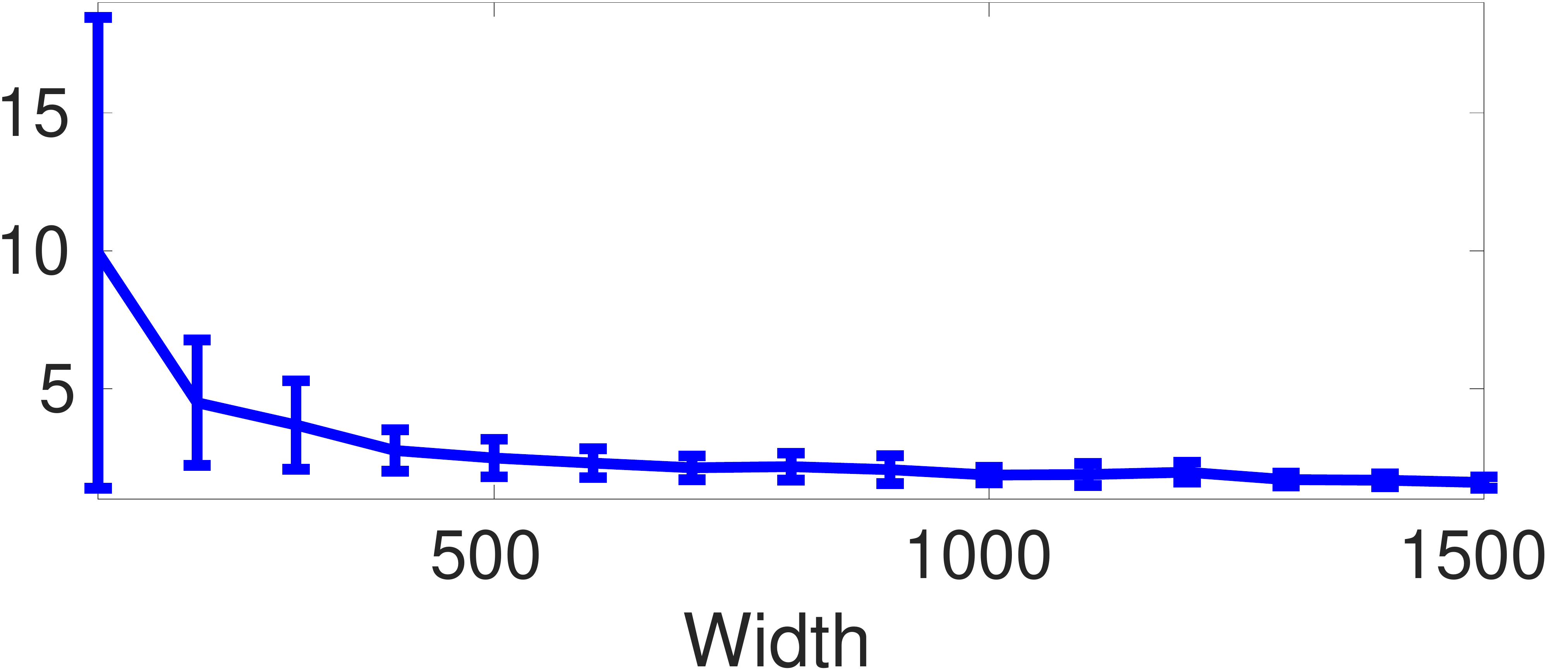}}
\hspace{0.5cm}\
\subfloat[LeakyReLU-GPN.]{\includegraphics[width=0.4\textwidth]{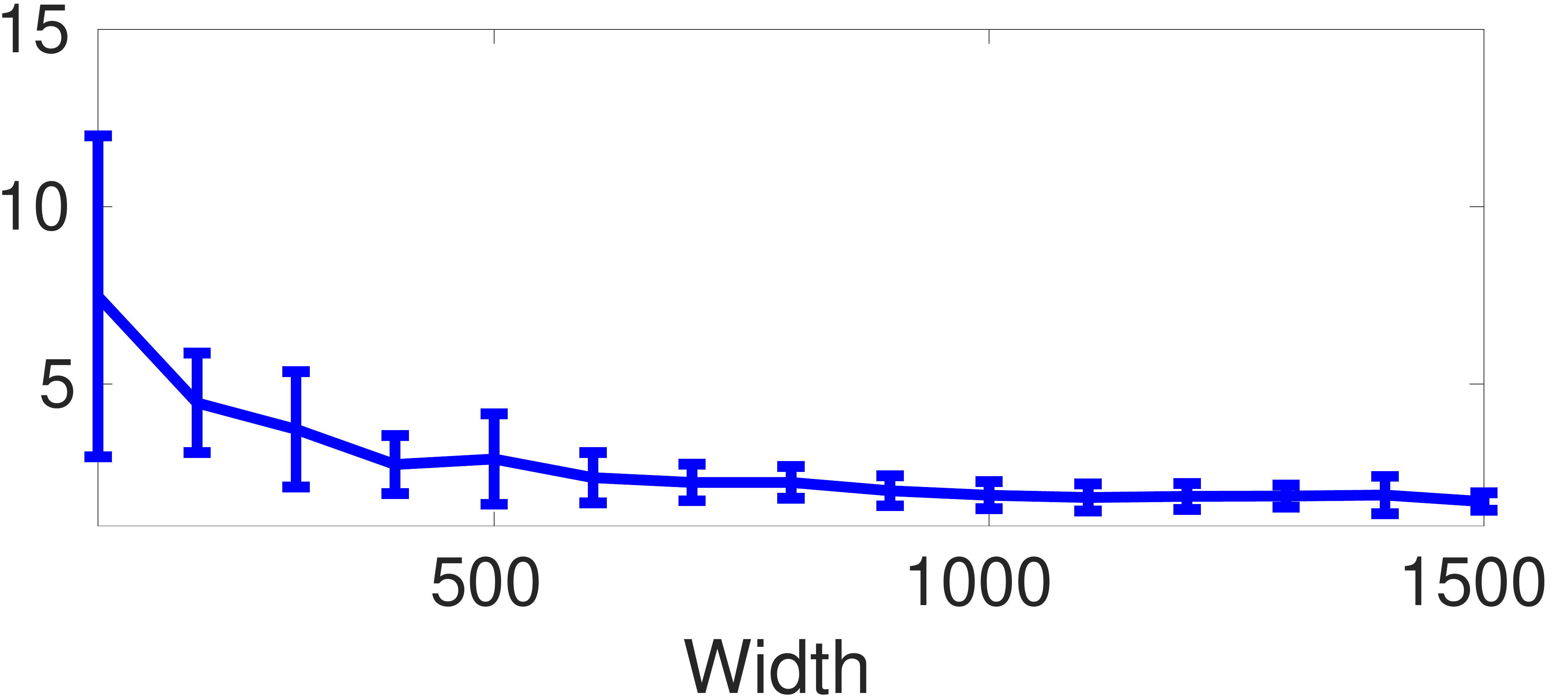}}

\subfloat[ELU-GPN.]{\includegraphics[width=0.4\textwidth]{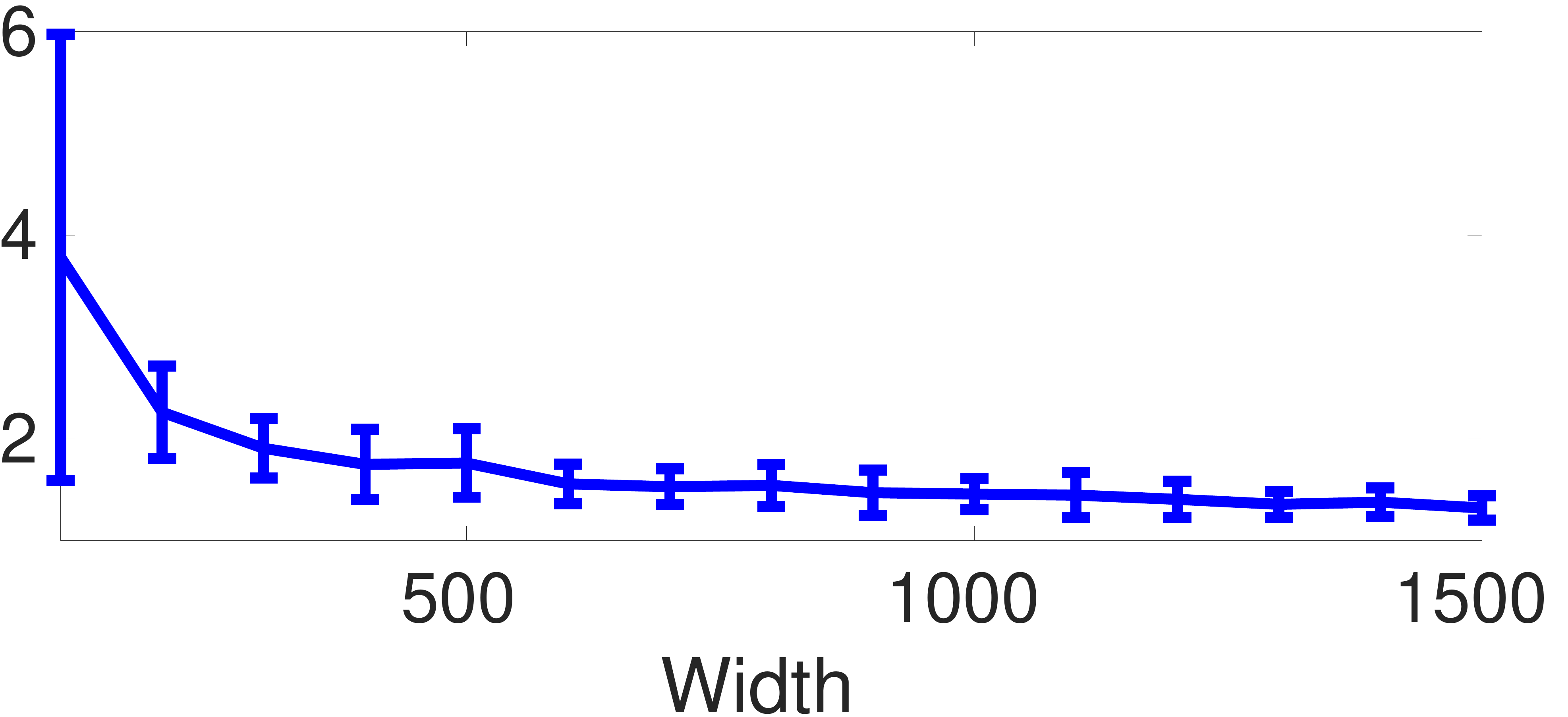}}
\hspace{0.5cm}\
\subfloat[GELU-GPN.]{\includegraphics[width=0.4\textwidth]{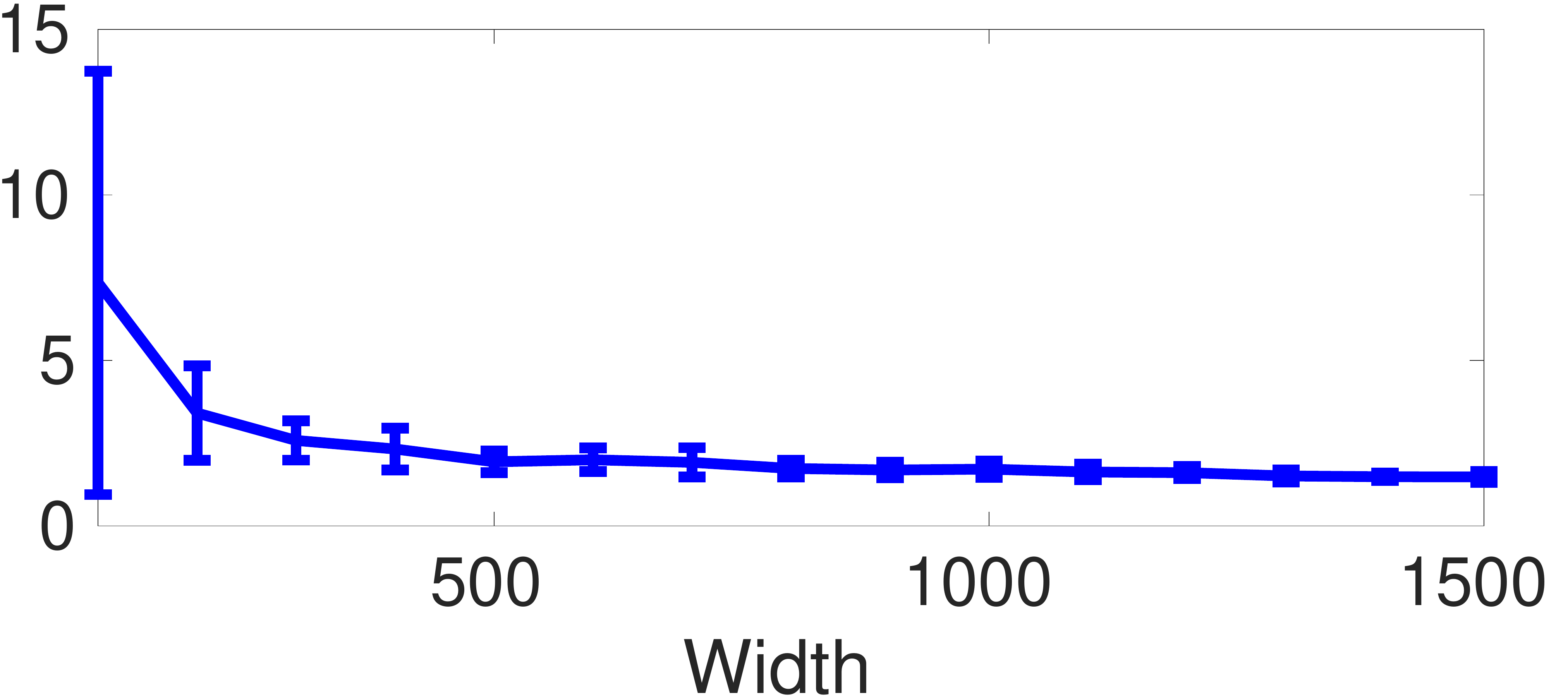}}

\vspace{0.5cm}
\caption{Gradient norm ratio for different layer width on synthetic data. The ratio is defined as $\max_l\|\frac{\partial E}{\partial \mathbf{W}^{(l)}}\|_F / \min_l\|\frac{\partial E}{\partial \mathbf{W}^{(l)}}\|_F$. The width ranges from 100 to 1500. The error bars show standard deviation.}
\label{exp:synthetic_width_2}
\end{figure}

\subsection{Real-World Data}

In Figure \ref{exp:test_acc_cifar10_2}, we show the test accuracy during training on CIFAR-10 in addition to Figure \myref{4} in the main text. In Figure \ref{exp:test_acc_mnist}, we show the experiments on MNIST.

In Figure \ref{fig:grad_1}, \ref{fig:grad_2},  \ref{fig:grad_3} and \ref{fig:grad_4}, we show a measure of the magnitude of the  vanishing/exploding gradients during training for different activation functions. The measure is defined as the ratio of the maximum gradient norm and the minimum gradient norm across layers.
Since we use the parametrization 
\begin{align}
\mathbf{W} = \Big(\frac{\mathbf{v}_1}{\|\mathbf{v}_1\|_2},...,\frac{\mathbf{v}_d}{\|\mathbf{v}_d\|_2}\Big)^T    
\end{align}
with $\mathbf{V} = (\mathbf{v}_1,...,\mathbf{v}_d)^T$, the gradient norm ratio is defined on the unconstrained weights $\mathbf{V}$, that is,
\begin{align}
\frac{\max_l\|\frac{\partial E}{\partial \mathbf{V}^{(l)}}\|_F}{ \min_l\|\frac{\partial E}{\partial \mathbf{V}^{(l)}}\|_F}.
\end{align}
Note that for ReLU, LeakyReLU and GELU, the gradient vanishes during training in some experiments and therefore the plots are empty. From the figures, we can see that batch normalization leads to gradient explosion especially at the early stage of training. On the other hand, without batch normalization, the gradient is stable throughout training for GPN activation functions.

\clearpage

\begin{figure}[h!]
\centering
\subfloat[ReLU.]{\includegraphics[width=0.4\textwidth]{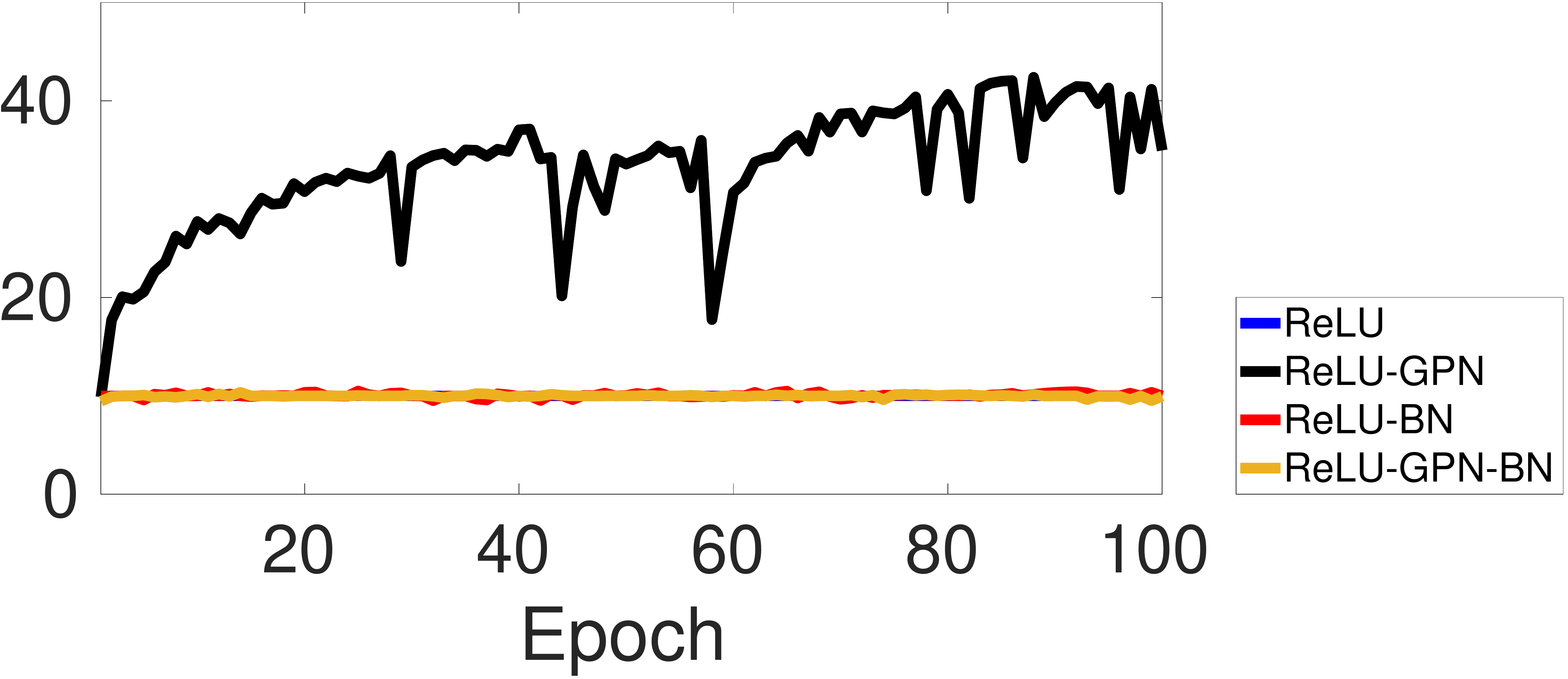}}
\hspace{0.5cm}
\subfloat[LeakyReLU.]{\includegraphics[width=0.4\textwidth]{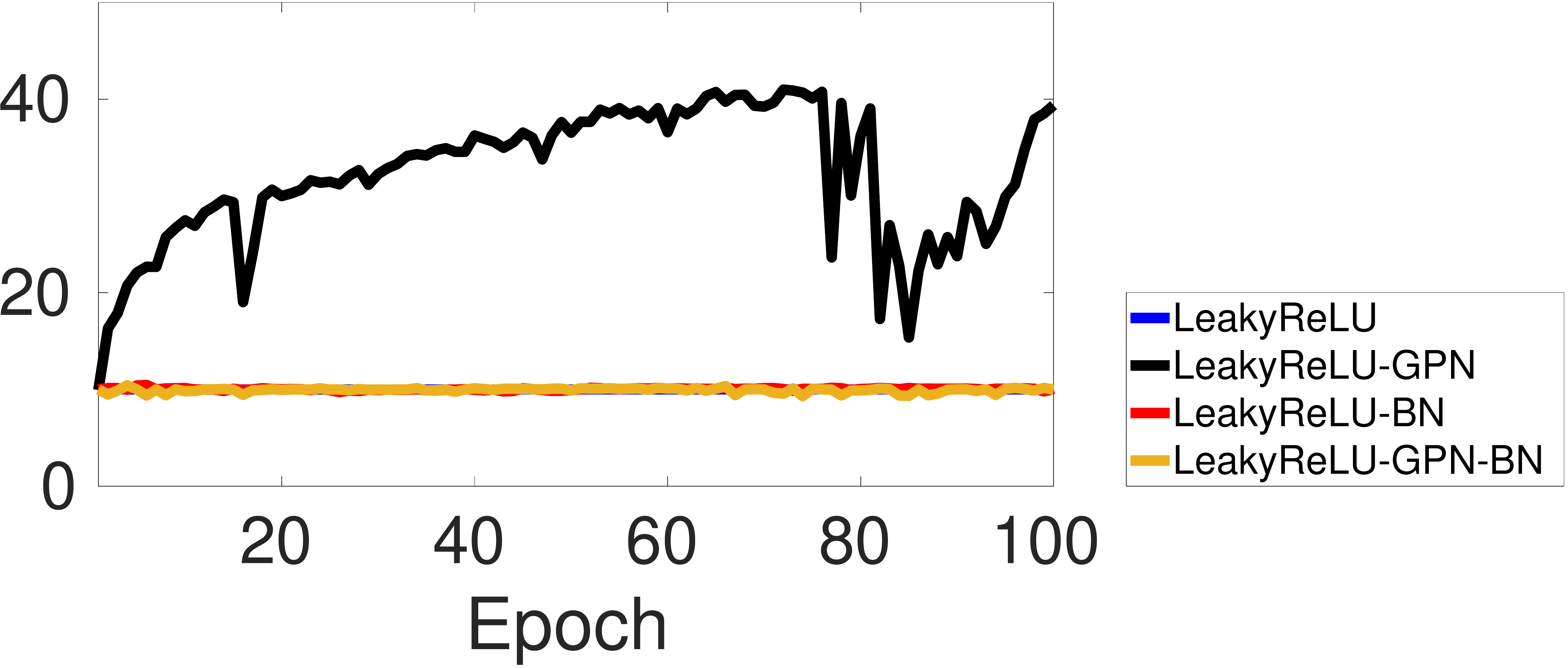}}

\subfloat[ELU.]{\includegraphics[width=0.4\textwidth]{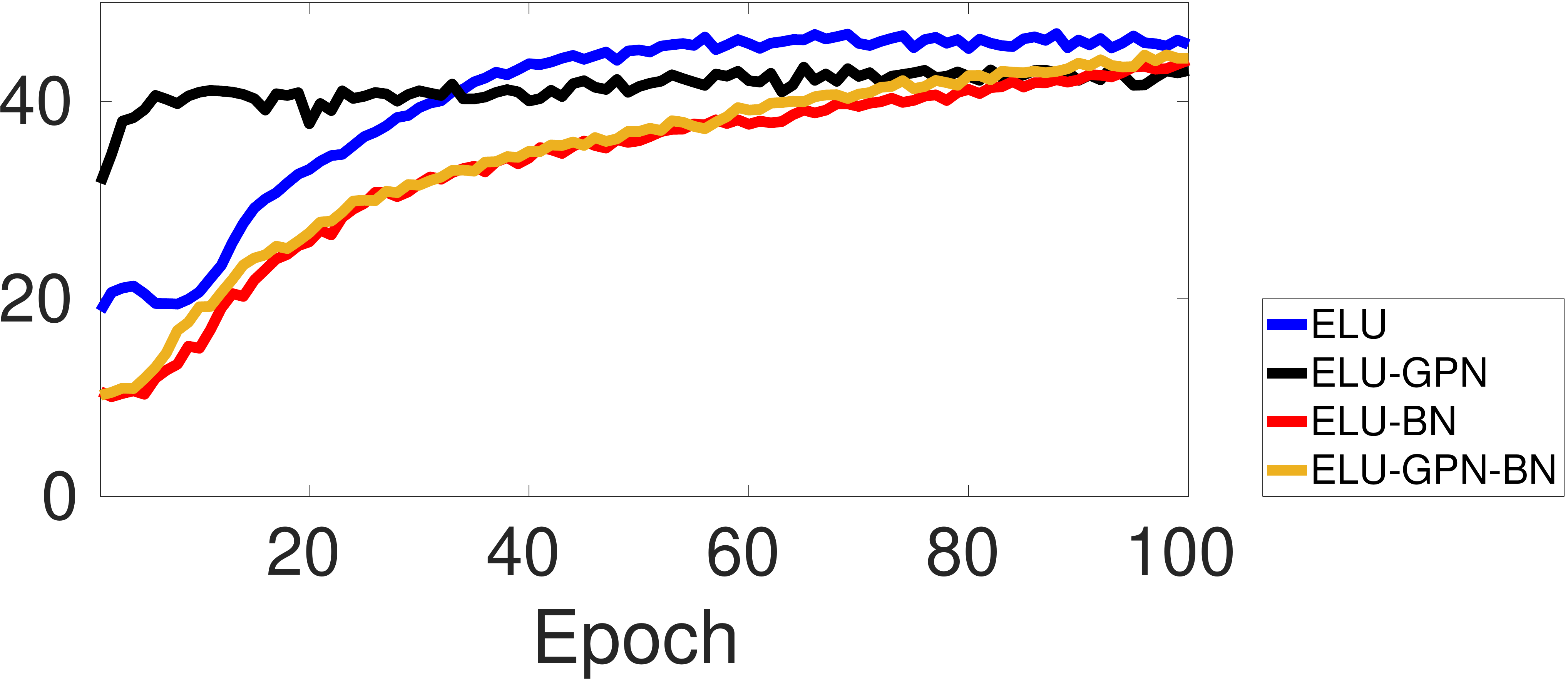}}
\hspace{0.5cm}
\subfloat[GELU.]{\includegraphics[width=0.4\textwidth]{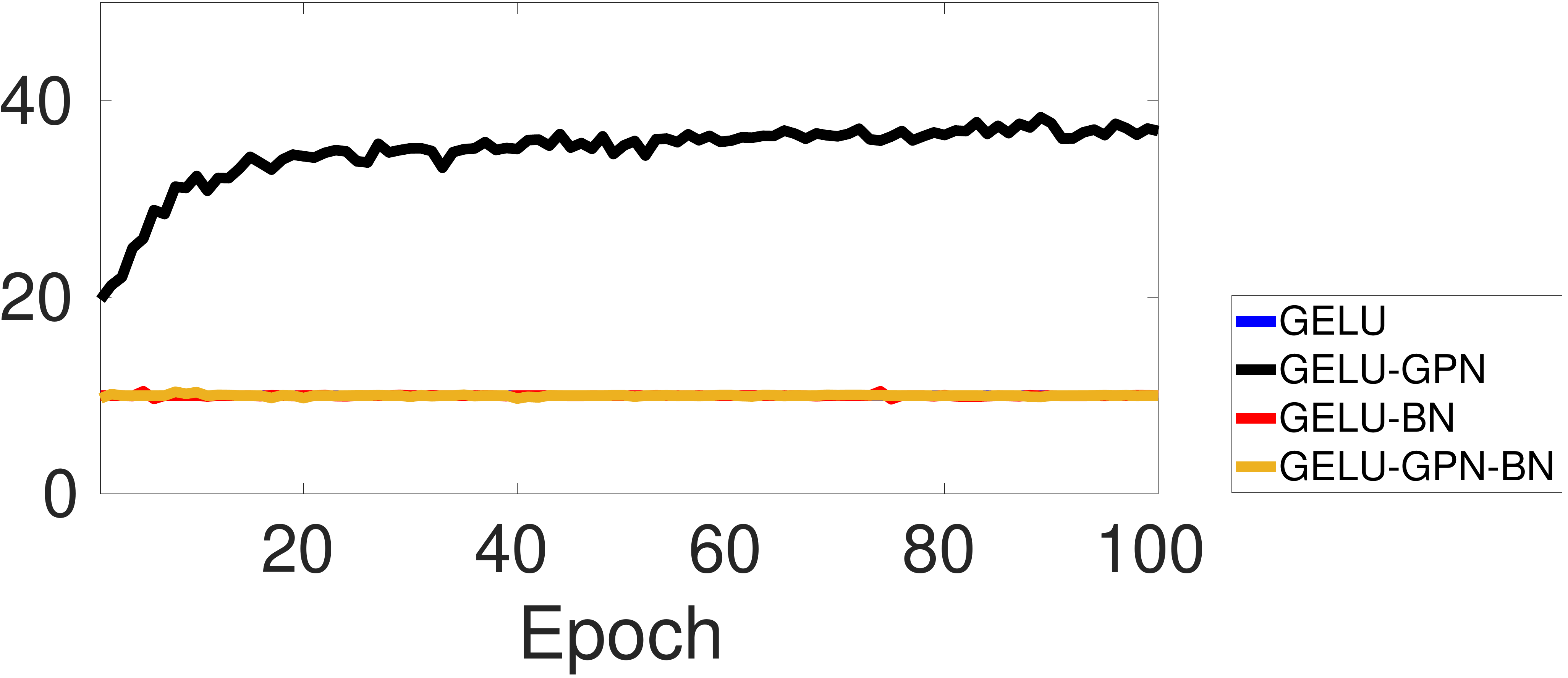}}

\vspace{0.5cm}
\caption{Test accuracy (percentage) during training on CIFAR-10.}
\label{exp:test_acc_cifar10_2}
\end{figure}

\hspace{2.5cm}

\begin{figure}[h!]
\centering
\subfloat[Tanh.]{\includegraphics[width=0.4\textwidth]{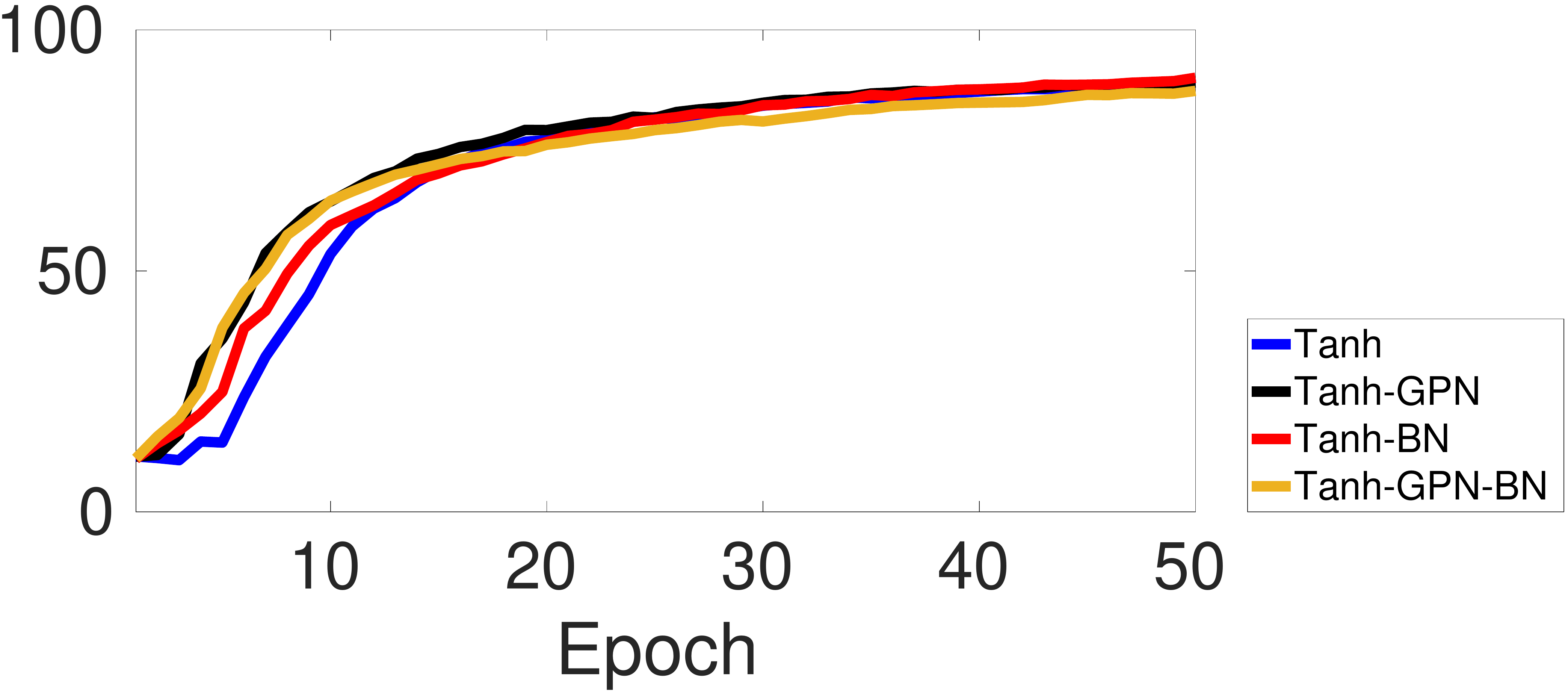}}
\hspace{0.5cm}
\subfloat[ReLU.]{\includegraphics[width=0.4\textwidth]{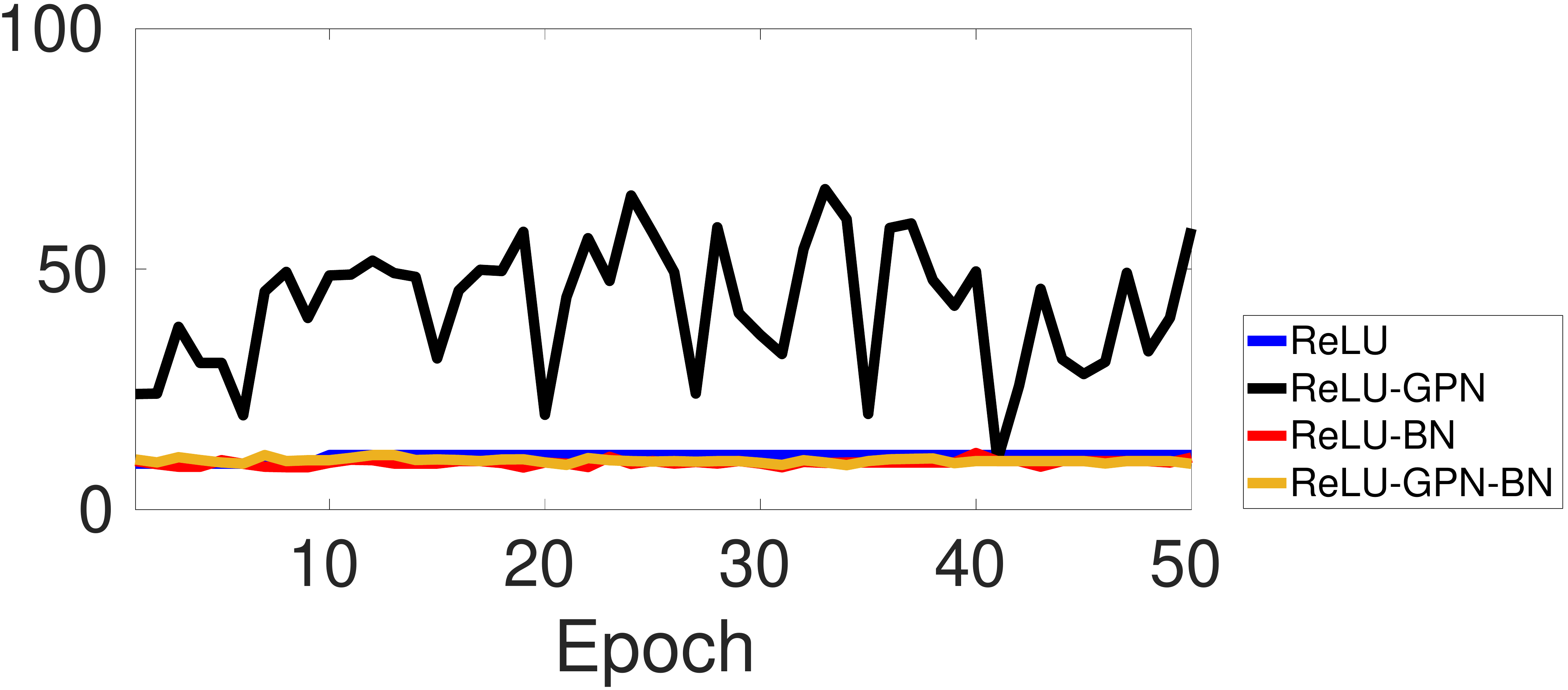}}

\subfloat[LeakyReLU.]{\includegraphics[width=0.4\textwidth]{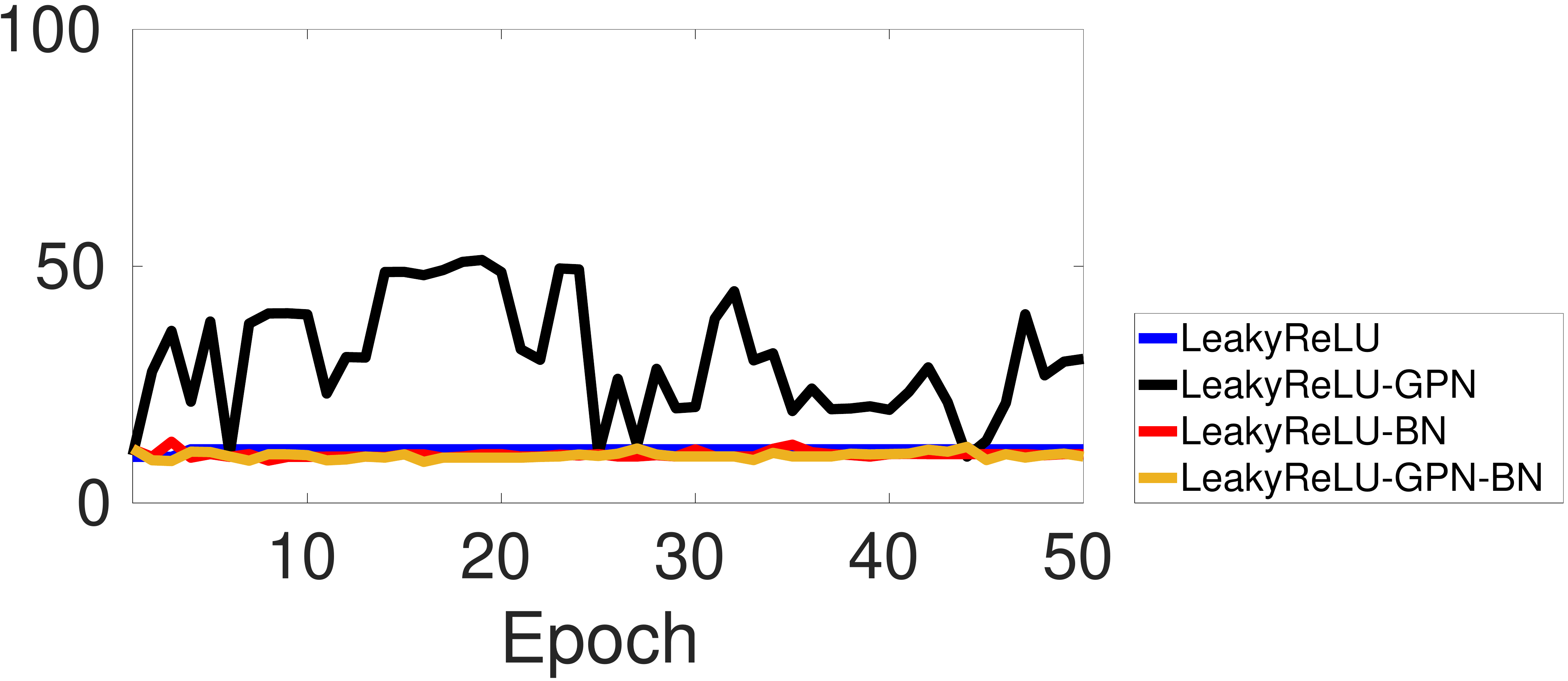}}
\hspace{0.5cm}
\subfloat[ELU.]{\includegraphics[width=0.4\textwidth]{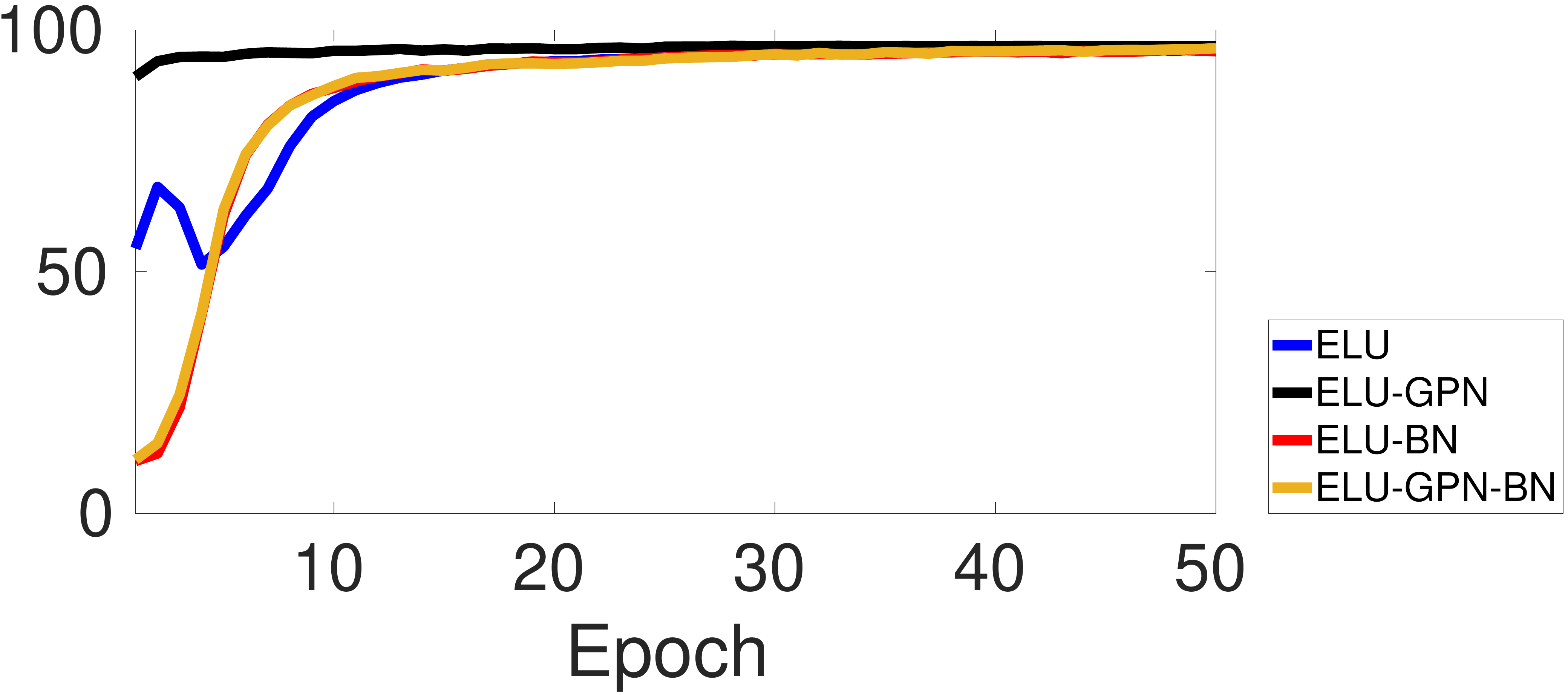}}

\subfloat[SELU.]{\includegraphics[width=0.4\textwidth]{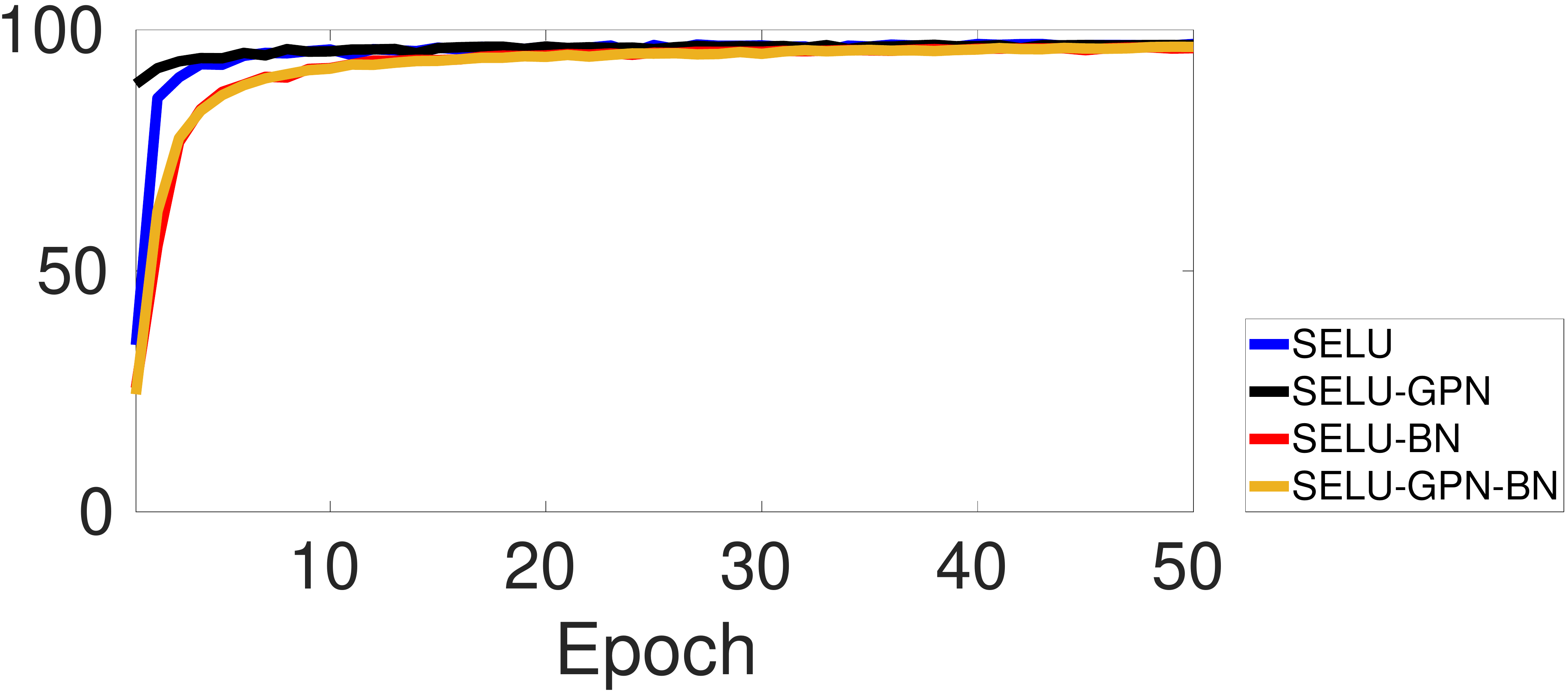}}
\hspace{0.5cm}
\subfloat[GELU.]{\includegraphics[width=0.4\textwidth]{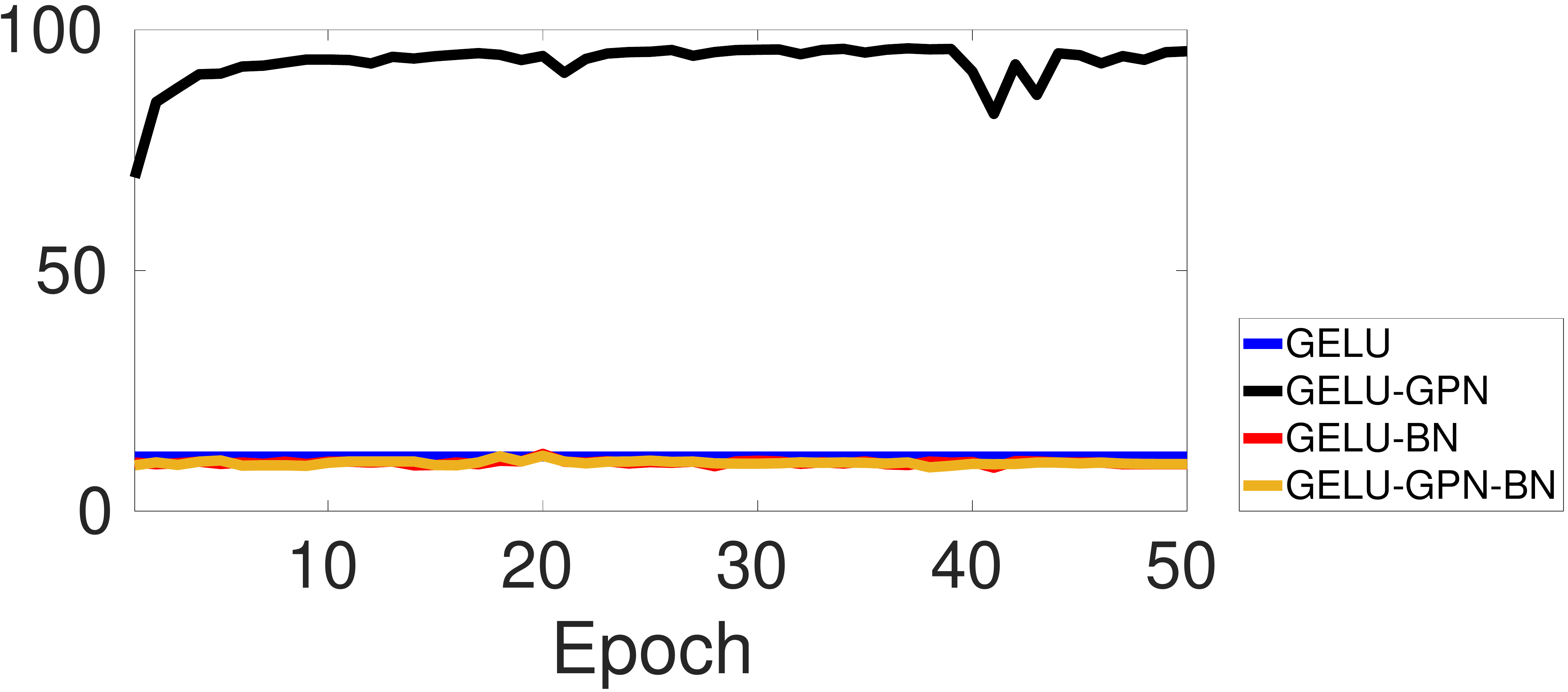}}

\vspace{0.5cm}
\caption{Test accuracy (percentage) during training on MNIST.}
\label{exp:test_acc_mnist}
\end{figure}

\clearpage

\begin{figure}[h!]
\centering
\subfloat[Tanh.]{\includegraphics[width=0.4\textwidth]{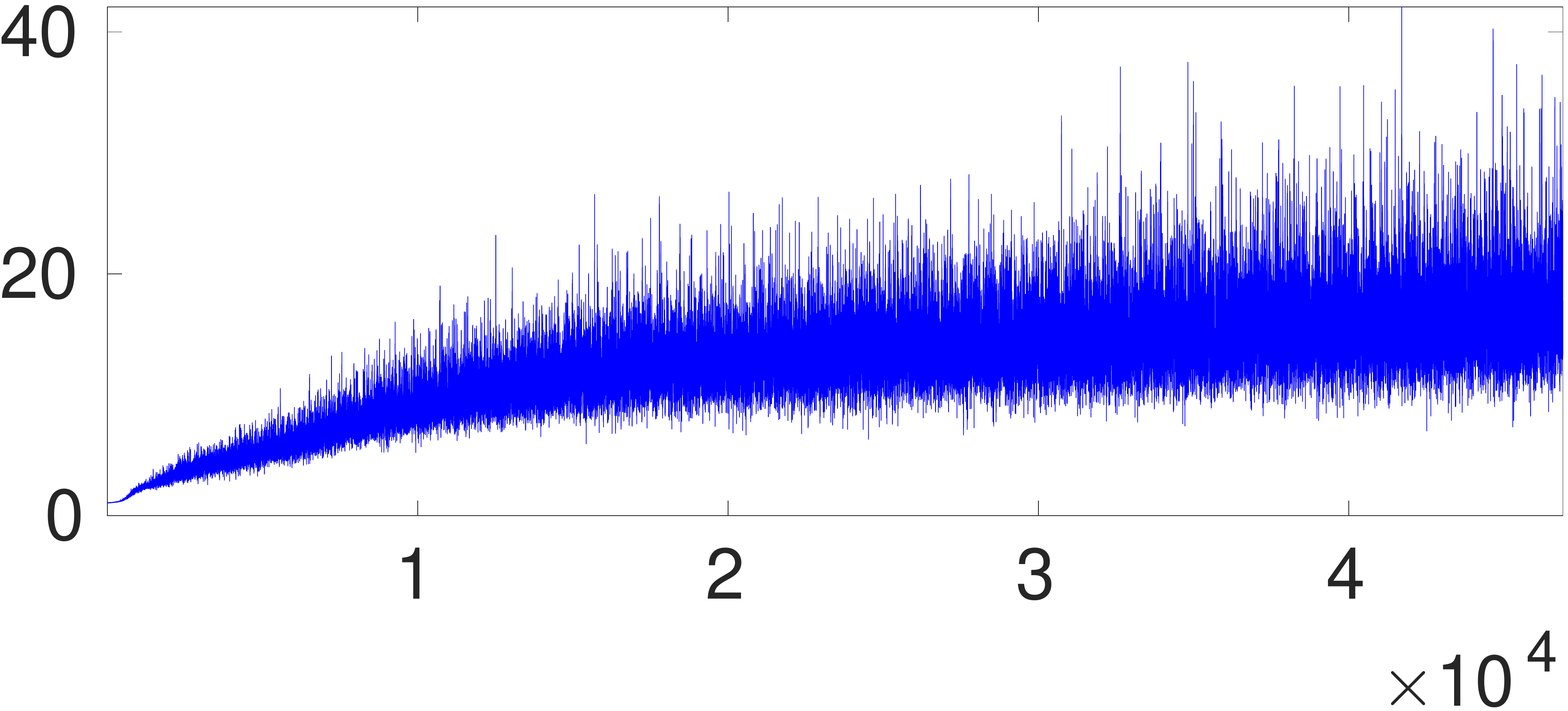}}
\hspace{0.5cm}
\subfloat[Tanh-GPN.]{\includegraphics[width=0.4\textwidth]{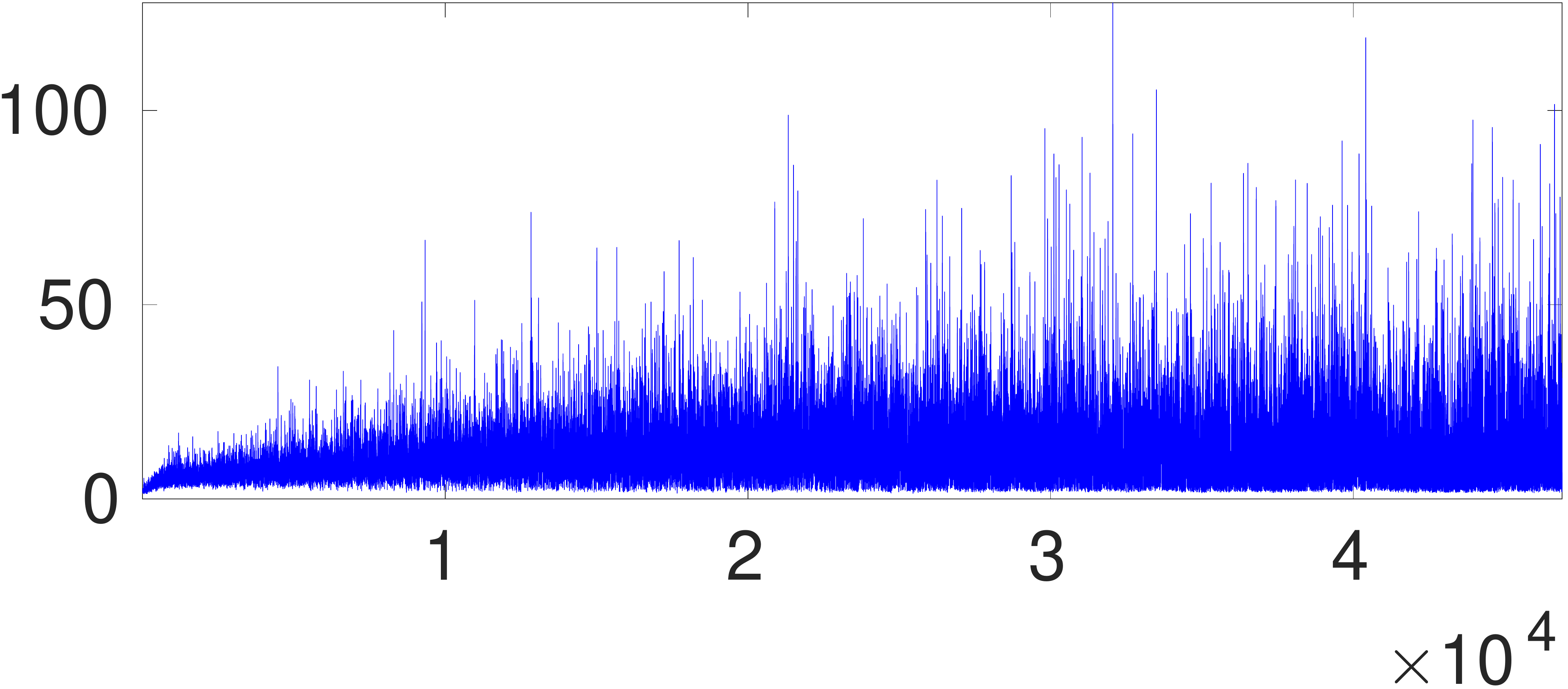}}

\subfloat[Tanh-BN.]{\includegraphics[width=0.4\textwidth]{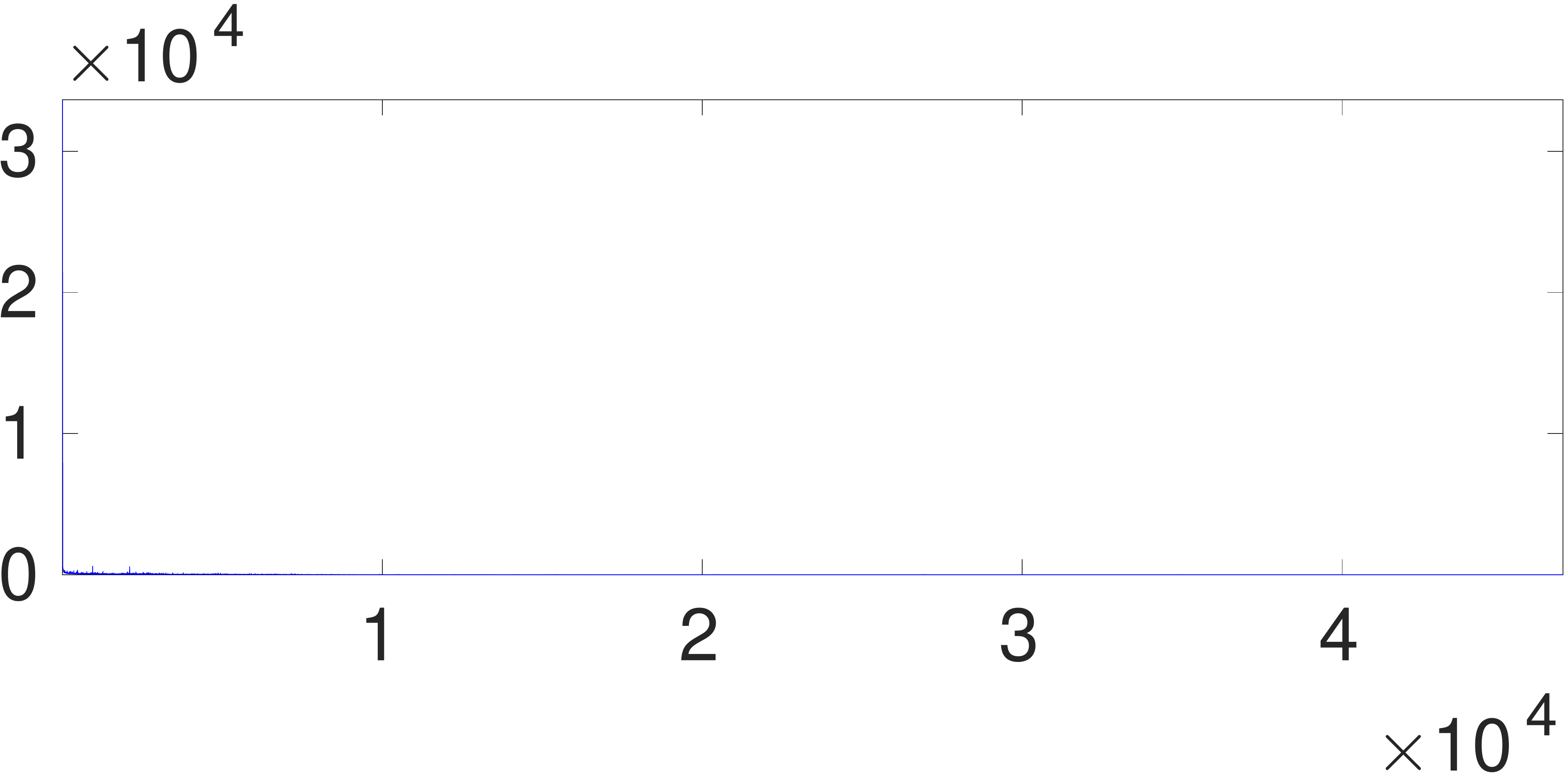}}
\hspace{0.5cm}
\subfloat[Tanh-GPN-BN.]{\includegraphics[width=0.4\textwidth]{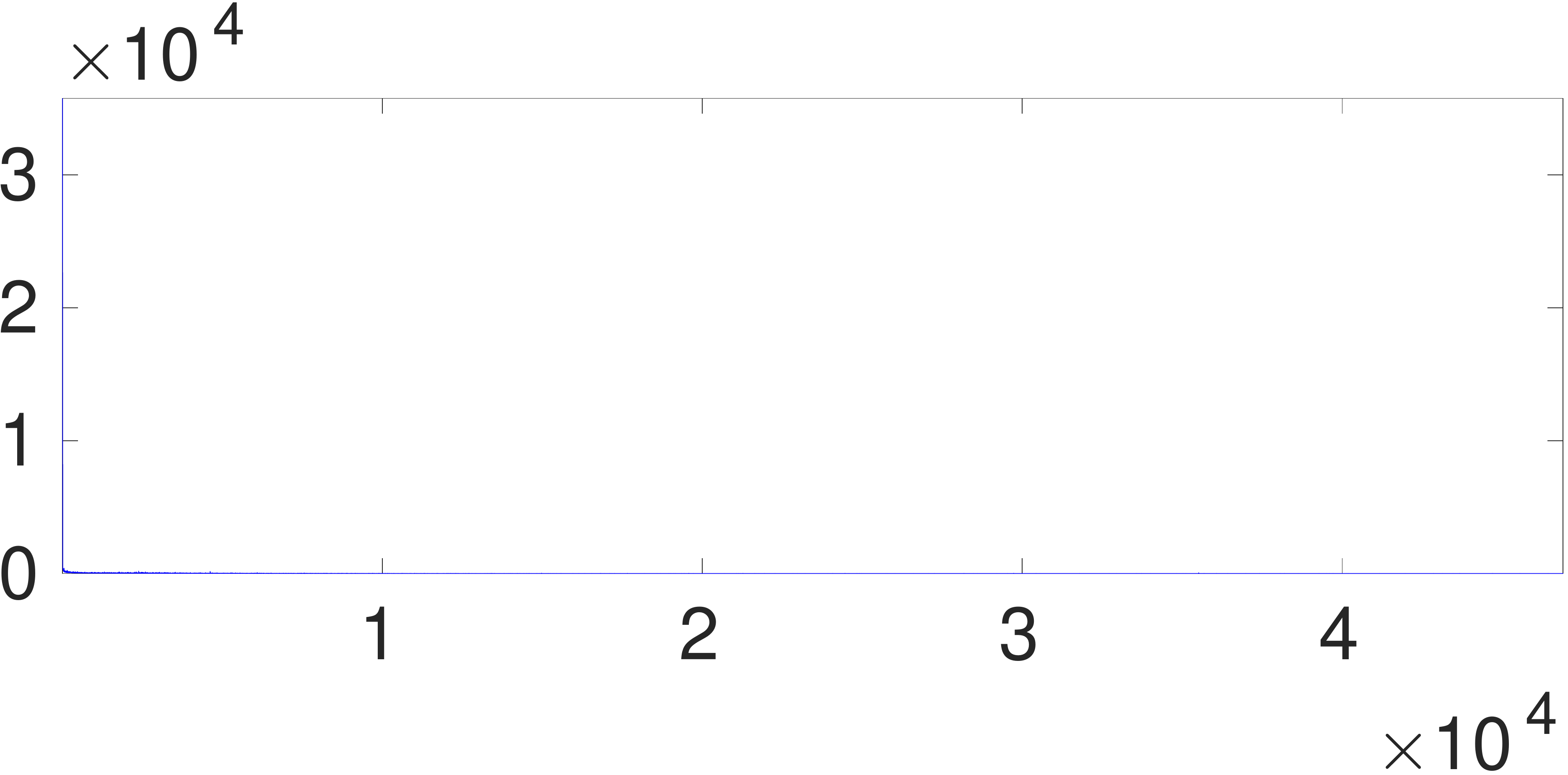}}

\subfloat[ReLU.]{\includegraphics[width=0.4\textwidth]{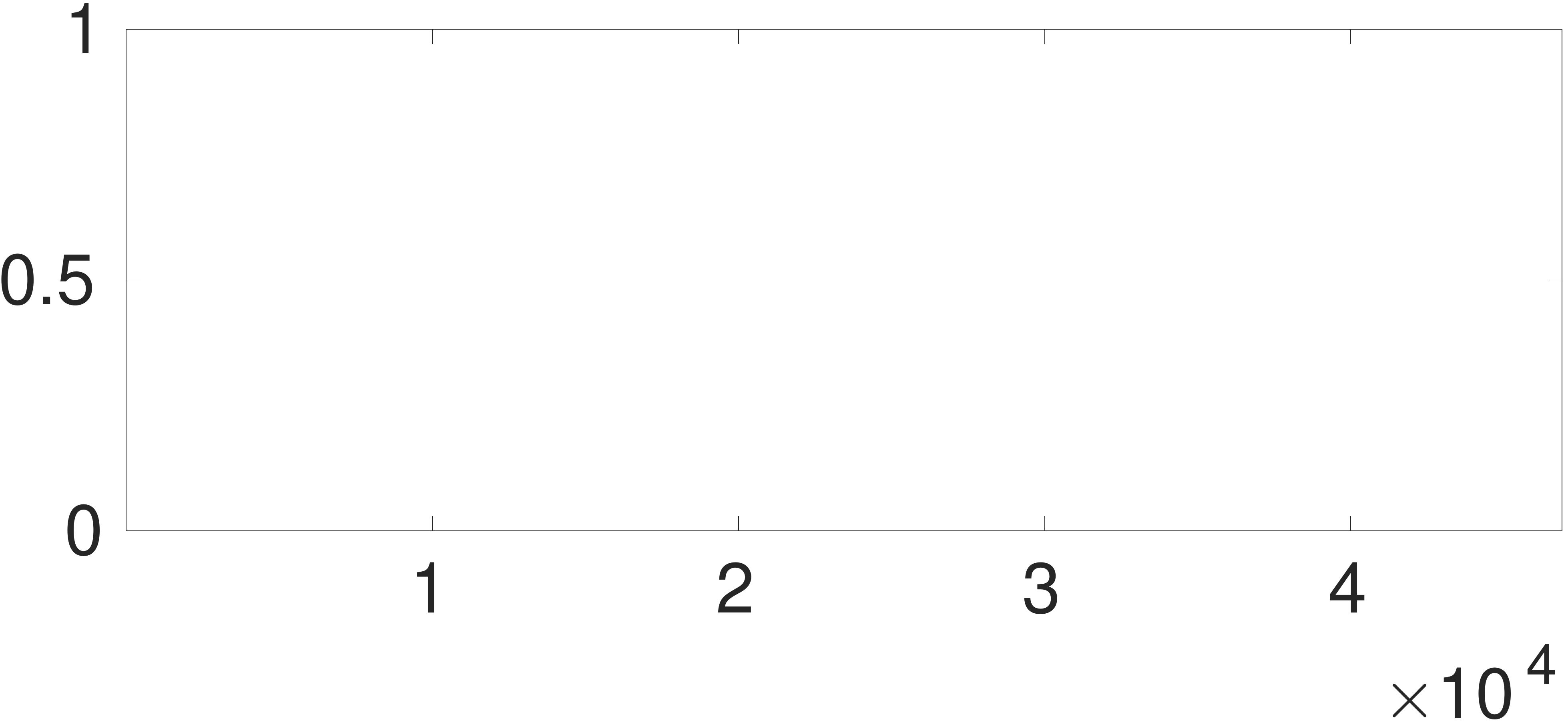}}
\hspace{0.5cm}
\subfloat[ReLU-GPN.]{\includegraphics[width=0.4\textwidth]{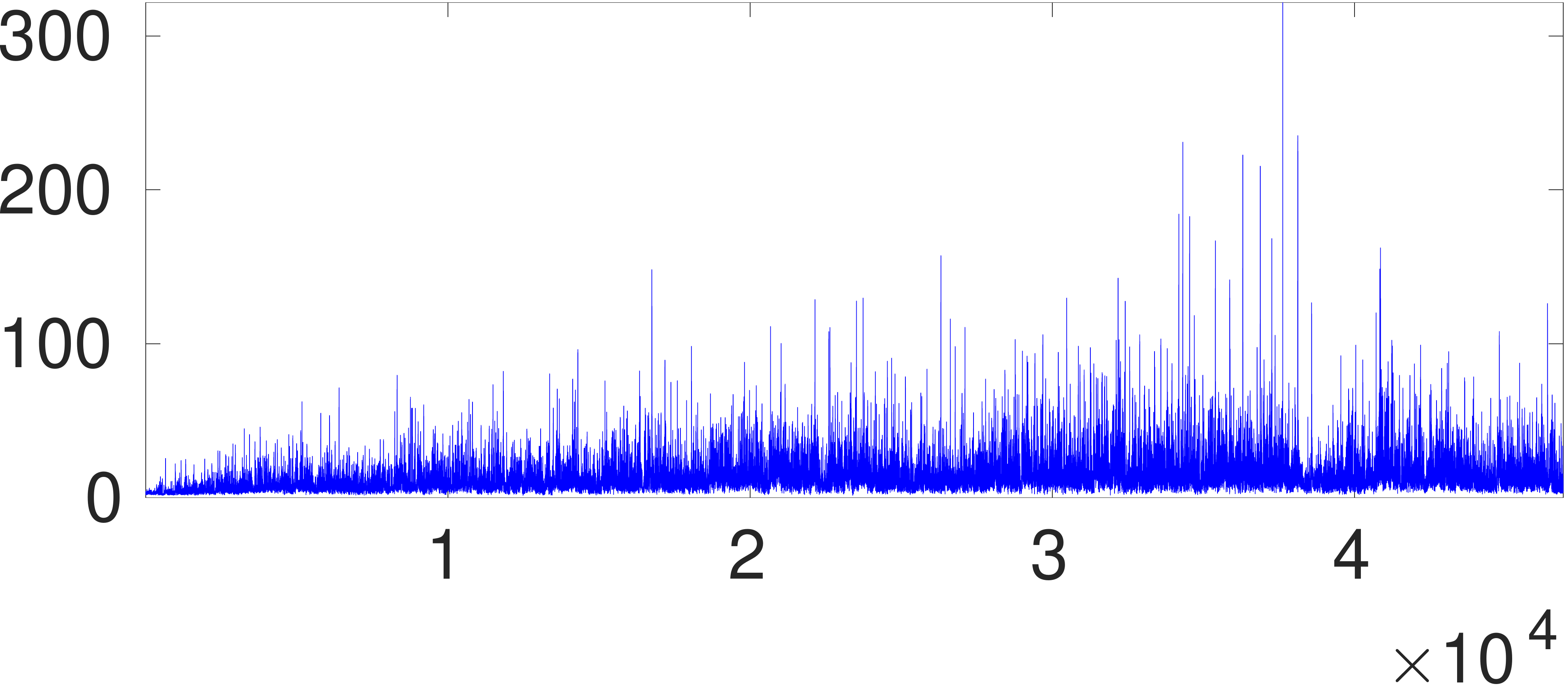}}

\subfloat[ReLU-BN.]{\includegraphics[width=0.4\textwidth]{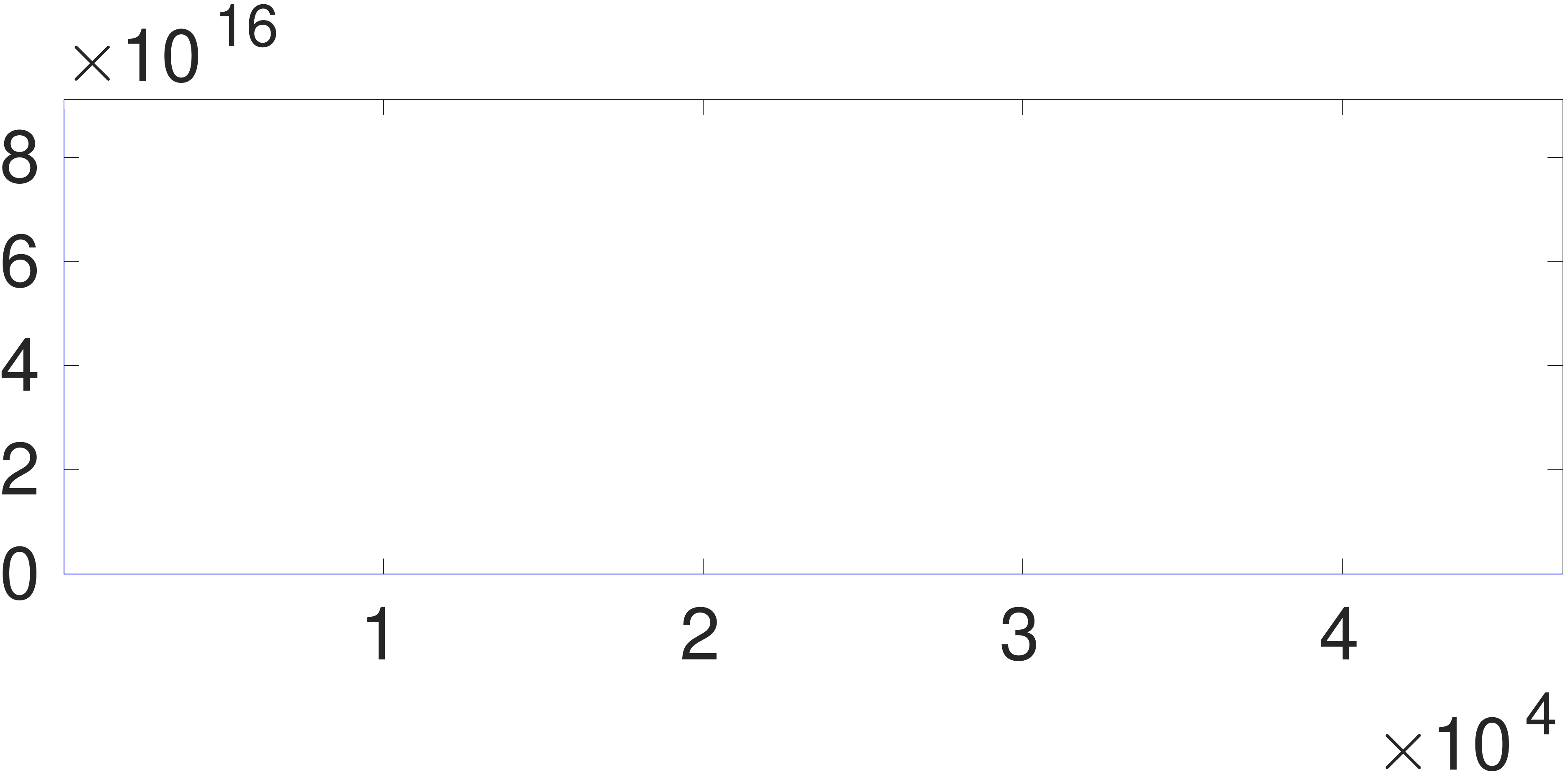}}
\hspace{0.5cm}
\subfloat[ReLU-GPN-BN.]{\includegraphics[width=0.4\textwidth]{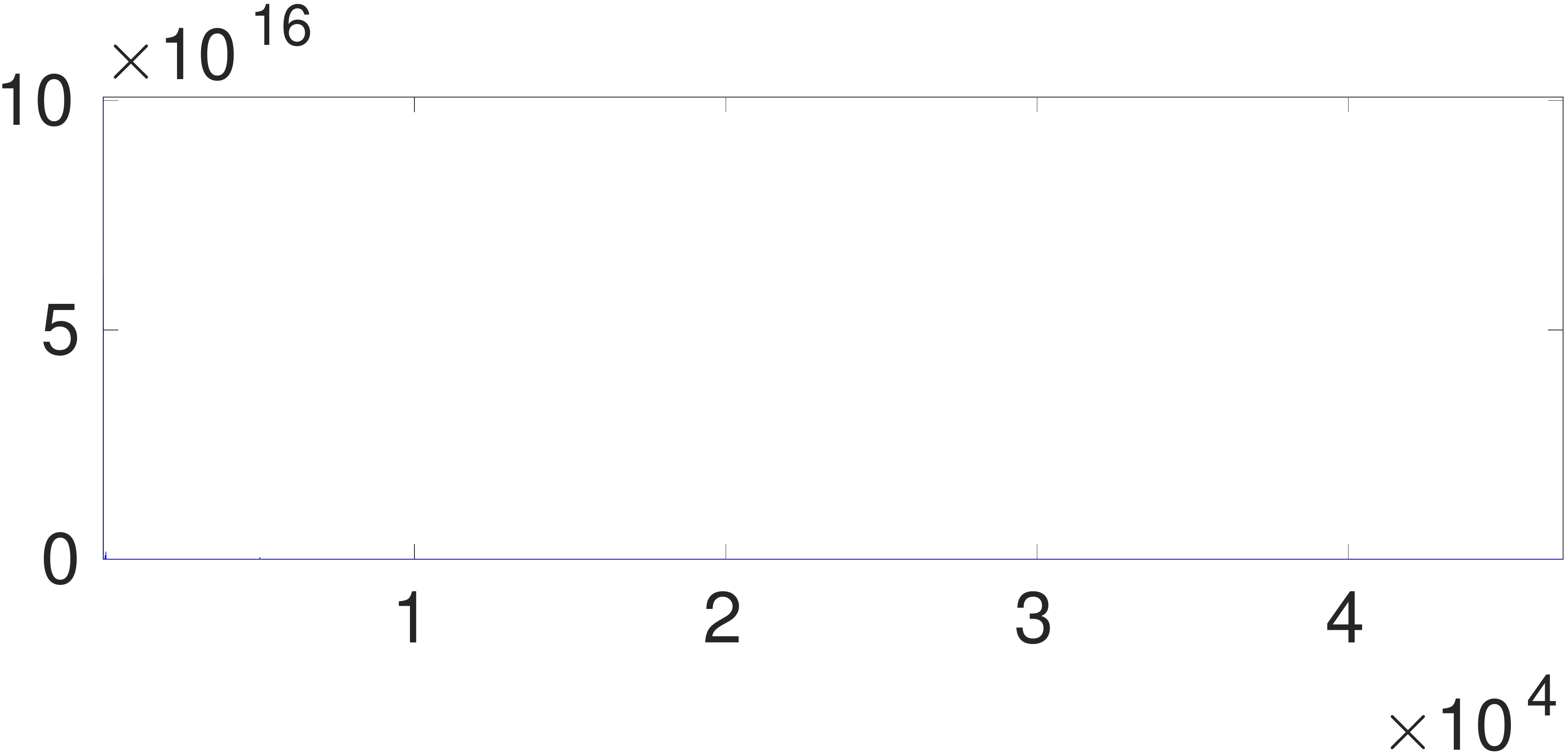}}

\subfloat[LeakyReLU.]{\includegraphics[width=0.4\textwidth]{img/empty_mnist_grad.pdf}}
\hspace{0.5cm}
\subfloat[LeakyReLU-GPN.]{\includegraphics[width=0.4\textwidth]{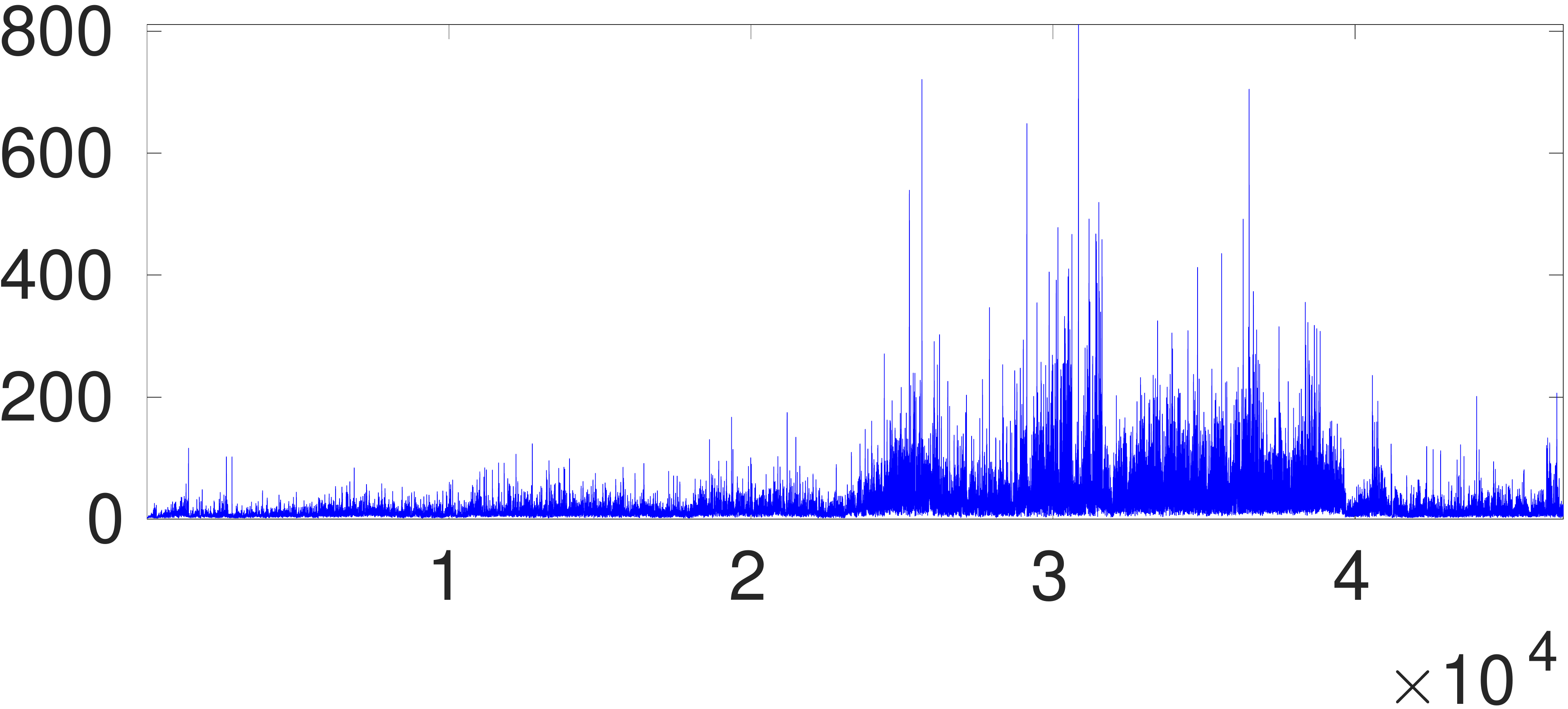}}

\subfloat[LeakyReLU-BN.]{\includegraphics[width=0.4\textwidth]{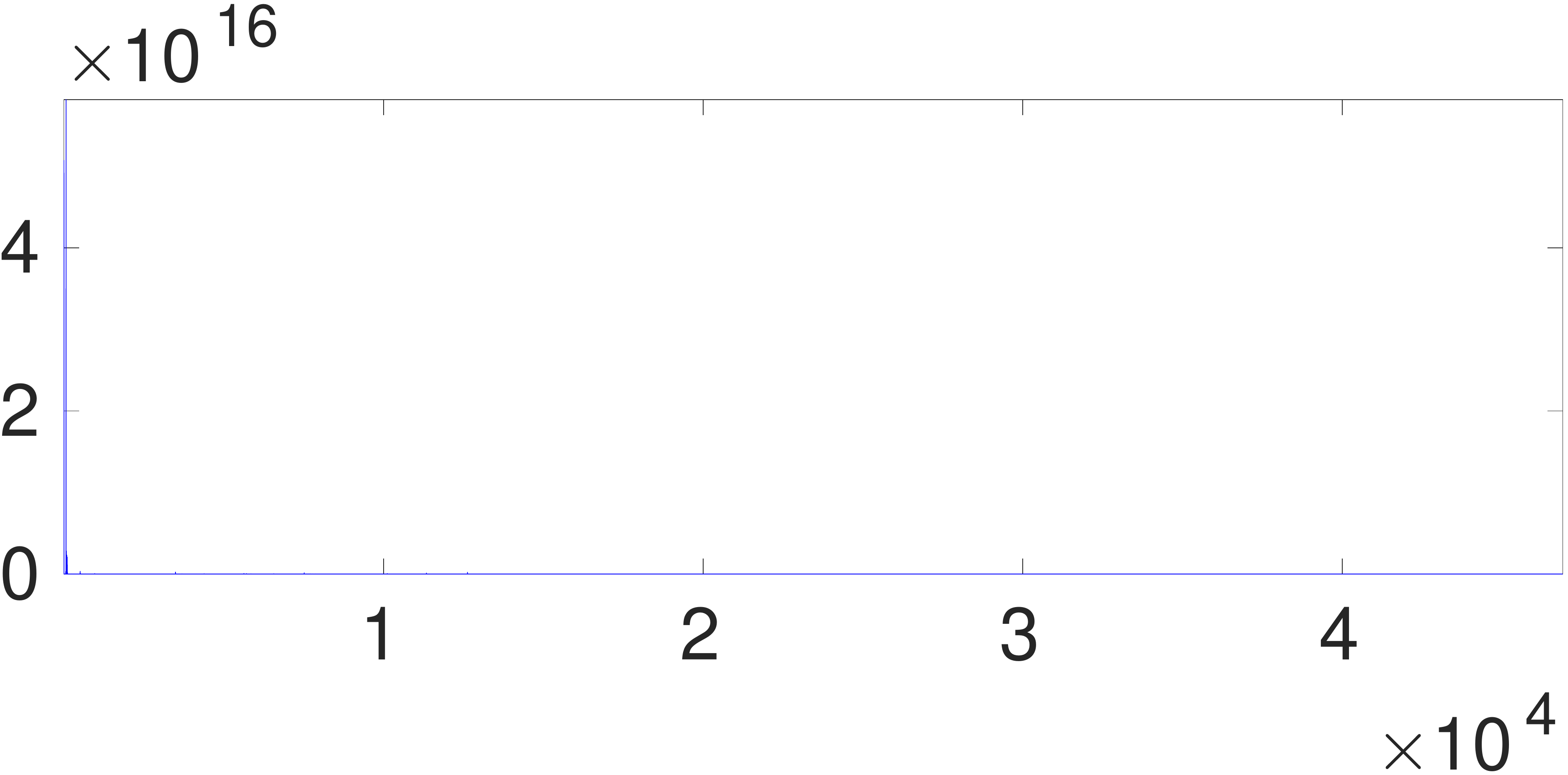}}
\hspace{0.5cm}
\subfloat[LeakyReLU-GPN-BN.]{\includegraphics[width=0.4\textwidth]{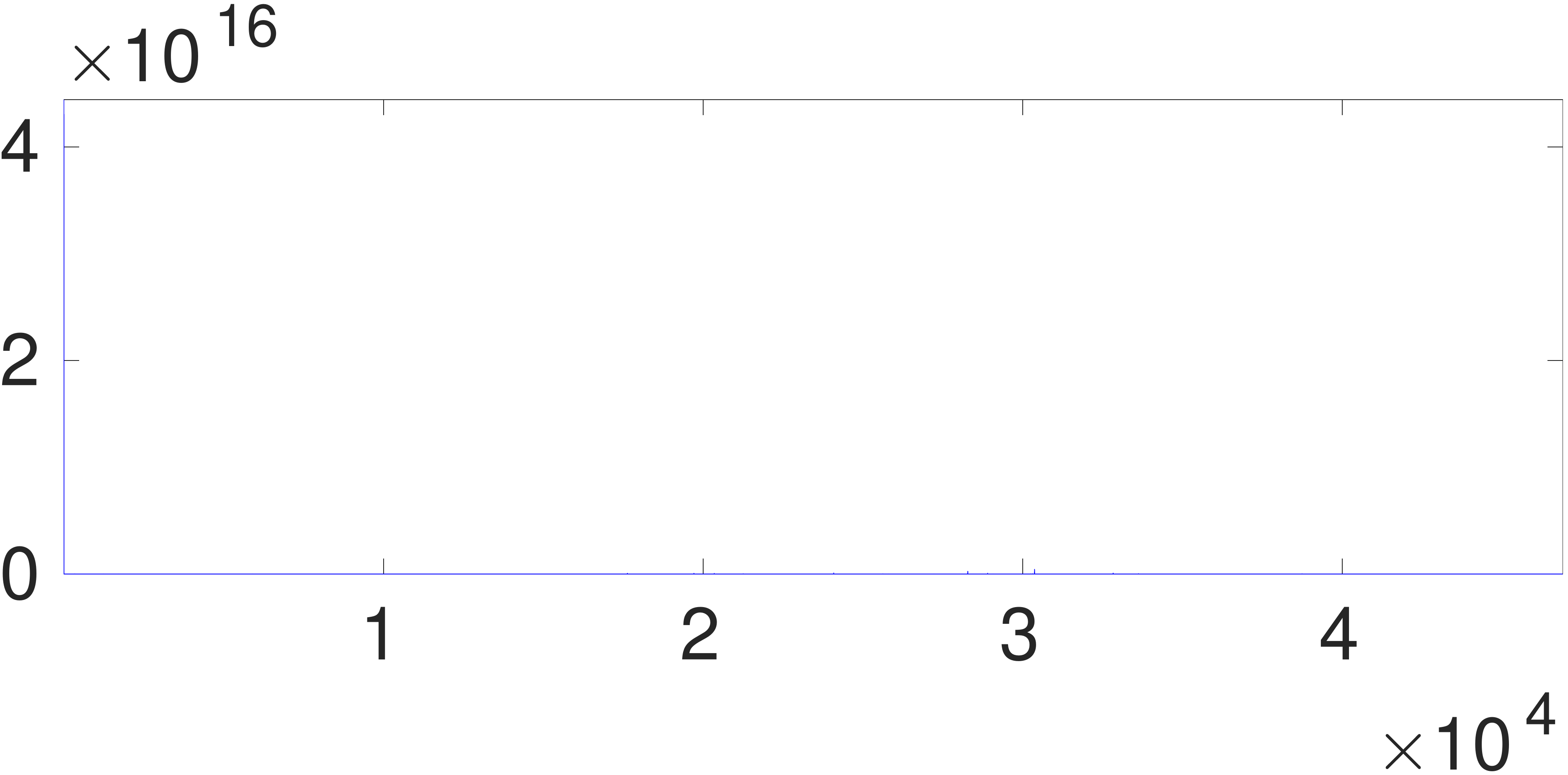}}

\vspace{0.5cm}
\caption{Gradient norm ratio during training on MNIST. Horizontal axis denotes the mini-batch updates. Vertical axis denotes the gradient norm ratio $\max_l\|\frac{\partial E}{\partial \mathbf{V}^{(l)}}\|_F / \min_l\|\frac{\partial E}{\partial \mathbf{V}^{(l)}}\|_F$. The gradient vanishes ($\|\frac{\partial E}{\partial \mathbf{V}^{(l)}}\|_F\approx 0$) for ReLU and LeakyReLU during training and hence the plots are empty.}

\label{fig:grad_1}
\end{figure}

\begin{figure}[h!]
\centering

\subfloat[ELU.]{\includegraphics[width=0.4\textwidth]{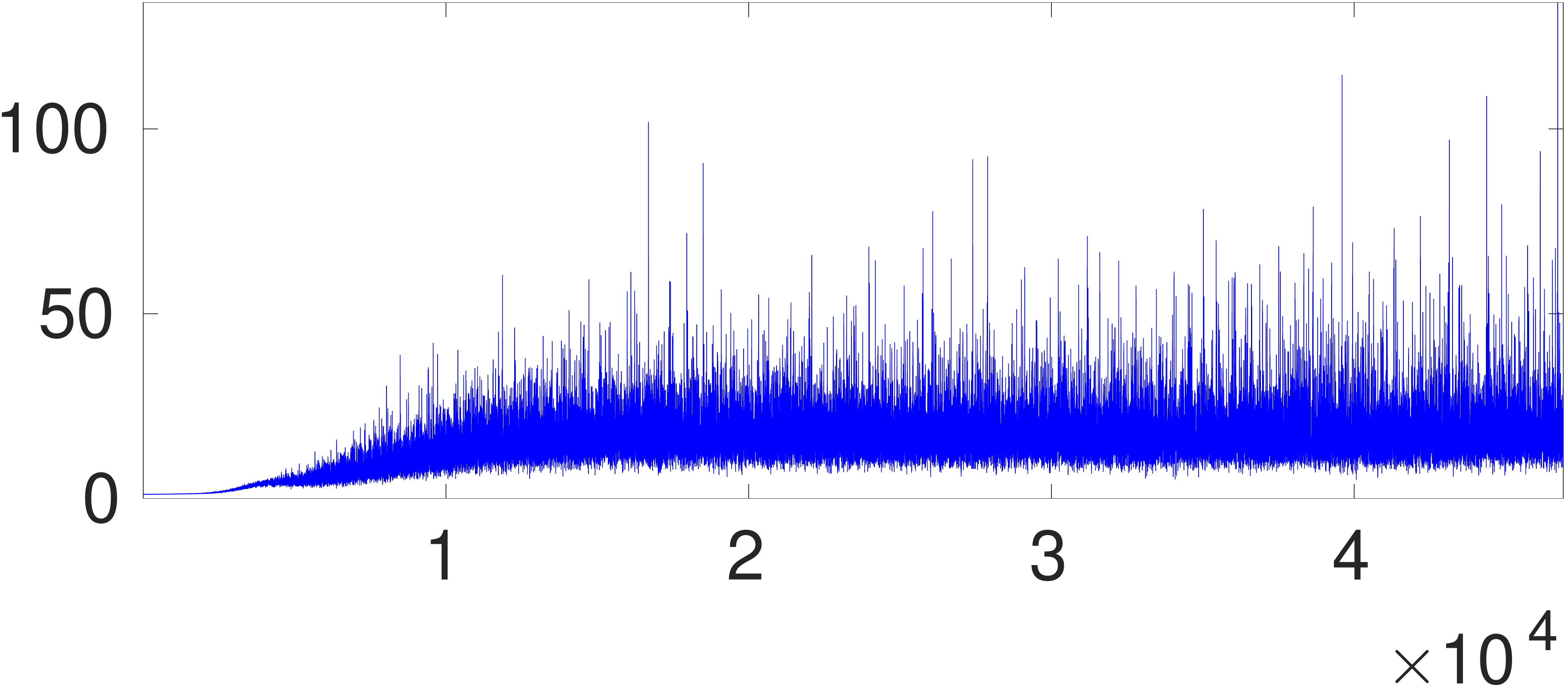}}
\hspace{0.5cm}
\subfloat[ELU-GPN.]{\includegraphics[width=0.4\textwidth]{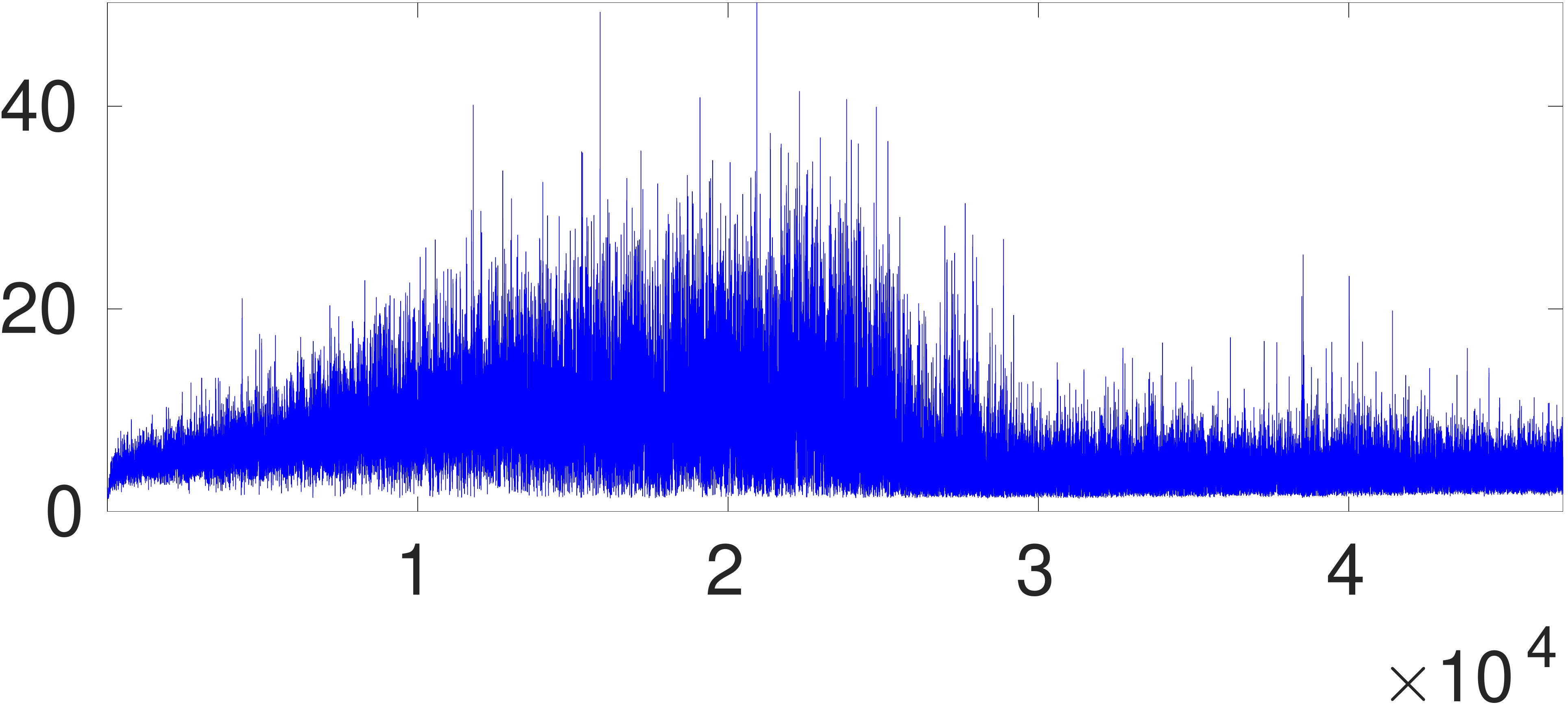}}

\subfloat[ELU-BN.]{\includegraphics[width=0.4\textwidth]{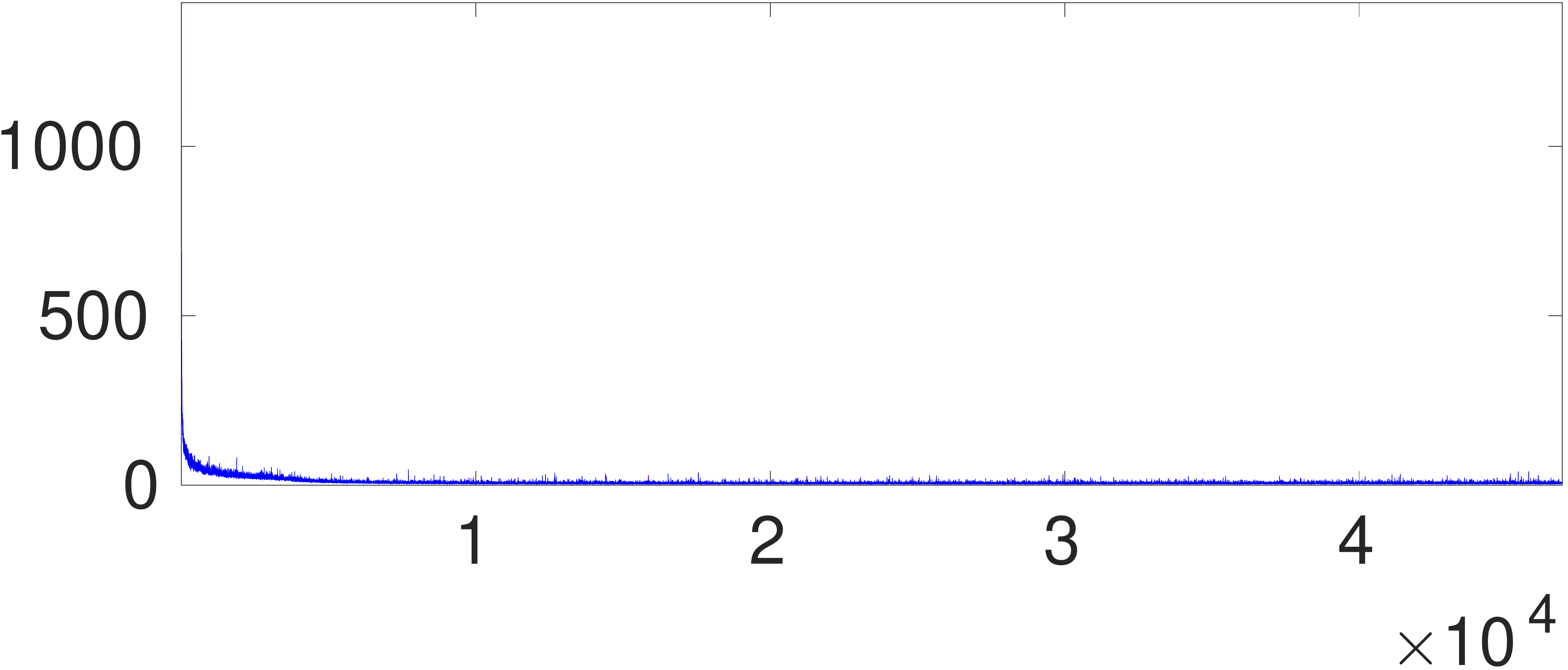}}
\hspace{0.5cm}
\subfloat[ELU-GPN-BN.]{\includegraphics[width=0.4\textwidth]{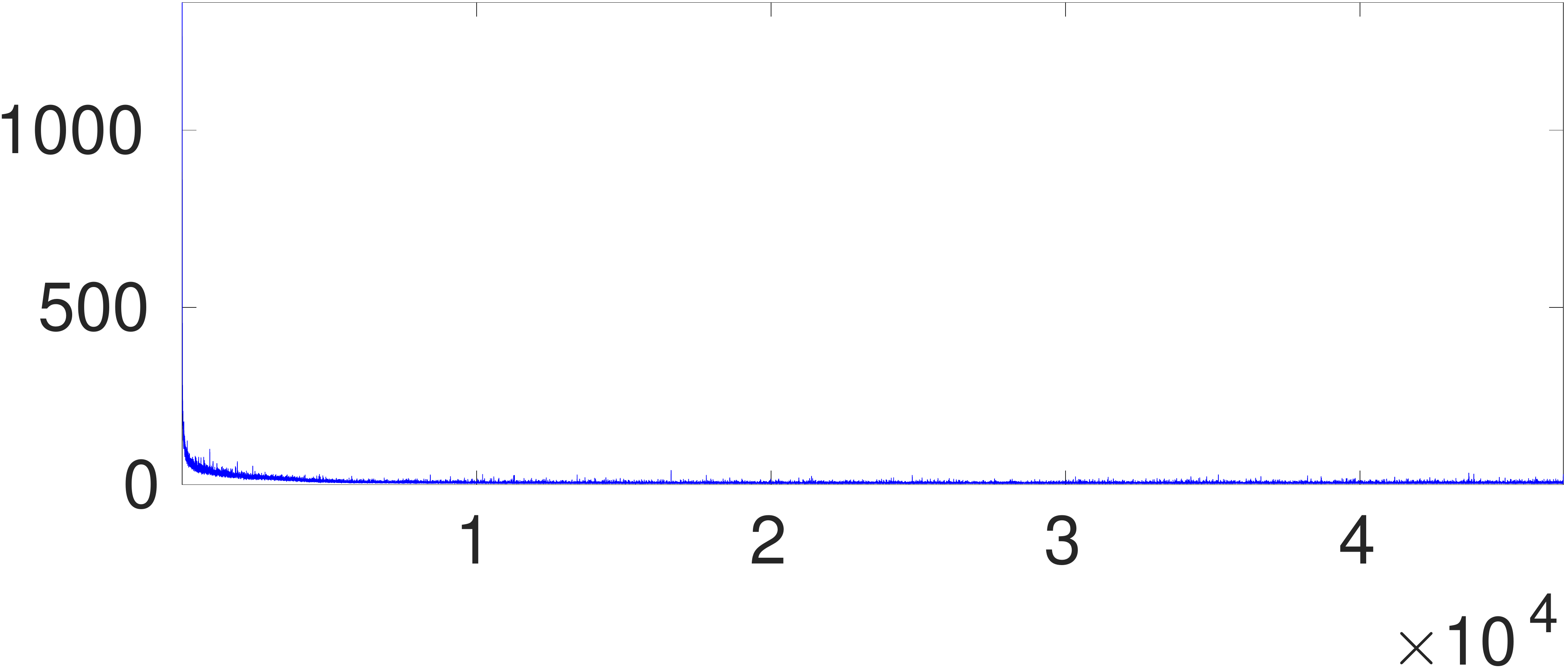}}

\subfloat[SELU.]{\includegraphics[width=0.4\textwidth]{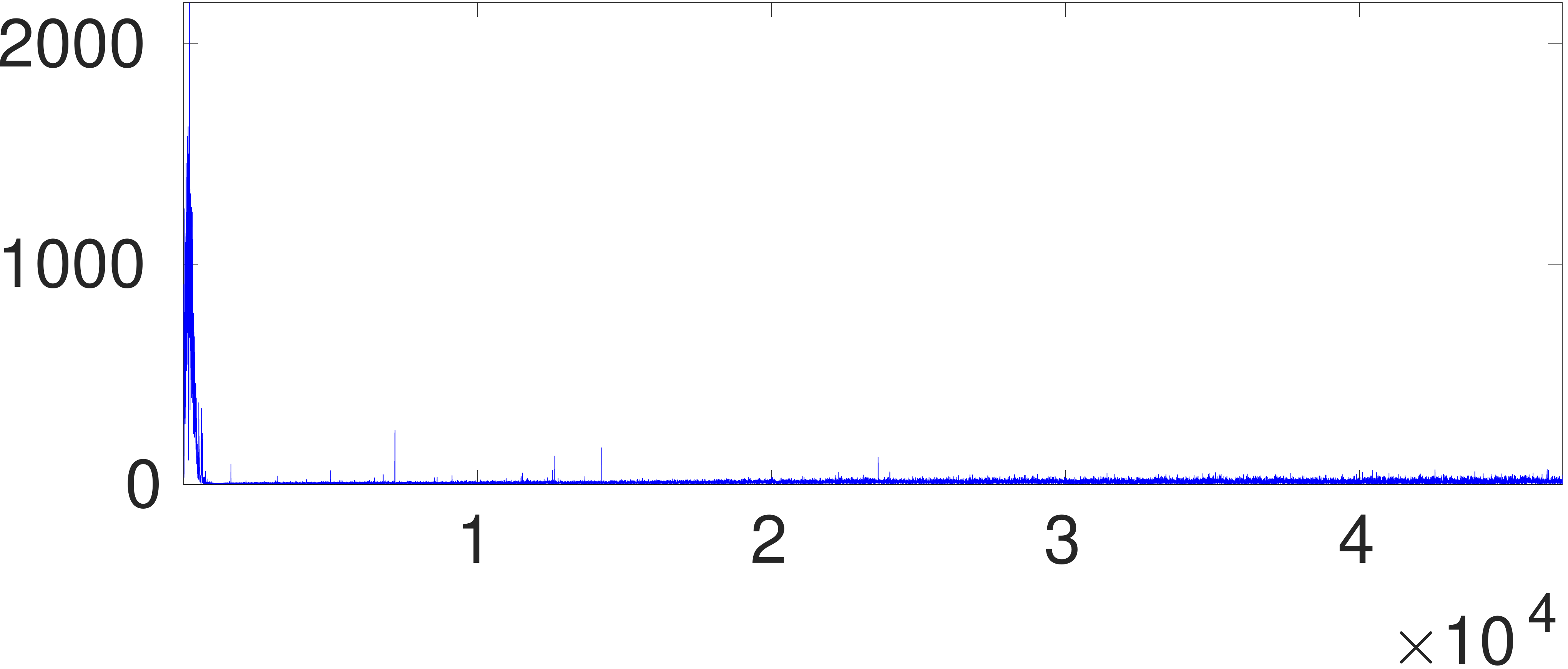}}
\hspace{0.5cm}
\subfloat[SELU-GPN.]{\includegraphics[width=0.4\textwidth]{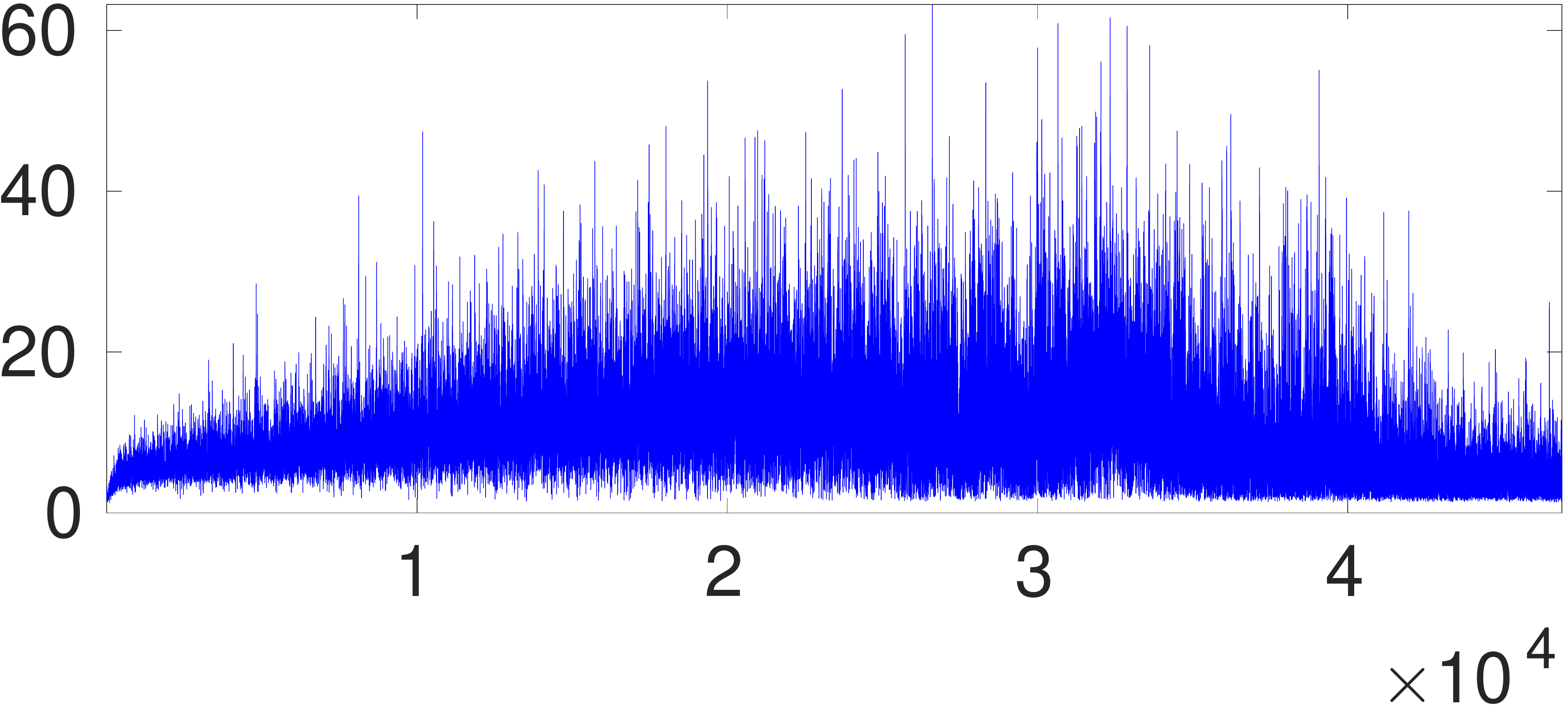}}

\subfloat[SELU-BN.]{\includegraphics[width=0.4\textwidth]{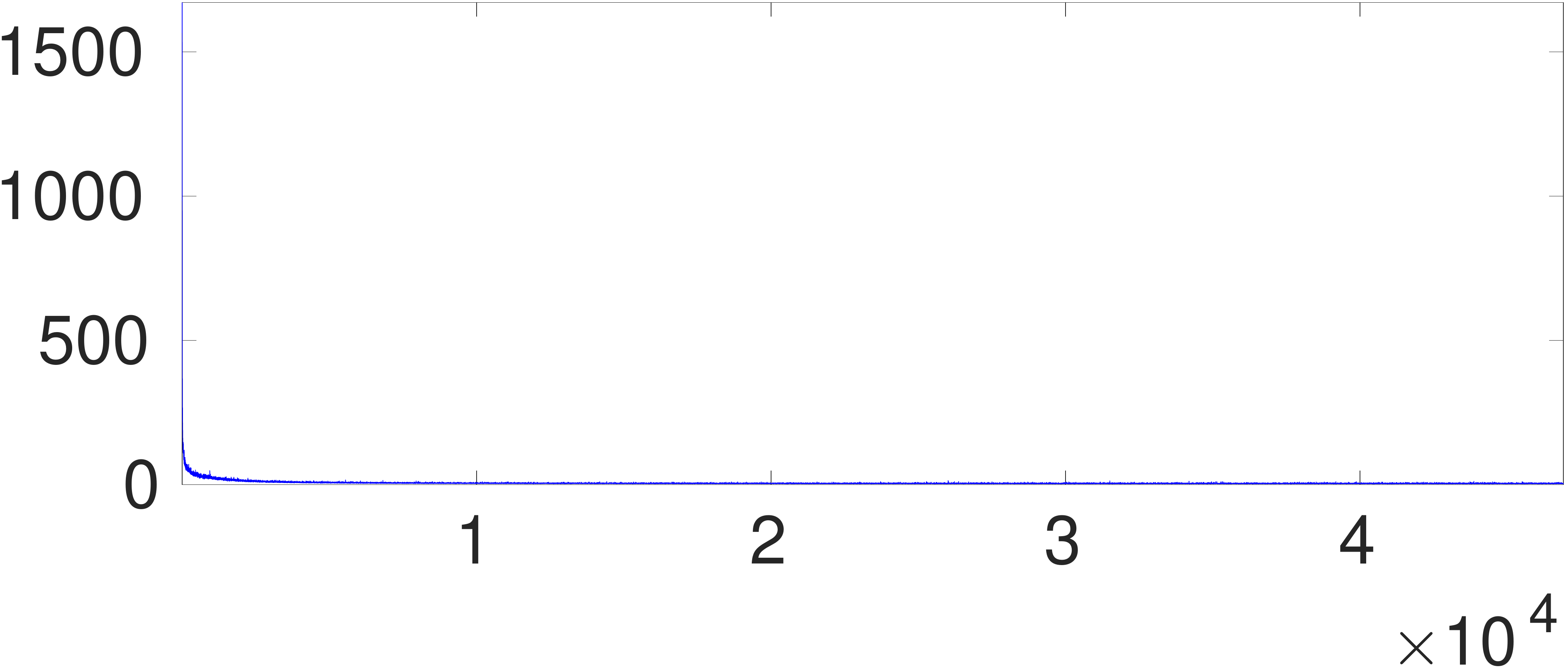}}
\hspace{0.5cm}
\subfloat[SELU-GPN-BN.]{\includegraphics[width=0.4\textwidth]{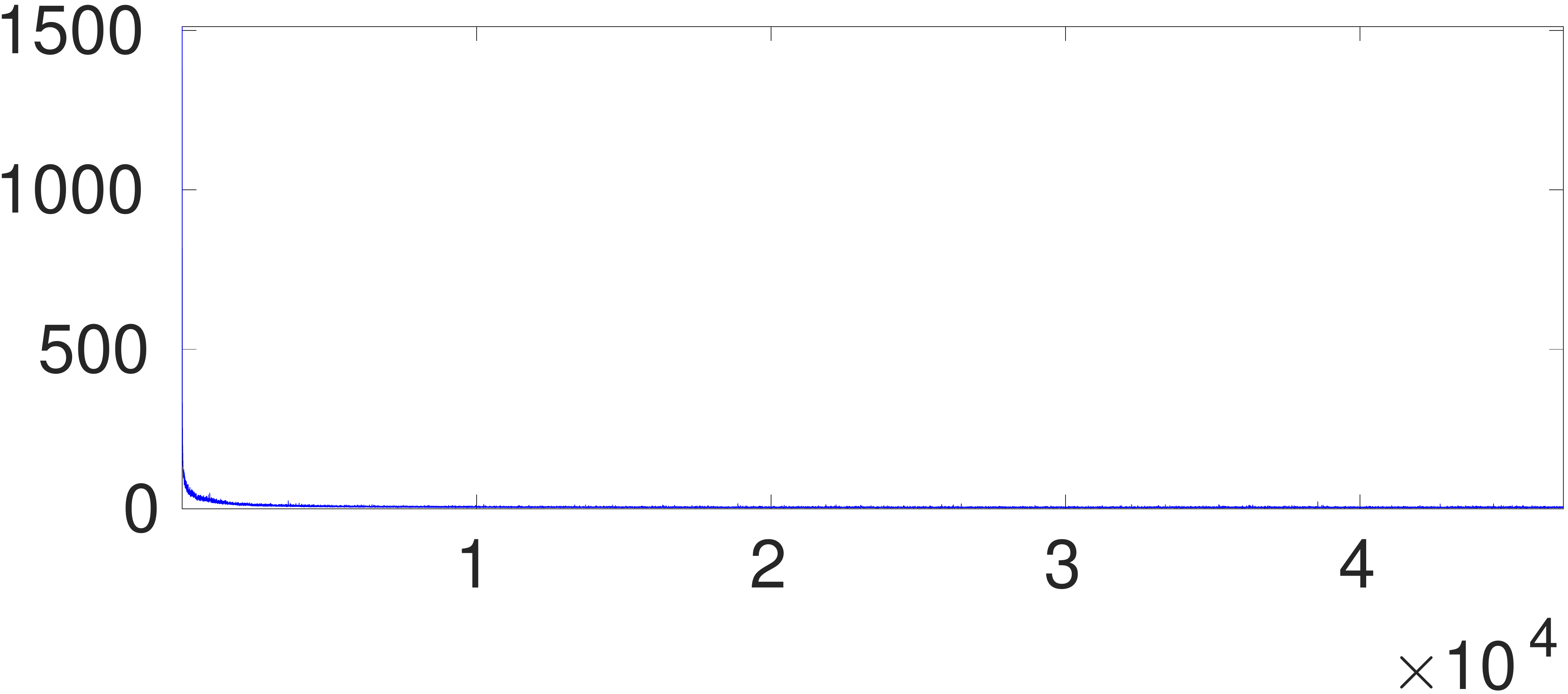}}

\subfloat[GELU.]{\includegraphics[width=0.4\textwidth]{img/empty_mnist_grad.pdf}}
\hspace{0.5cm}
\subfloat[GELU-GPN.]{\includegraphics[width=0.4\textwidth]{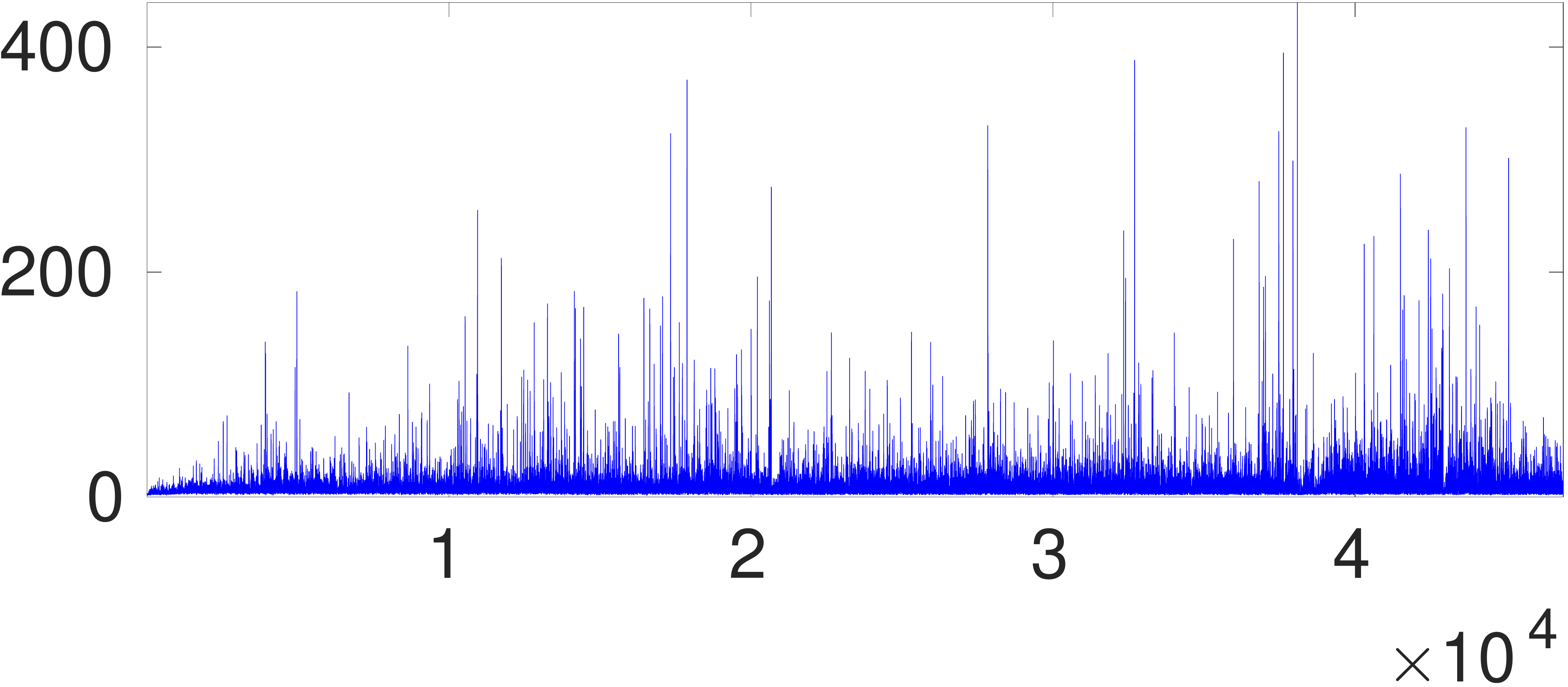}}

\subfloat[GELU-BN.]{\includegraphics[width=0.4\textwidth]{img/empty_mnist_grad.pdf}}
\hspace{0.5cm}
\subfloat[GELU-GPN-BN.]{\includegraphics[width=0.4\textwidth]{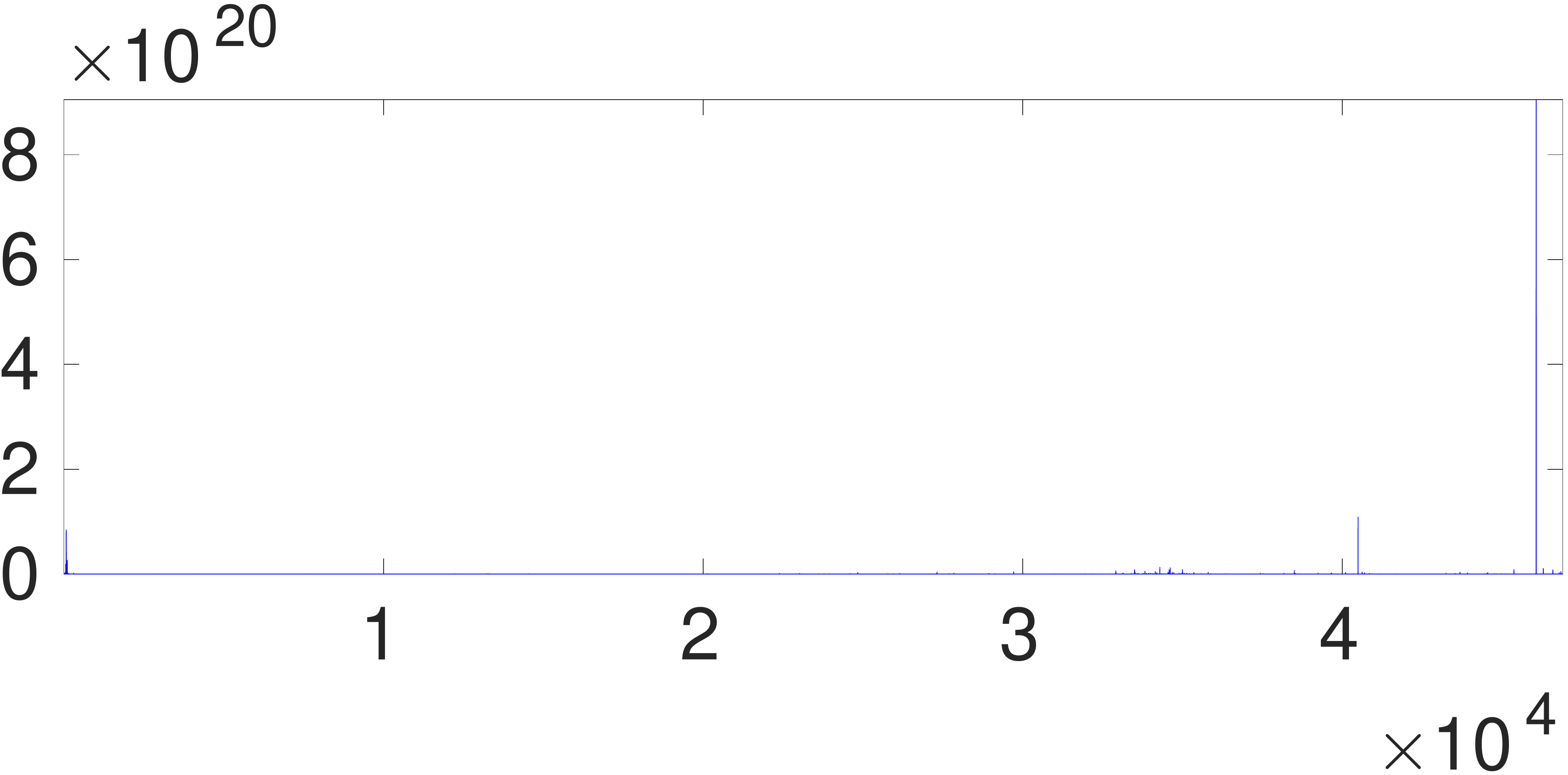}}

\vspace{0.5cm}
\caption{Gradient norm ratio during training on MNIST. Horizontal axis denotes the mini-batch updates. Vertical axis denotes the gradient norm ratio $\max_l\|\frac{\partial E}{\partial \mathbf{V}^{(l)}}\|_F / \min_l\|\frac{\partial E}{\partial \mathbf{V}^{(l)}}\|_F$. The gradient vanishes ($\|\frac{\partial E}{\partial \mathbf{V}^{(l)}}\|_F\approx 0$) for GELU and GELU-BN during training and hence the plots are empty.}
\label{fig:grad_2}

\end{figure}

\begin{figure}[h!]
\centering
\subfloat[Tanh.]{\includegraphics[width=0.4\textwidth]{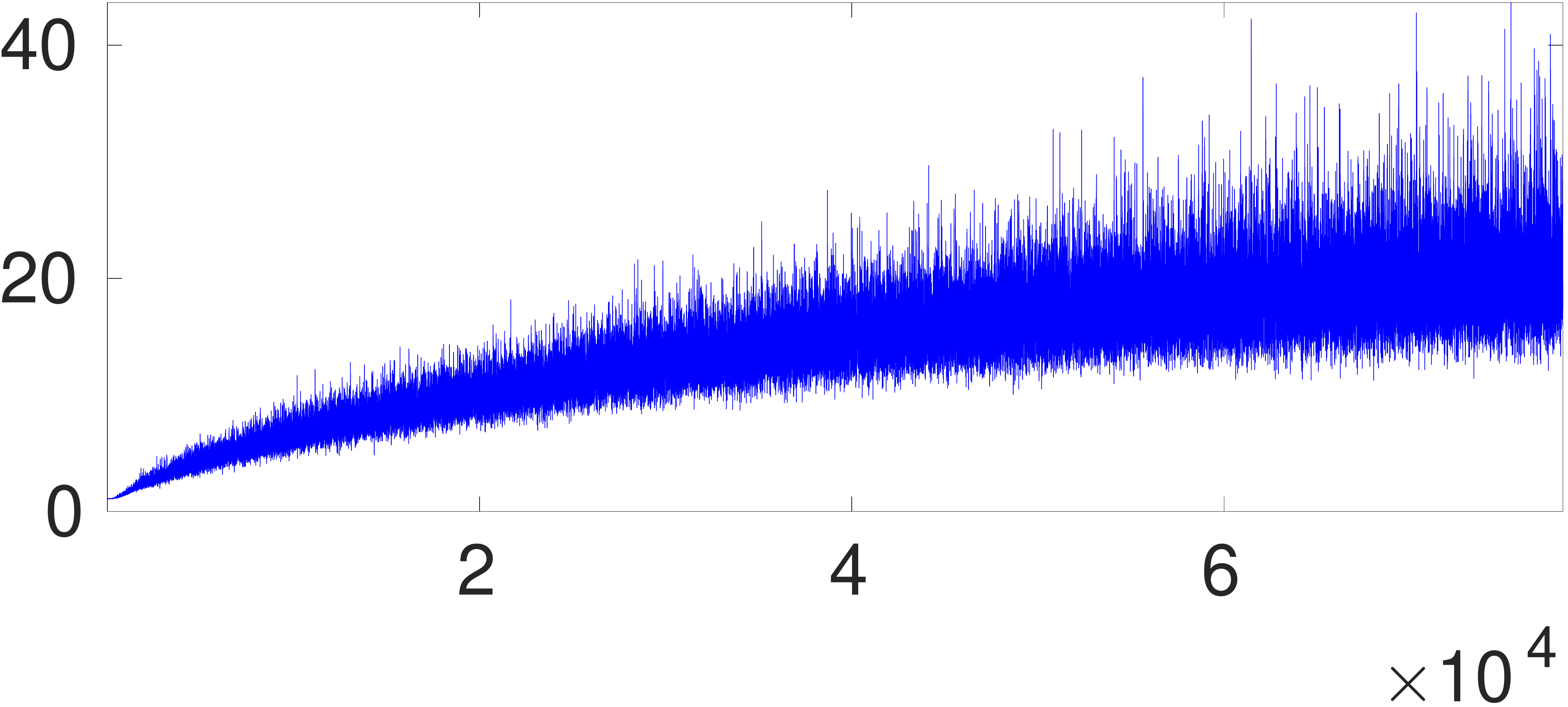}}
\hspace{0.5cm}
\subfloat[Tanh-GPN.]{\includegraphics[width=0.4\textwidth]{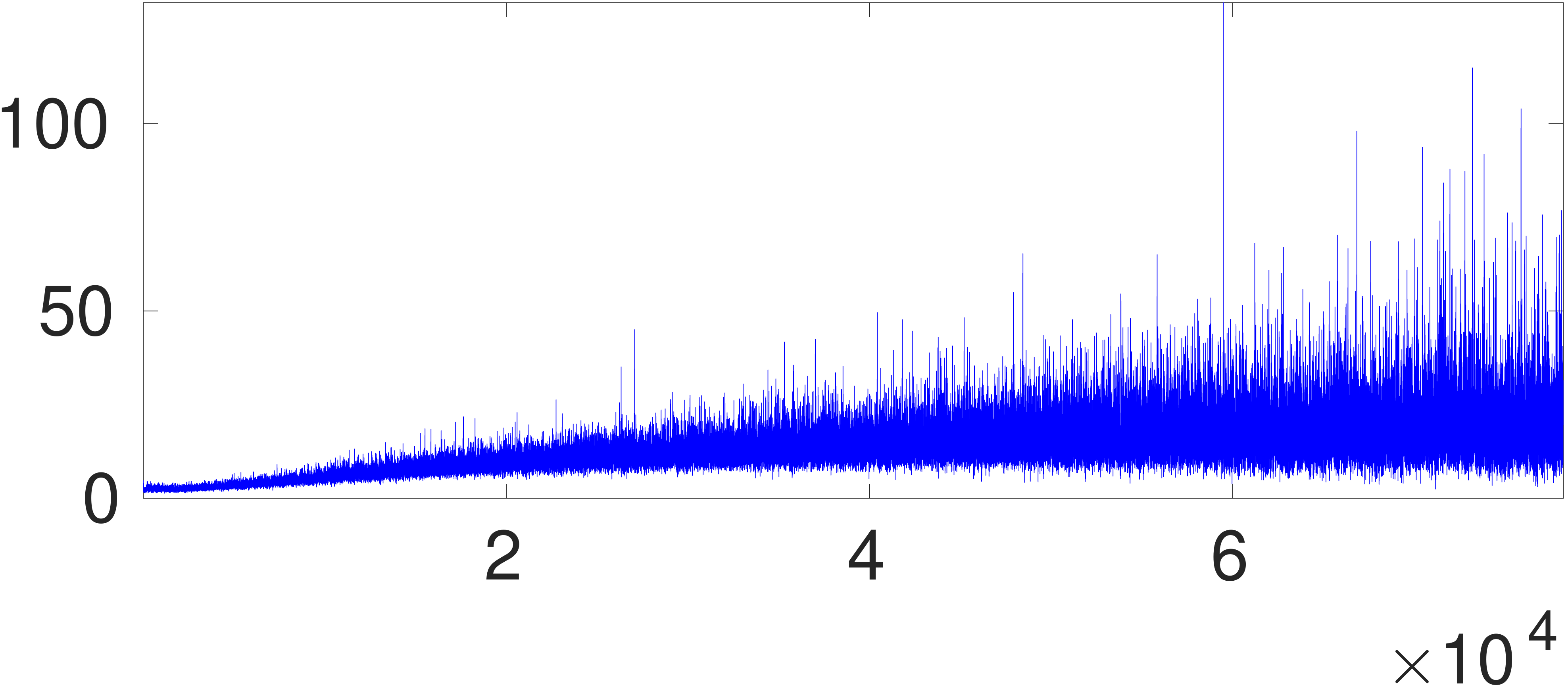}}

\subfloat[Tanh-BN.]{\includegraphics[width=0.4\textwidth]{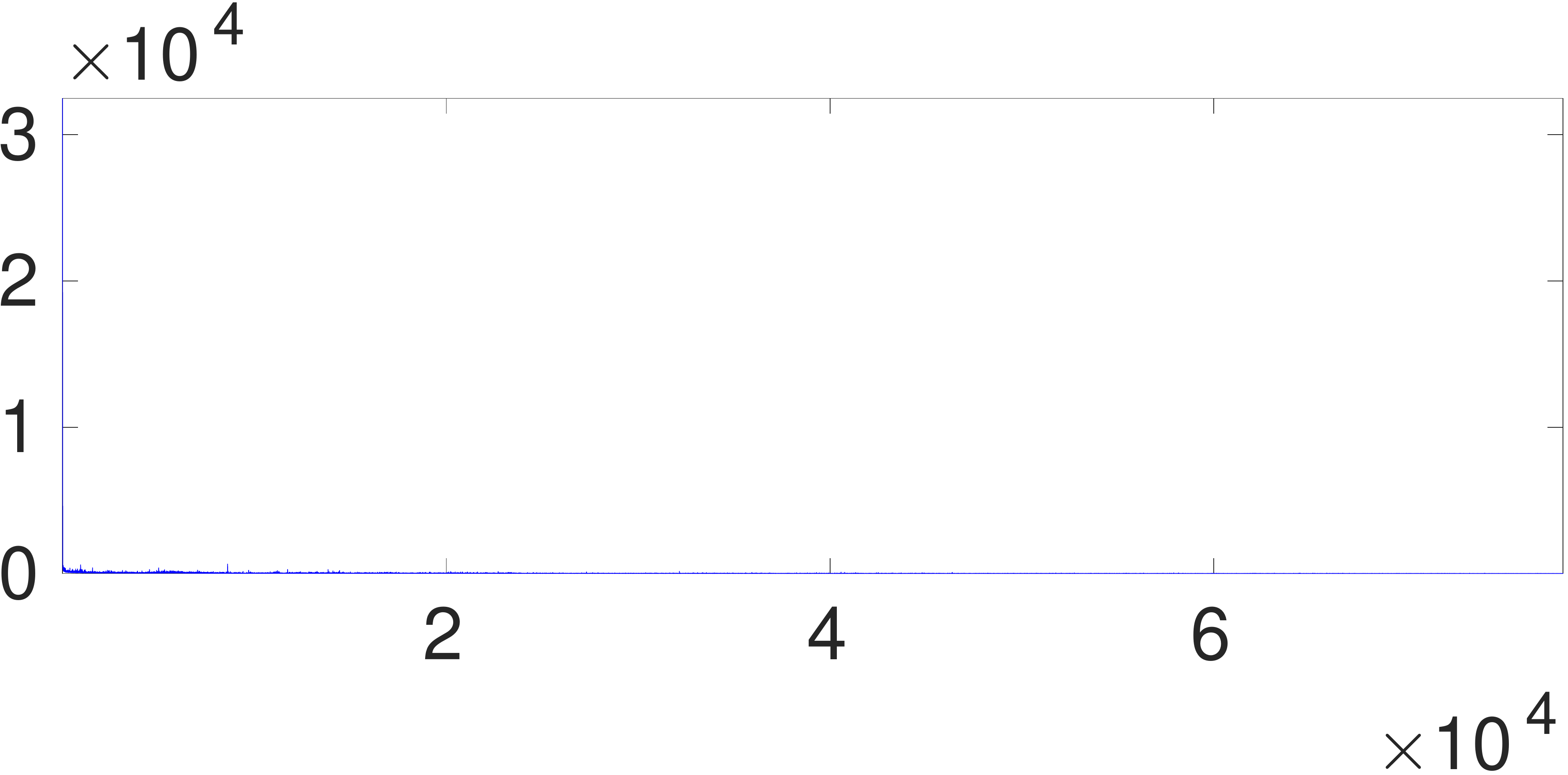}}
\hspace{0.5cm}
\subfloat[Tanh-GPN-BN.]{\includegraphics[width=0.4\textwidth]{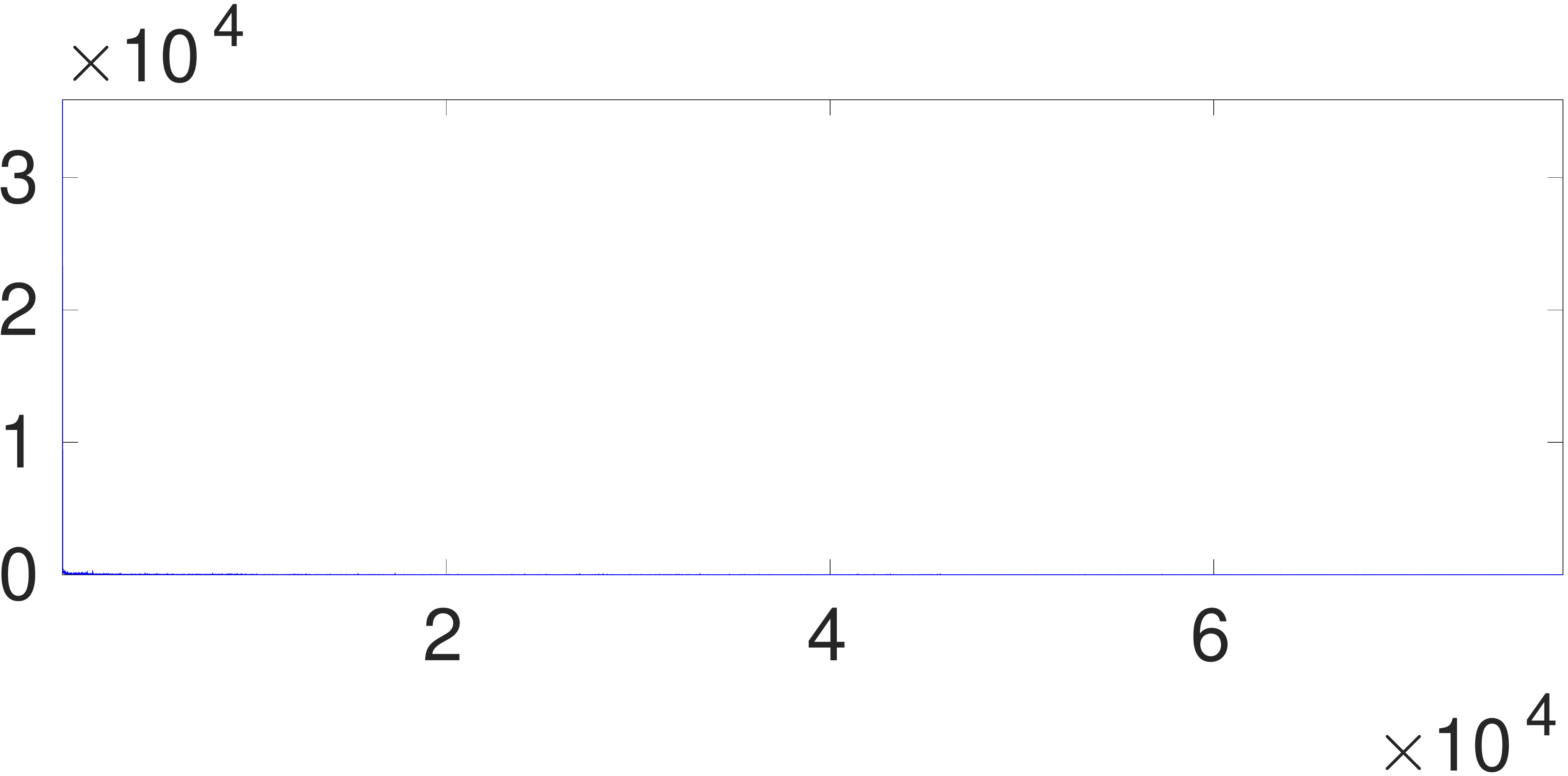}}

\subfloat[ReLU.]{\includegraphics[width=0.4\textwidth]{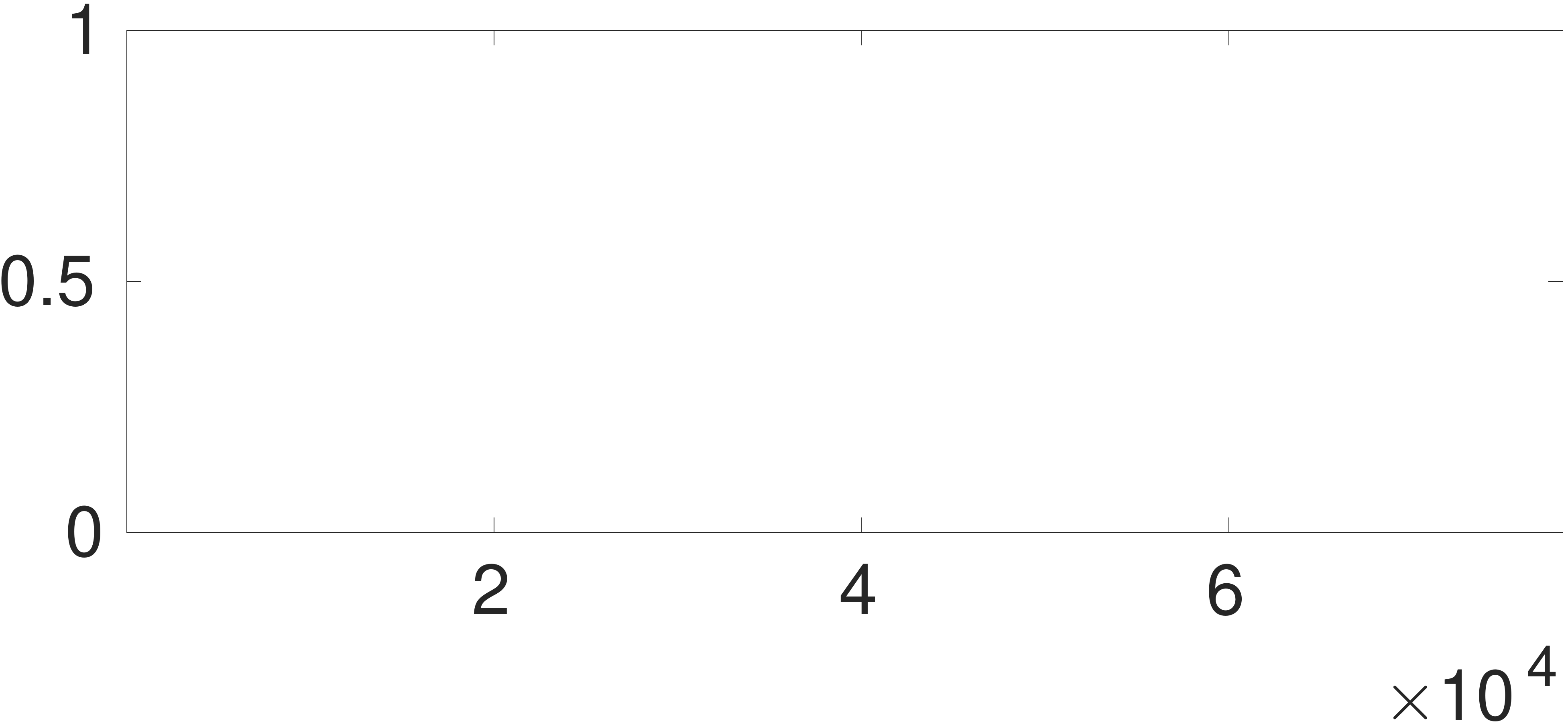}}
\hspace{0.5cm}
\subfloat[ReLU-GPN.]{\includegraphics[width=0.4\textwidth]{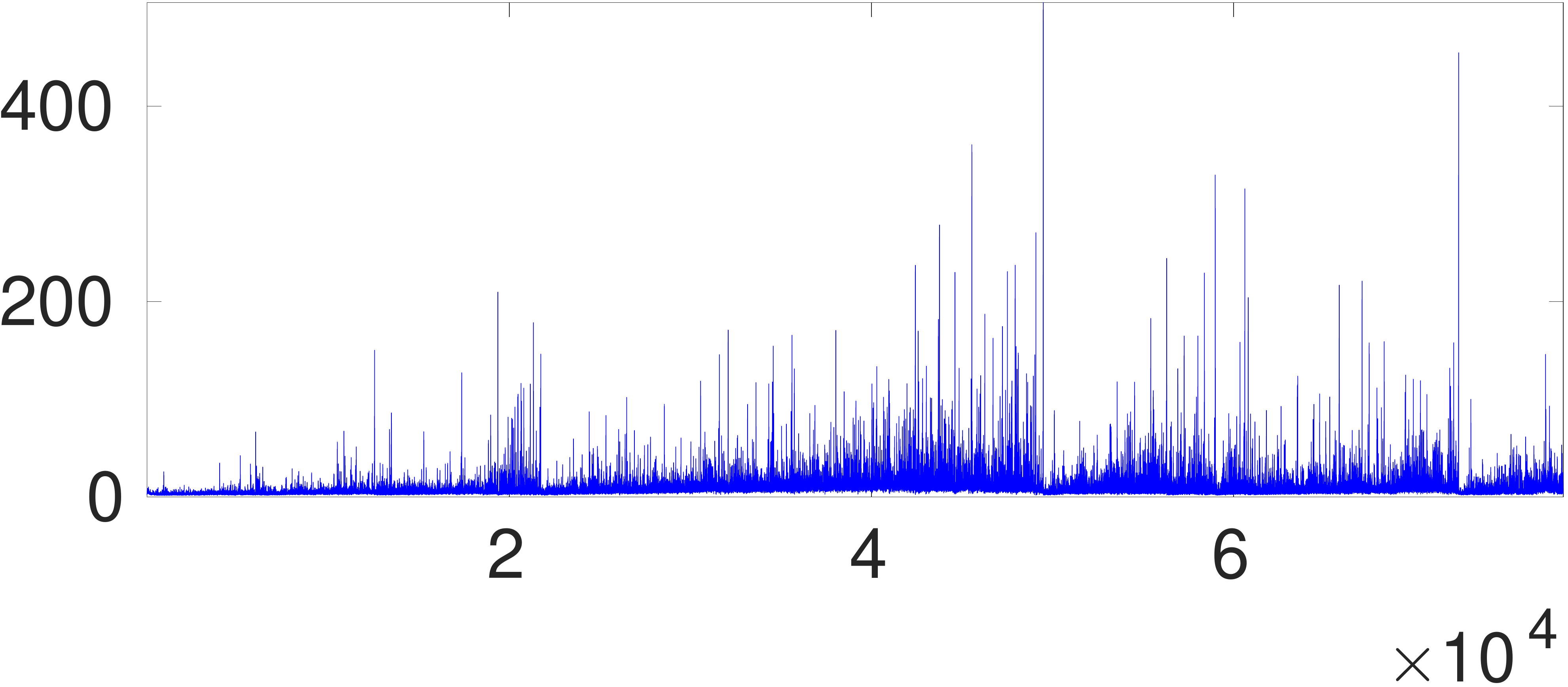}}

\subfloat[ReLU-BN.]{\includegraphics[width=0.4\textwidth]{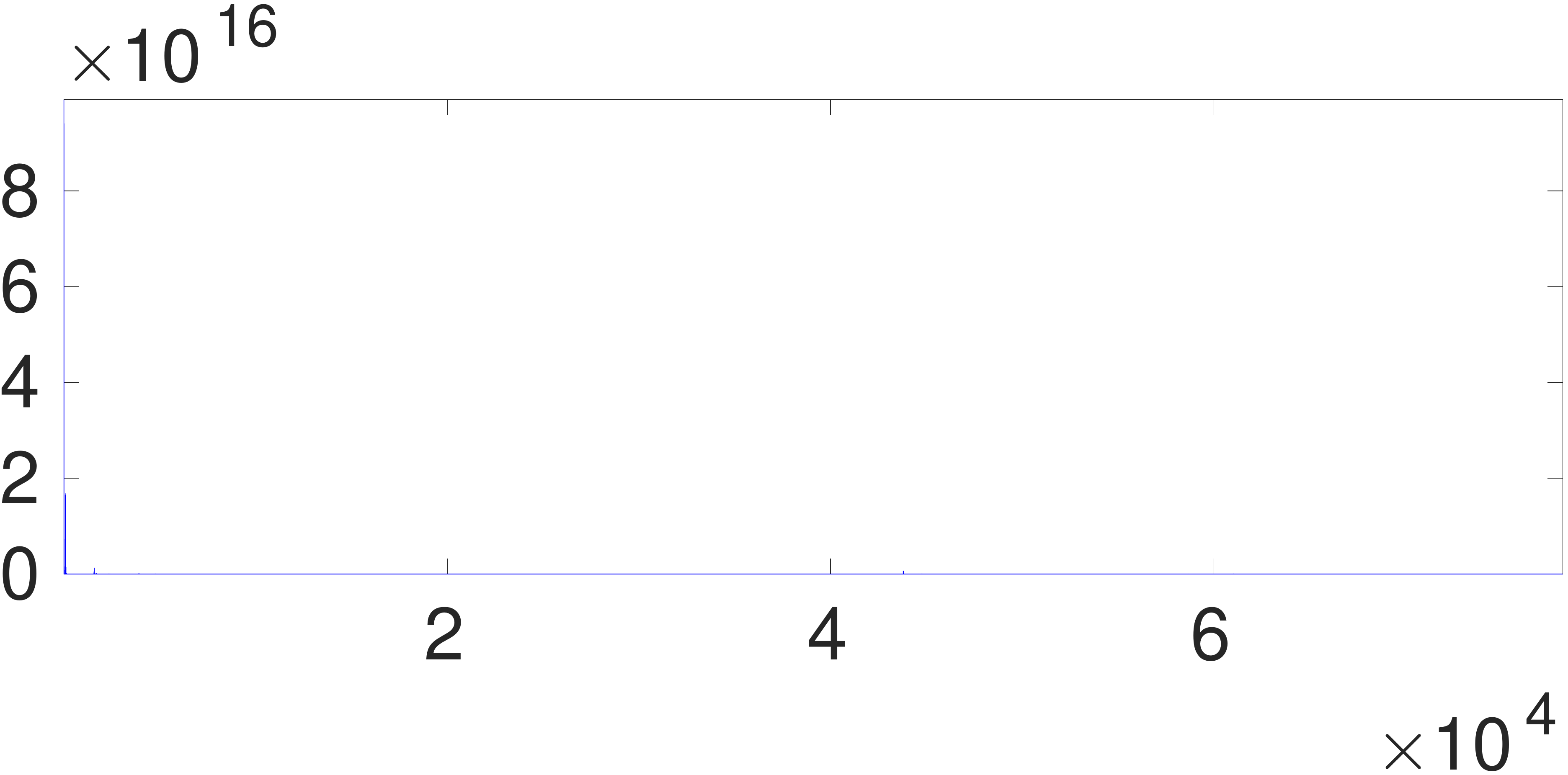}}
\hspace{0.5cm}
\subfloat[ReLU-GPN-BN.]{\includegraphics[width=0.4\textwidth]{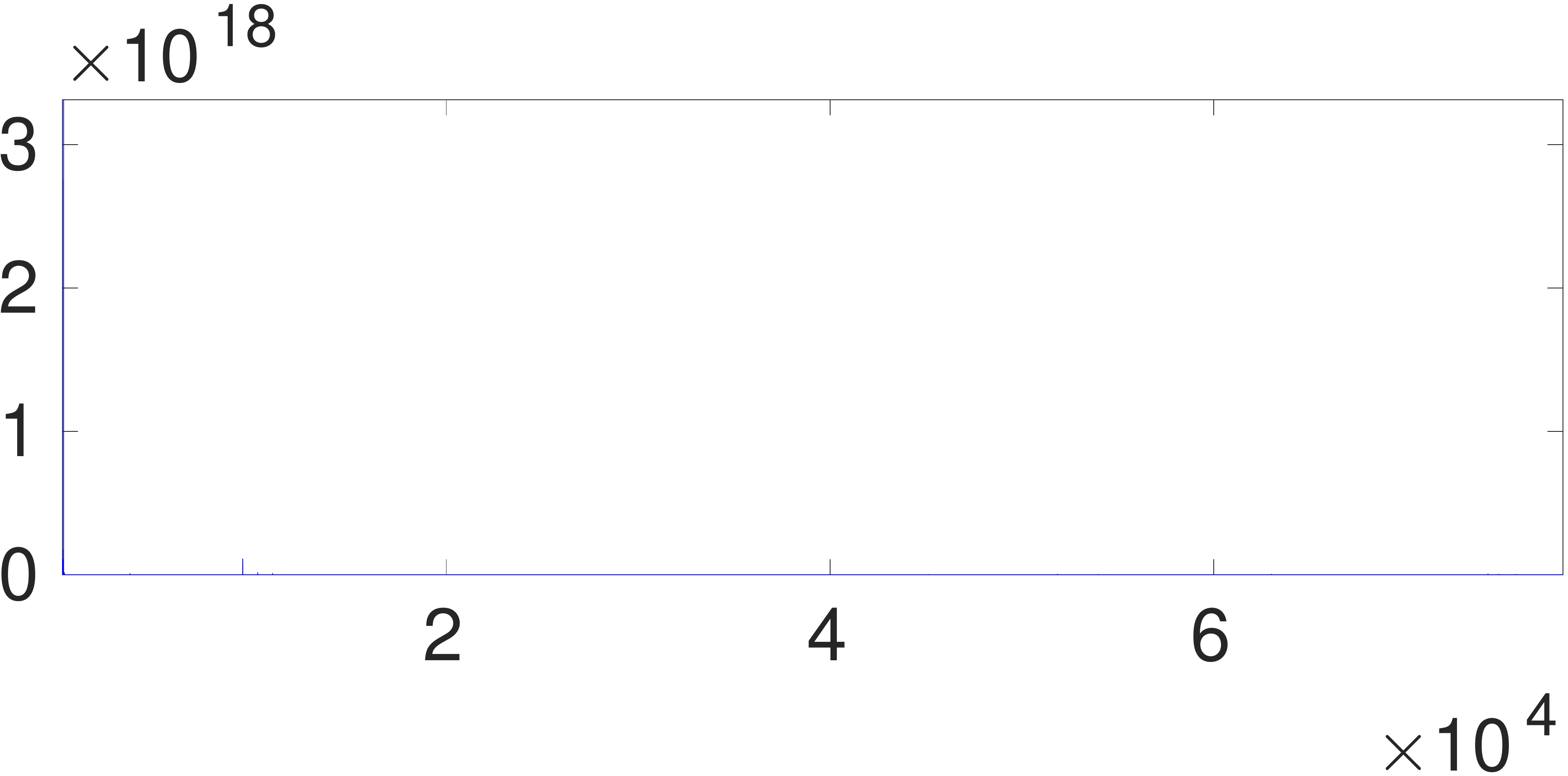}}

\subfloat[LeakyReLU.]{\includegraphics[width=0.4\textwidth]{img/empty_cifar10_grad.pdf}}
\hspace{0.5cm}
\subfloat[LeakyReLU-GPN.]{\includegraphics[width=0.4\textwidth]{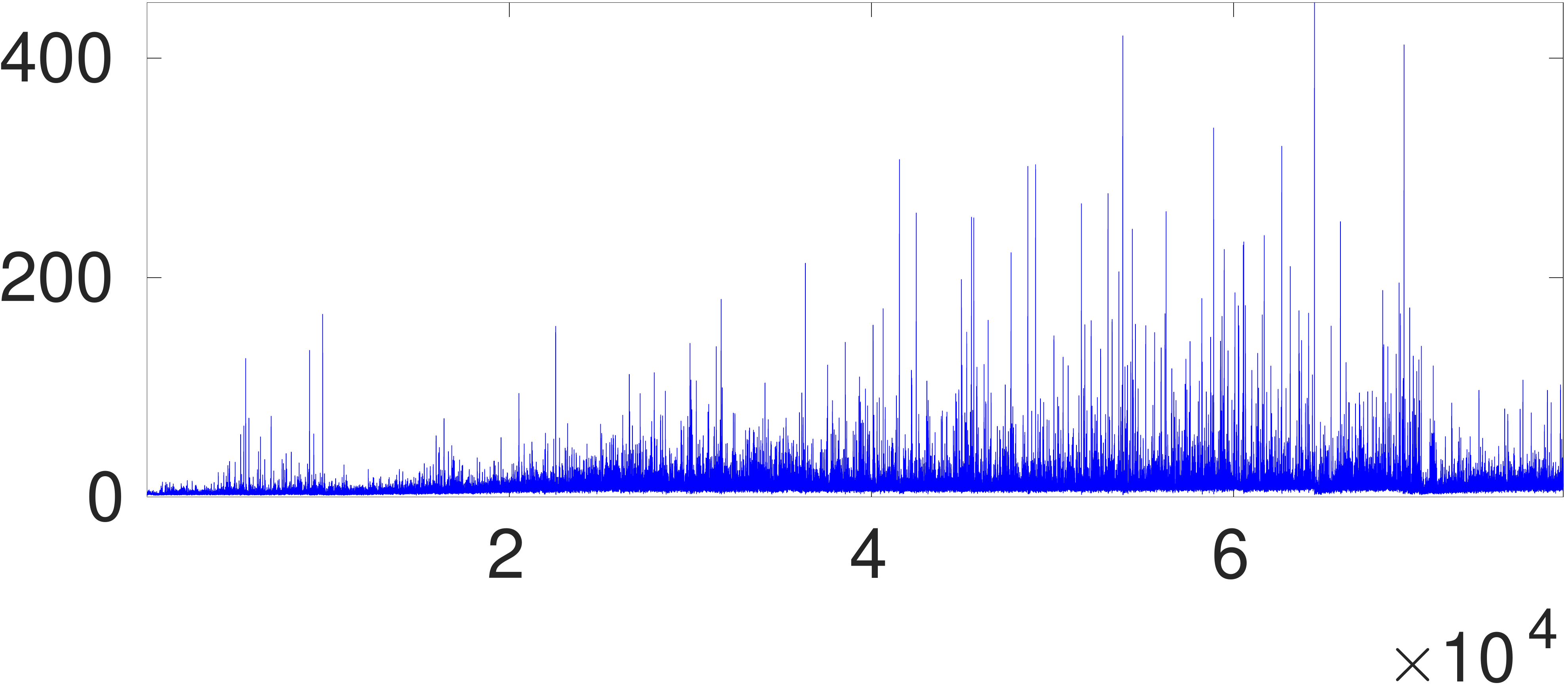}}

\subfloat[LeakyReLU-BN.]{\includegraphics[width=0.4\textwidth]{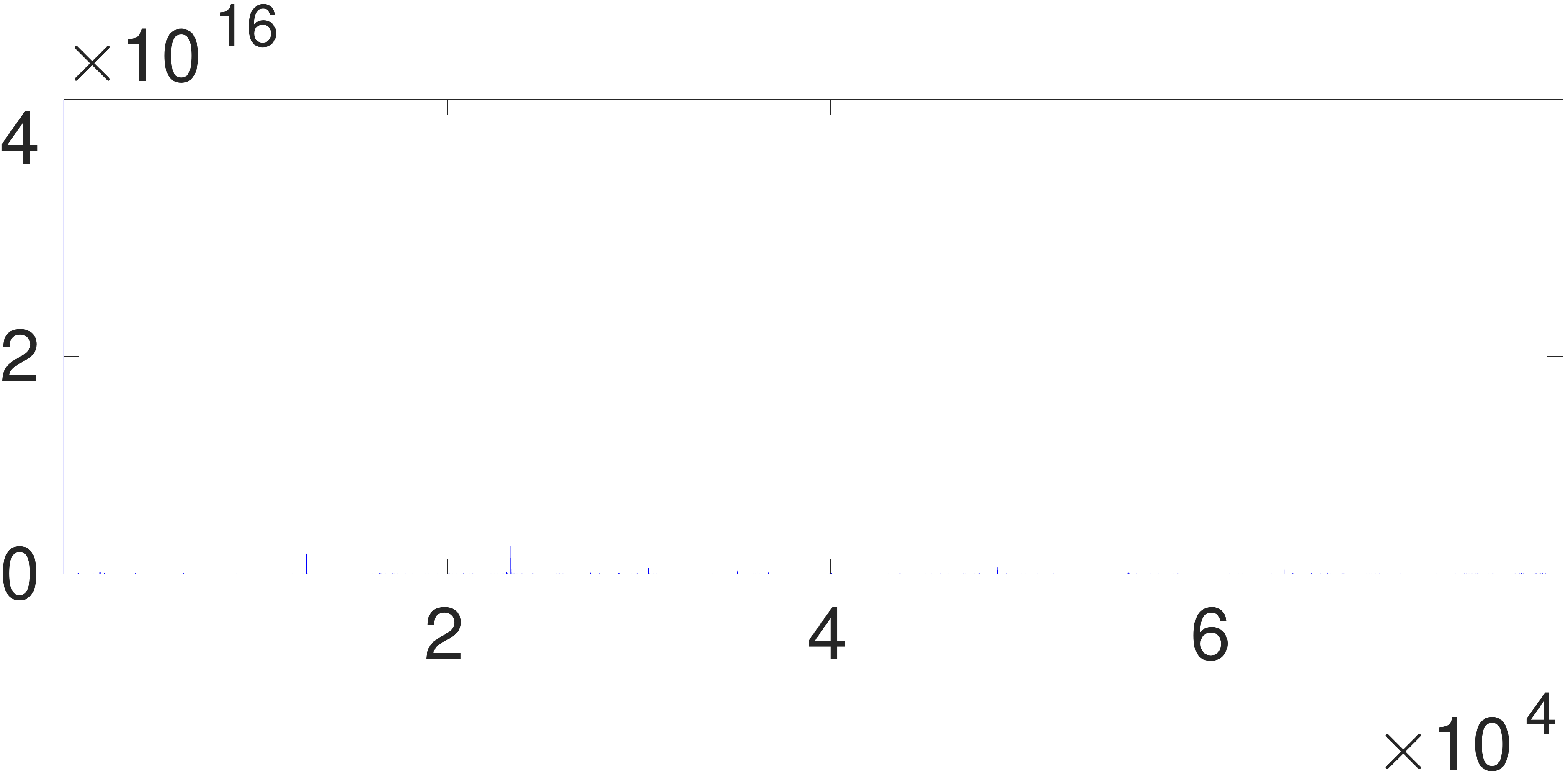}}
\hspace{0.5cm}
\subfloat[LeakyReLU-GPN-BN.]{\includegraphics[width=0.4\textwidth]{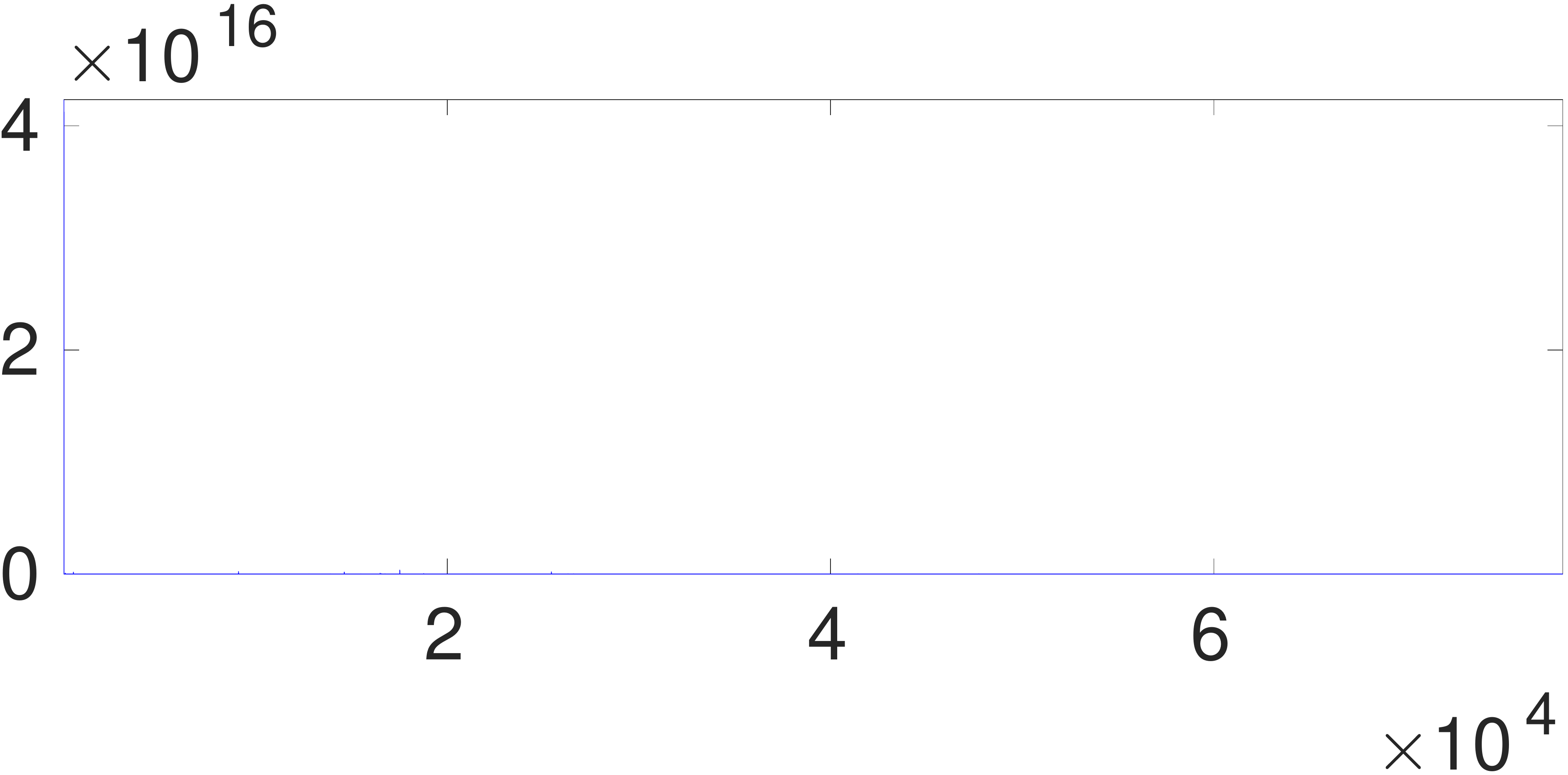}}

\vspace{0.5cm}
\caption{Gradient norm ratio during training on CIFAR-10. Horizontal axis denotes the mini-batch updates. Vertical axis denotes the gradient norm ratio $\max_l\|\frac{\partial E}{\partial \mathbf{V}^{(l)}}\|_F / \min_l\|\frac{\partial E}{\partial \mathbf{V}^{(l)}}\|_F$. The gradient vanishes ($\|\frac{\partial E}{\partial \mathbf{V}^{(l)}}\|_F\approx 0$) for ReLU and LeakyReLU during training and hence the plots are empty.}
\label{fig:grad_3}

\end{figure}

\begin{figure}[h!]
\centering

\subfloat[ELU.]{\includegraphics[width=0.4\textwidth]{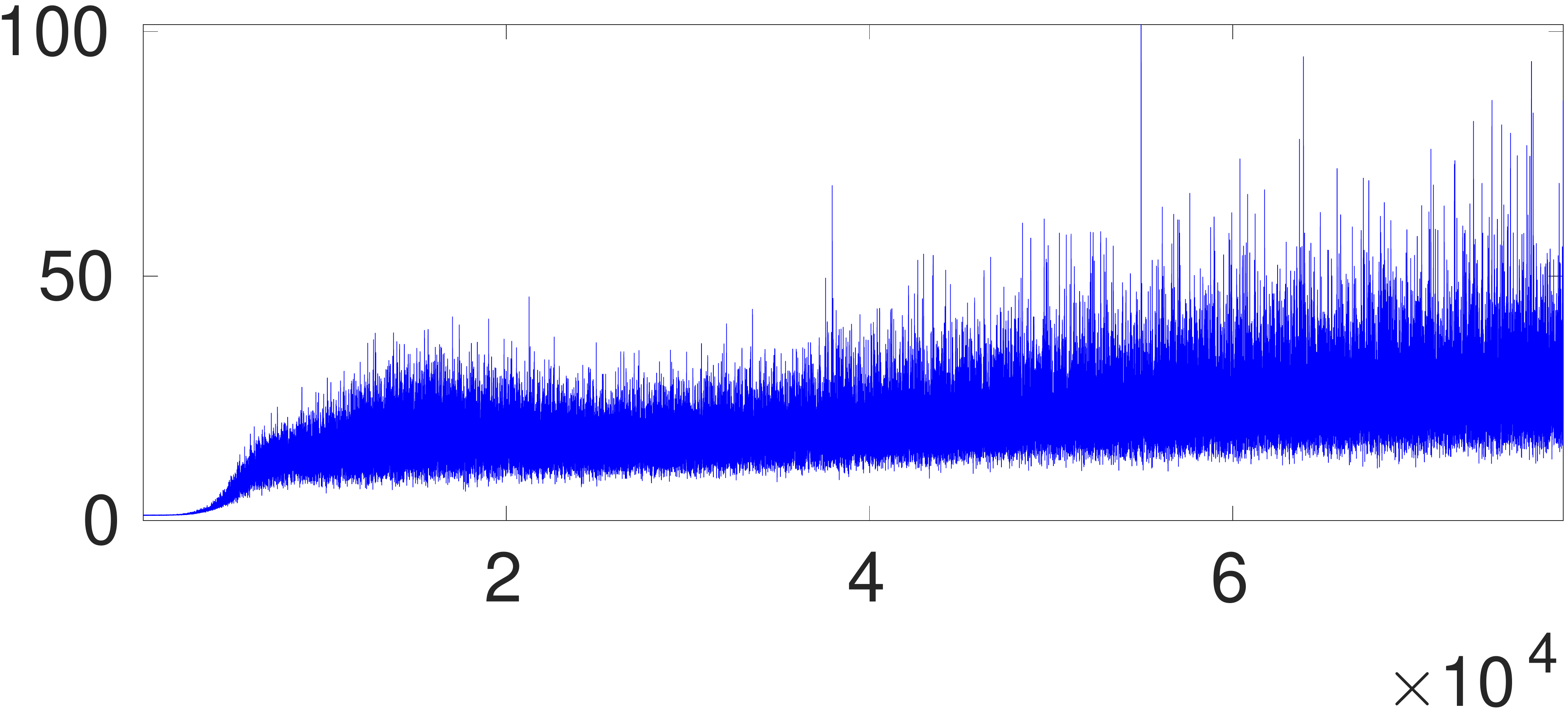}}
\hspace{0.5cm}
\subfloat[ELU-GPN.]{\includegraphics[width=0.4\textwidth]{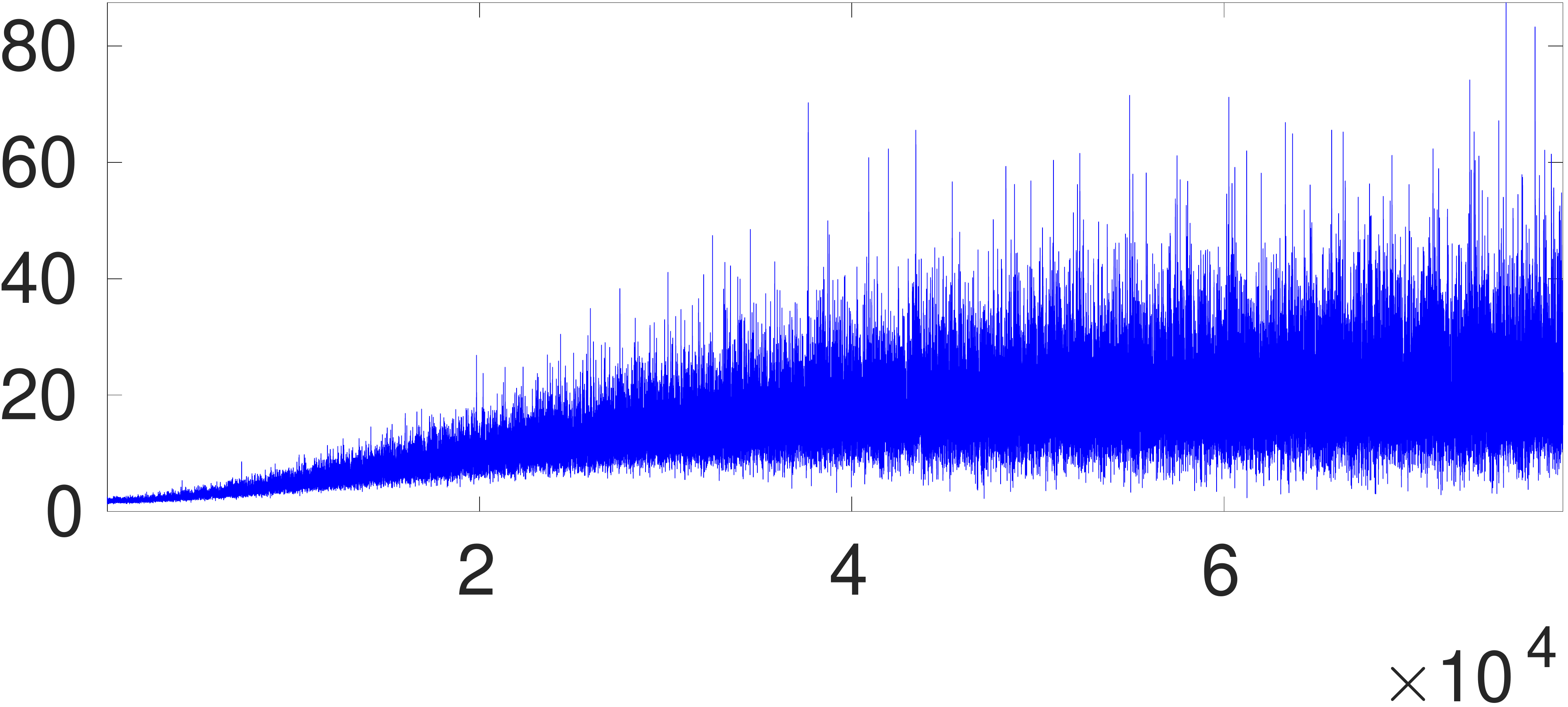}}

\subfloat[ELU-BN.]{\includegraphics[width=0.4\textwidth]{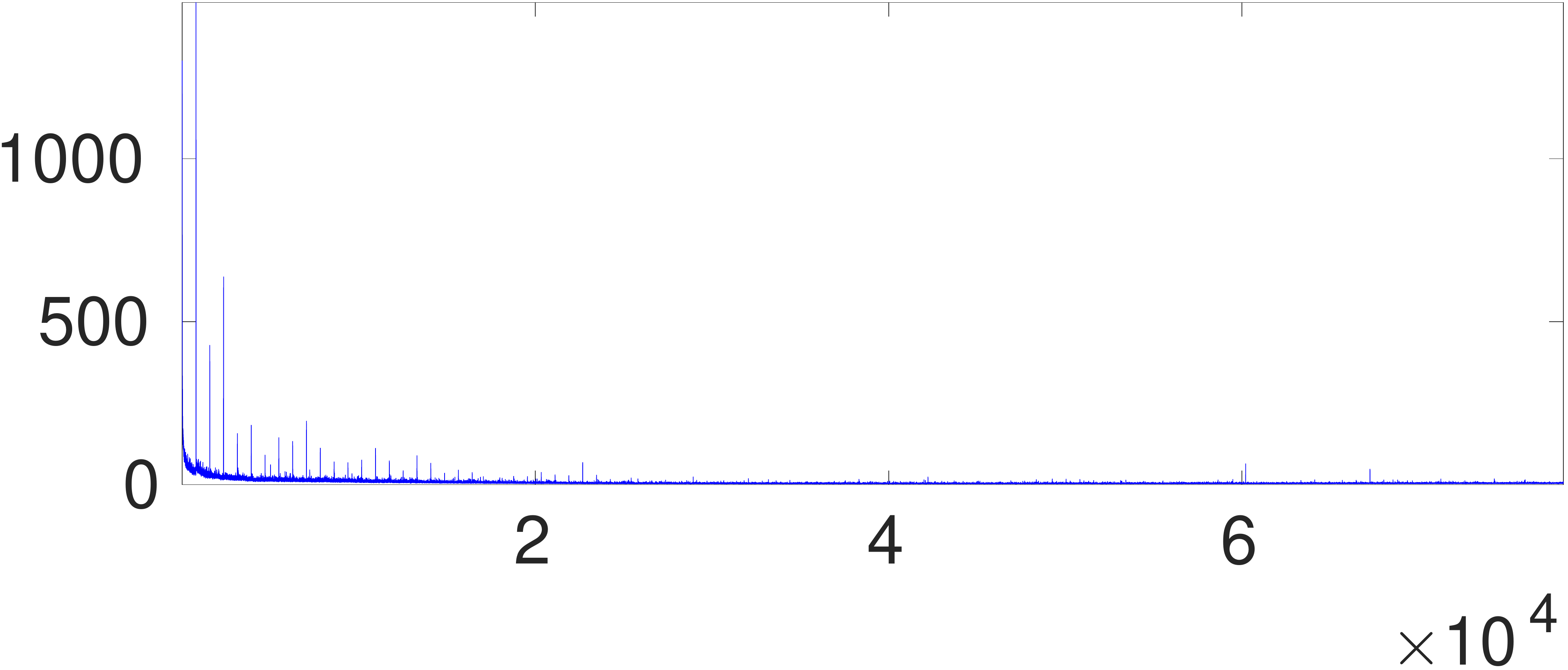}}
\hspace{0.5cm}
\subfloat[ELU-GPN-BN.]{\includegraphics[width=0.4\textwidth]{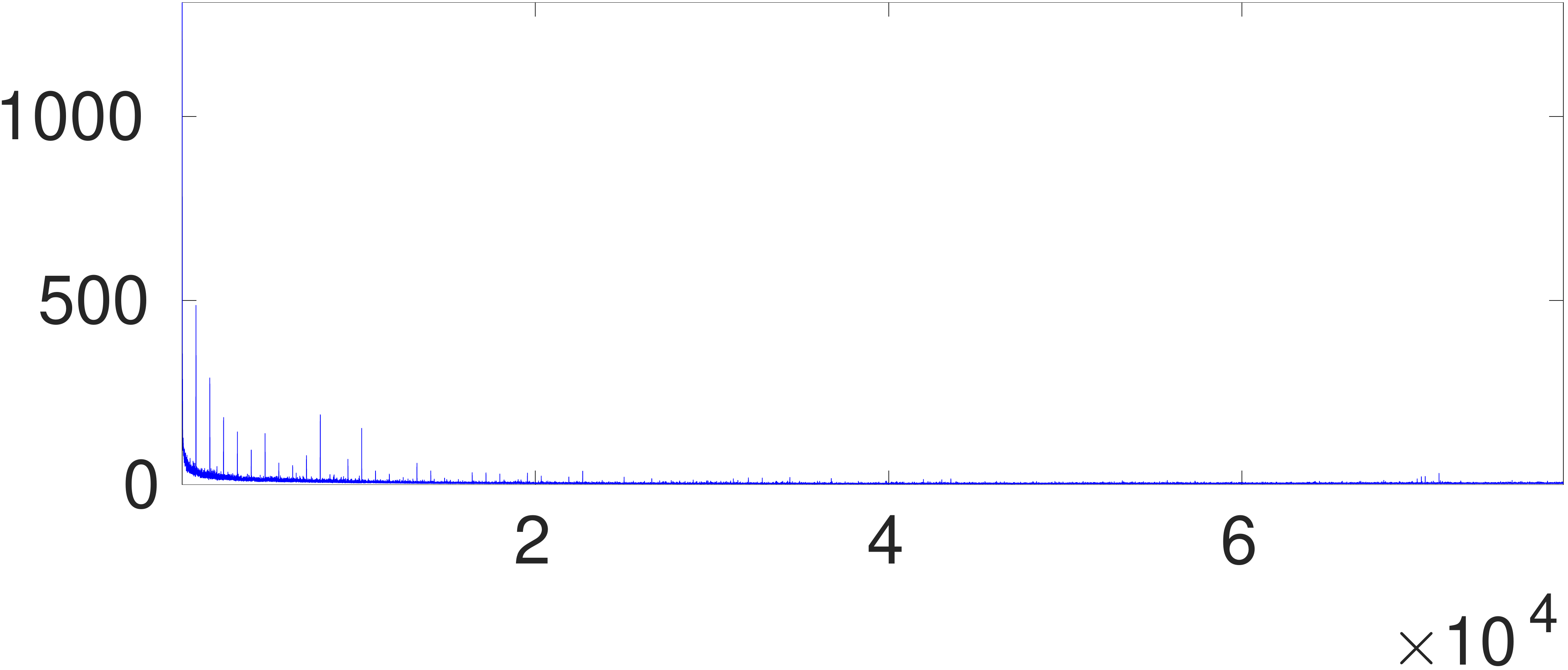}}

\subfloat[SELU.]{\includegraphics[width=0.4\textwidth]{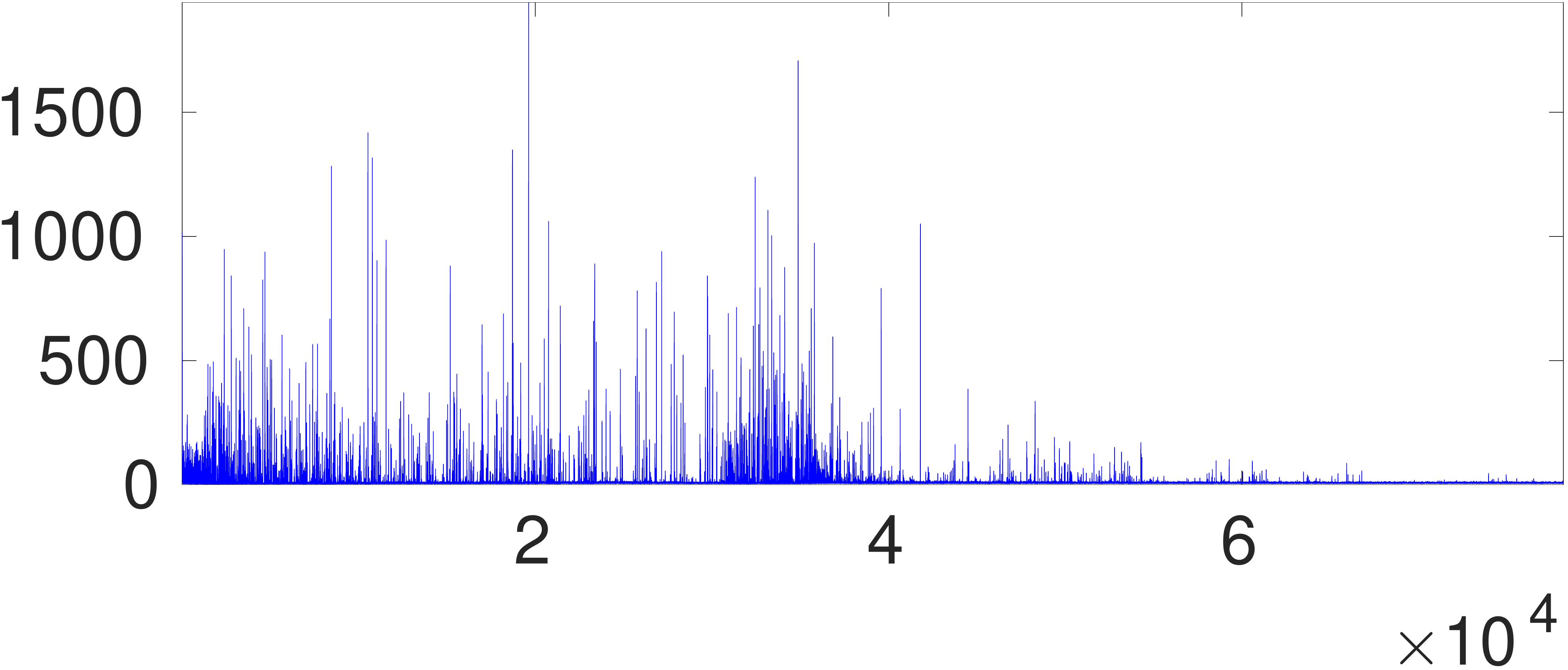}}
\hspace{0.5cm}
\subfloat[SELU-GPN.]{\includegraphics[width=0.4\textwidth]{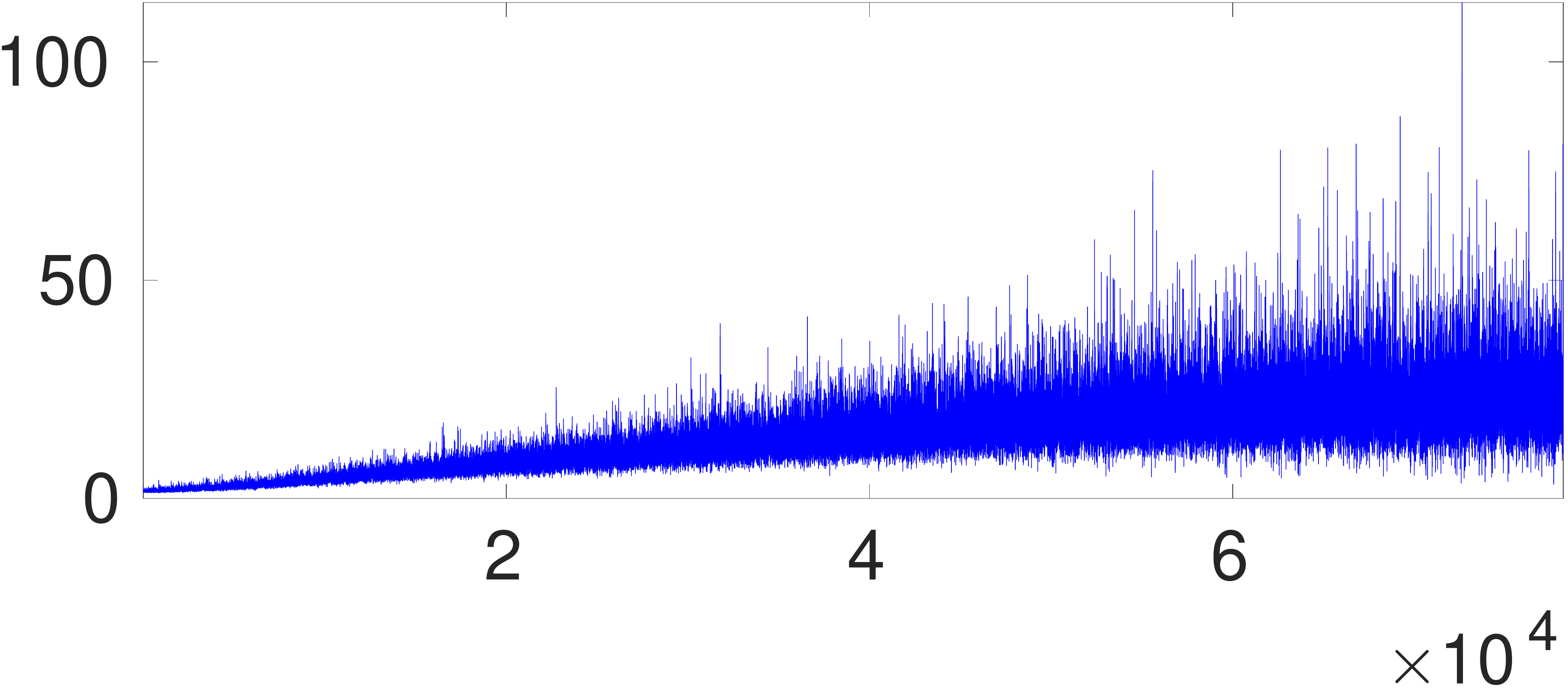}}

\subfloat[SELU-BN.]{\includegraphics[width=0.4\textwidth]{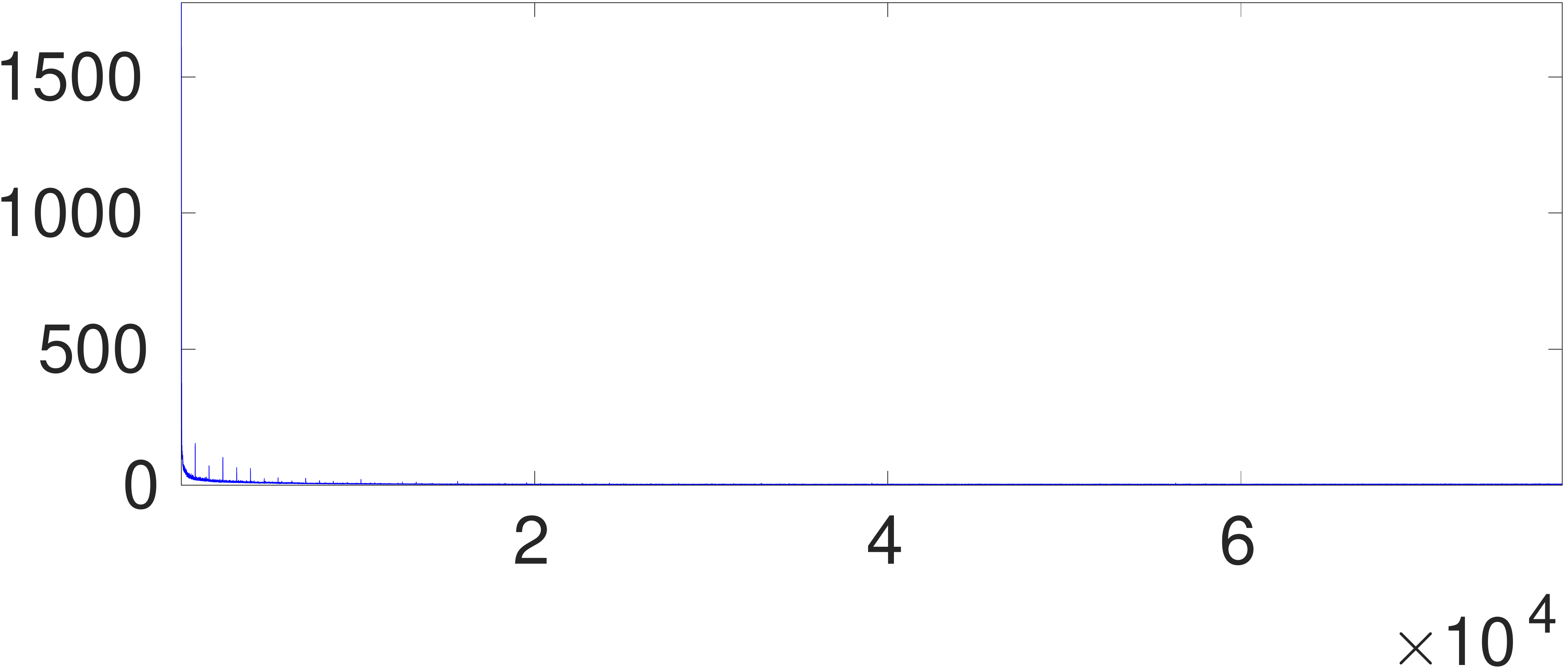}}
\hspace{0.5cm}
\subfloat[SELU-GPN-BN.]{\includegraphics[width=0.4\textwidth]{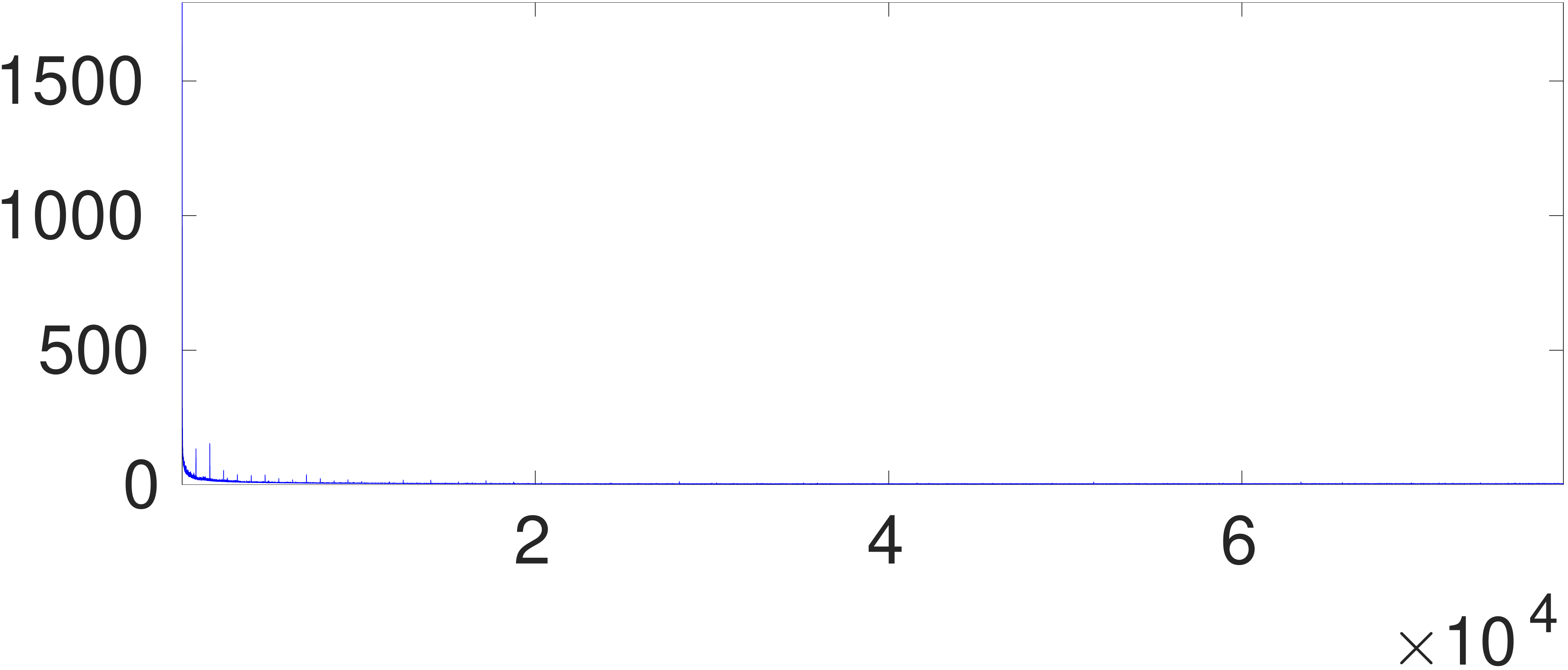}}

\subfloat[GELU-BN.]{\includegraphics[width=0.4\textwidth]{img/empty_cifar10_grad.pdf}}
\hspace{0.5cm}
\subfloat[GELU-GPN.]{\includegraphics[width=0.4\textwidth]{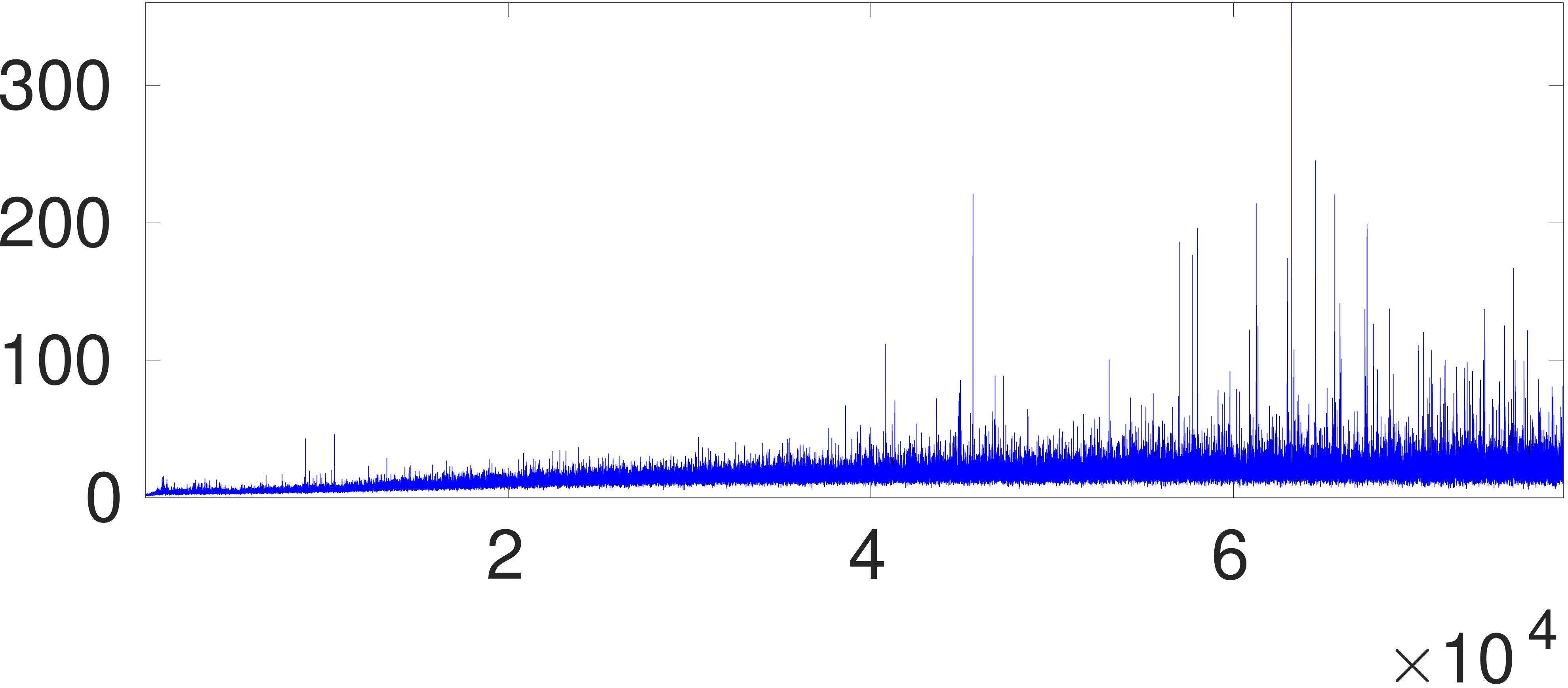}}

\subfloat[GELU-BN.]{\includegraphics[width=0.4\textwidth]{img/empty_cifar10_grad.pdf}}
\hspace{0.5cm}
\subfloat[GELU-GPN-BN.]{\includegraphics[width=0.4\textwidth]{img/empty_cifar10_grad.pdf}}

\vspace{0.5cm}
\caption{Gradient norm ratio during training on CIFAR-10. Horizontal axis denotes the mini-batch updates. Vertical axis denotes the gradient norm ratio $\max_l\|\frac{\partial E}{\partial \mathbf{V}^{(l)}}\|_F / \min_l\|\frac{\partial E}{\partial \mathbf{V}^{(l)}}\|_F$. The gradient vanishes ($\|\frac{\partial E}{\partial \mathbf{V}^{(l)}}\|_F\approx 0$) for GELU, GELU-BN and GELU-GPN-BN during training and hence the plots are empty. For GELU-GPN-BN, both gradient exploding and gradient vanishing are observed.}
\label{fig:grad_4}

\end{figure}
\end{appendices}

\end{document}